\documentclass[nohyperref]{article}

\usepackage{microtype}
\usepackage{graphicx}
\usepackage{booktabs} % for professional tables

\usepackage[pagebackref=true]{hyperref} %[draft]
\renewcommand*\backref[1]{\ifx#1\relax \else (Cited on #1) \fi}

\usepackage[accepted]{icml2023}

\usepackage{amsmath}
\usepackage{amssymb}
\usepackage{mathtools}
\usepackage{amsthm}

\usepackage[capitalize]{cleveref} %noabbrev

\usepackage{algorithm, algorithmic}
\usepackage{amsfonts,color,url}
\usepackage{thm-restate}
\usepackage{xspace}
\usepackage{enumitem}
\usepackage{dsfont}
\usepackage{adjustbox}
\newcommand{\mtodo}[1]{\todo[color=red!30, inline]{ #1}} 
\usepackage{subcaption}
\graphicspath{{./figs/}}

\theoremstyle{plain}
\newtheorem{theorem}{Theorem}[section]
\newtheorem{proposition}[theorem]{Proposition}
\newtheorem{lemma}[theorem]{Lemma}
\newtheorem{corollary}[theorem]{Corollary}
\newtheorem{example}[theorem]{Example}

\theoremstyle{definition}
\newtheorem{definition}[theorem]{Definition}

\theoremstyle{remark}
\newtheorem{remark}[theorem]{Remark}

\usepackage[disable]{todonotes}

\def\argmin{ \mathop{{\rm argmin}}}

\newcommand{\diam}{\mathrm{diam}\,}
\newcommand{\dom}{\mathrm{dom}\,}

\newcommand{\ri}{\mathrm{ri}\,}

\newcommand{\p}{\partial}

\newcommand{\R}{\mathbb{R}}

\newcommand{\rbar}{\overline{\mathbb R}}

\newcommand{\bN}{\mathbb{N}}

\newcommand{\ip}[2]{\langle #1,\, #2\rangle}

\newcommand{\set}[2]{\left\{#1\,\left\vert\; #2\right.\right\}}

\newcommand{\cN}{\mathcal{N}}

\newcommand{\1}{\mathds{1}}

\newcommand{\DCAR}{{DCAR}\xspace}
\newcommand{\CDCAR}{{CDCAR}\xspace}
\newcommand{\ADCAR}{{ADCAR}\xspace}
\newcommand{\round}{\mathrm{Round}_F}

\newcommand{\red}[1]{}%{\textcolor{red}{#1}}
\newcommand{\blue}[1]{\textcolor{black}{#1}}
\newcommand{\green}[1]{\textcolor{black}{#1}}

\icmltitlerunning{Difference of submodular minimization via DC programming}

\begin{document}

\twocolumn[
\icmltitle{Difference of Submodular Minimization via DC Programming}



\icmlsetsymbol{equal}{*}

\begin{icmlauthorlist}
\icmlauthor{Marwa El Halabi}{S}
\icmlauthor{George Orfanides}{M}
\icmlauthor{Tim Hoheisel}{M}
\end{icmlauthorlist}

\icmlaffiliation{S}{Samsung - SAIT AI Lab, Montreal}
\icmlaffiliation{M}{Department of Mathematics and Statistics, McGill University}

\icmlcorrespondingauthor{Marwa El Halabi}{marwa.elhalabi@gmail.com}

\icmlkeywords{Difference of submodular functions, Difference of convex functions, DC programming, DCA, Complete DCA, Submodular-Supermodular procedure, SubSup, Lov\'asz extension}

\vskip 0.3in
]



\printAffiliationsAndNotice{}  % leave blank if no need to mention equal contribution

\begin{abstract}
Minimizing the difference of two submodular (DS) functions is a problem that naturally occurs in various machine learning problems. Although it is well known that a DS problem can be equivalently formulated as the minimization of the difference of two convex (DC) functions, existing algorithms do not fully exploit this connection.  A classical algorithm for DC problems is called the DC algorithm (DCA). We introduce variants of DCA and its complete form (CDCA) that we apply to the DC program corresponding to DS minimization. 
We extend existing convergence properties of DCA, and connect them to convergence properties on the DS problem. Our results on DCA match the theoretical guarantees satisfied by existing DS algorithms, while providing a more complete characterization of convergence properties. In the case of CDCA, we obtain a stronger local minimality guarantee. Our numerical results show that our proposed algorithms outperform existing baselines on two applications: speech corpus selection and feature selection.  
%Unlike DCA which can be efficiently implemented, CDCA requires solving an NP-Hard subproblem at each iteration. We show that we can efficiently obtain a stationary point of this subproblem using the Frank-Wolfe algorithm. 
%, the submodular-supermodular procedure (Subsup) \cite{}, is a special case of DCA 
\end{abstract}

\section{Introduction}
We study the difference of submodular (DS) functions minimization problem
%the minimization of the difference of two submodular (DS) functions: 
\begin{equation}\label{eq:DS}
\min_{X \subseteq V} F(X):= G(X) - H(X),
\end{equation}
\looseness=-1 where $G$ and $H$ are normalized submodular functions (see \cref{sec:prelim} for definitions). We denote the minimum of \eqref{eq:DS} by $F^\star$.
Submodular functions are set functions that satisfy a diminishing returns property, which naturally occurs in a variety of machine learning applications. Many of these applications involve DS minimization, %minimizing the difference of two submodular (DS) functions, 
such as feature selection, probabilistic inference \cite{Iyer2012a}, learning discriminatively structured graphical models \cite{Narasimhan2005a}, and learning decision rule sets   \cite{Yang2021}. In fact, this problem is ubiquitous as any set function can be expressed as a DS function, though finding a DS decomposition has exponential complexity in general \cite{Narasimhan2005a, Iyer2012a}. 

Unlike submodular functions which can be minimized in polynomial time, 
minimizing DS functions up to any constant factor multiplicative approximation requires exponential time, and obtaining any  positive polynomial time computable multiplicative approximation is NP-Hard \citep[Theorems 5.1 and 5.2]{Iyer2012a}. Even finding a local minimum (see \cref{def:localmin}) of DS functions is PLS complete \citep[Section 5.3]{Iyer2012a}.

DS minimization was first studied in \cite{Narasimhan2005a}, who proposed the submodular-supermodular (SubSup) procedure; an algorithm inspired by 
the convex-concave procedure \cite{Yuille2001}, which monotonically
reduces the objective function at every step and converges to a local minimum. \citet{Iyer2012a} extended the work of \cite{Narasimhan2005a} by proposing two other algorithms, the supermodular-submodular (SupSub) and  the modular-modular (ModMod) procedures, which have lower per-iteration cost than the SubSup method, while satisfying the same theoretical guarantees. 
%\citep{Maehara2015} studied the more general problem of minimizing the difference of discrete convex functions, and propose an analogue of the continuous DC algorithm 

The DS problem can be equivalently formulated as a difference of convex (DC) functions minimization problem (see \cref{sec:prelim}).
%, by replacing each submodular function by its Lov\'asz extension (see \cref{def:LE}). 
%DC programs play an important role in continuous optimization, since most non-convex problems encountered in practice can be formulated as a DC program \cite{}
DC programs are well studied problems for which a classical popular algorithm is the DC algorithm (DCA) \cite{Tao1997, tao1988duality}. DCA has been successfully applied to a wide range of non-convex optimization problems, and several algorithms can be viewed as special cases of it, such as the convex-concave procedure, the expectation-maximization \cite{Dempster1977}, and the iterative shrinkage-thresholding algorithm \cite{Chambolle1998}; see \cite{le2018dc} for an extensive survey on DCA. % and  Proximal point algorithm  

Existing DS algorithms, while inspired by DCA, do not fully exploit this connection to DC programming. In this paper, we 
apply DCA and its complete form (CDCA) to the DC program equivalent to the DS problem. We establish new connections between the two problems which allow us to leverage convergence properties of DCA to obtain theoretical guarantees on the DS problem that match ones by existing methods, and stronger ones when using CDCA. In particular, our key contributions are:

\begin{itemize}
\item We show that a special instance of DCA and CDCA, where iterates are integral,  monotonically decreases the DS function value at every iteration, and  converges with rate $O(1/k)$ to a local minimum and strong local minimum (see \cref{def:localmin}) of the DS problem, respectively. DCA reduces to SubSup in this case.
%\item We extend known convergence results of DCA and CDCA to handle approximate convergence and inexact iterates.
\item We introduce variants of DCA and CDCA, where iterates are rounded at each iteration, which allow us to add regularization. We extend the convergence properties of DCA and CDCA to these variants. %showing in particular that they monotonically decrease the DS function at every iteration, and converge to a local minimum and strong local minimum (see \cref{def:localmin}) of the DS problem, respectively. 
%\item We extend convergence results of DCA and CDCA to handle approximate convergence and inexact iterates.
\item CDCA requires solving a concave minimization subproblem at each iteration. We show how to efficiently obtain an approximate stationary point of this subproblem using the Frank-Wolfe (FW) algorithm. % when rho=0, \phi is not differentiable, and FW converges to critical point, which is not called stationary point in this case. There's some results showing though that strong criticality is somewhat equivalent to d-stationarity. So criticality is a notion of weak stationarity, so we can be vague here and say approximate stationary pt, as ppl would be more familiar with this term vs critical pt.
\item We study the effect of adding regularization both theoretically and empirically. 
\item We demonstrate that our proposed methods outperform existing baselines empirically on two applications: speech corpus selection and feature selection. 
\end{itemize}

\subsection{Additional related work}
\looseness=-1 An accelerated variant of DCA (ADCA) which incorporates Nesterov's acceleration into DCA was presented in \cite{Nhat2018}. We investigate the effect of acceleration in our experiments (\cref{sec:exps}).
\citet{Kawahara2011} proposed an exact branch-and-bound algorithm for DS minimization, which has exponential time-complexity.
\citet{Maehara2015} proposed a discrete analogue of the continuous DCA for minimizing the difference of discrete convex functions, of which DS minimization is a special case, where the proposed algorithm reduces to SubSup.
%A line of work studied special cases of the DS problem. 
Several works studied a special case of the DS problem where $G$ is modular \cite{Sviridenko2017,Feldman2019,Harshaw2019}, or approximately modular \cite{Perrault2021}, providing approximation guarantees based on greedy algorithms. 
% Note that \cite{Sviridenko2017,Feldman2019} consider a constrained variant of this problem with cardinality or matroid constraints, but the unconstrained case is a special case of the cardinality constraint one with k = |V|, their work then applies to a special case of the DS problem.  \cite{Harshaw2019} consider both constrained and unconstrained variant of the problem.
\citet{Halabi20} provided approximation guarantees to the related problem of minimizing the difference between an approximately submodular function and an approximately supermodular function.
In this work we focus on general DS minimization, we discuss some implications of our results to certain special cases in \cref{sec:specialCases}.

\section{Preliminaries}\label{sec:prelim}

We begin by introducing our notation and relevant background on DS and DC minimization.

\paragraph{Notation:} Given a ground set $V = \{1, \cdots, d\}$ and a set function $F: 2^V \to \R$, we denote the \emph{marginal gain} of adding an element $i$ to a set $X \subseteq V$ by $F( i | X) = F( X \cup \{i\}) - F(X)$. The indicator vector $\1_X \in \R^d$ %is the indicator vector of $X$, 
is the vector whose $i$-th entry is $1$ if $i \in X$ and $0$ otherwise. 
Let $S_d$ denote the set of permutations on $V$. Given $\sigma \in S_d$, set $S^\sigma_k :=\{\sigma(1), \cdots, \sigma(k)\}$, %be the set of the first $k$ elements, 
with $S^\sigma_0 = \emptyset$.  The symmetric difference of two sets $X, Y$ is denoted by $X \Delta Y = (X \setminus Y) \cup (Y \setminus X)$.
%Given a vector $\x \in \R^d$, $x_i$ is its $i$-th entry and $\supp(\x) = \{ i \in V | x_i \not = 0\}$ is its support set; $\x$ also defines a \emph{modular} set function as $\x(A) = \sum_{i \in A} x_i$. \vspace{-5pt}
%Let $A\subset V:=\{1,\dots,d\}$. Then $\1_A\in \R^d$ is the vector whose $i$-th entry is $1$ if $i\in A$ and $0$ otherwise. In particular,  $e:=e_V$ is the vector of all ones. 
Denote by $\Gamma_0$ the set of all proper lower semicontinuous convex functions on $\R^d$. We write $\rbar$ for $\R \cup \{+\infty\}$. Given a set $C \subseteq \R^d, \delta_C$ denotes the indicator function of $C$ taking value $0$ on $C$ and $+\infty$ outside it. Throughout,
$\| \cdot \|$ denotes the $\ell_2$-norm.

\paragraph{DS minimization} A set function $F$ is \emph{normalized} if $F(\emptyset) = 0$ and
\emph{non-decreasing} if $F(X) \leq F(Y)$  for all $X \subseteq Y$.
$F$ is \emph{submodular} if it has diminishing marginal gains: $F( i\mid X) \geq F(i \mid Y)$ for all $X \subseteq Y$, $i \in V\setminus Y$, %\emph{modular} if the inequality holds as an equality, and 
\emph{supermodular} if $-F$ is submodular, and \emph{modular} if it is both submodular and supermodular. Given a vector $x \in \R^d$, $x$ defines a \emph{modular} set function as $x(A) = \sum_{i \in A} x_i$.
Note that minimizing the difference between two submodular functions is equivalent to maximizing the difference between two submodular functions, and minimizing or maximizing the difference of two supermodular functions.

Given the inapproximability of Problem \eqref{eq:DS}, 
%of obtaining an approximate global minimum of \eqref{eq:DS}, 
we are interested in obtaining approximate local minimizers. % in the following sense. 
\begin{definition}\label{def:localmin}
Given $\epsilon \geq 0$, a set $X \subseteq V$ is an $\epsilon$-\emph{local minimum} of $F$ if $F(X) \leq F(X \cup i) + \epsilon \text{ for all } i \in V\setminus X$ and $F(X) \leq F(X \setminus i) + \epsilon \text{ for all } i \in X$.
%\begin{align*}
%F(X) &\leq F(X \cup i) + \epsilon \text{ for all } i \in V\setminus X \\ 
%\text{ and } F(X) &\leq F(X \setminus i) + \epsilon \text{ for all } i \in X.
%\end{align*}
Moreover, $X$ is an $\epsilon$-\emph{strong local minimum} of $F$ if $F(X) \leq F(Y) + \epsilon \text{ for all } Y \subseteq X \text{ and all } Y \supseteq X$.
\end{definition}
In \cref{sec:specialCases}, we show that if $F$ is submodular then any $\epsilon$-strong local minimum $\hat{X}$ of $F$ is also an $\epsilon$-global minimum, i.e., $F(\hat X) \leq F^\star + \epsilon$. It was also shown in \citep[Theorem 3.4]{Feige2011} that if $F$ is supermodular then any $\epsilon$-strong local minimum $\hat{X}$ satisfies $\min\{F(\hat{X}), F(V \setminus \hat{X})\} \leq \tfrac{1}{3} F^\star + \tfrac{2}{3} \epsilon$. We further show relaxed versions of these properties for approximately submodular and supermodular functions in  \cref{sec:specialCases}.
Moreover, the two notions of approximate local minimality are similar if $F$ is supermodular: any $\epsilon$-{local minimum} of $F$ is also an $\epsilon d$-{strong local minimum} of $F$ \citep[Lemma 3.3]{Feige2011}.
However, in general, a local miniumum can have an arbitrarily worse objective value than any strong local minimum, as illustrated in \cref{ex:strongLocalMin}.
% this actually is true if F is weakly DR-supermodular, but that would imply that F is non-decreasing if it's not exactly supermodular, which would make this trivial.
 %for all $Y \subseteq V$. , 
%and the closer $F$ is to being submodular the closer are its  $\epsilon$-strong local minima to being  $\epsilon$-global minima (see  \cref{sec:specialCases}).

Minimizing a set function $F$ is equivalent to minimizing a \emph{continuous extension} of $F$ %, obtained by continuous interpolation of $F$ on the full hypercube $[0, 1]^d$ 
called the \emph{Lov\'asz extension} \cite{Lovasz1983} on the hypercube $[0, 1]^d$ . 
%
%from vertices of the hypercube $\{0, 1\}^d$ to the full hypercube $[0, 1]^d$. 
%This extension, called the \emph{Lov\'asz extension}  \cite{Lovasz1983},  is convex if and only if $F$ is submodular. 
%
\begin{definition}[Lov\'asz extension] \label{def:LE}
Given a normalized set function $F$, its {Lov\'asz extension} $f_L: \R^d \to \R$ is defined as follows: Given $x \in \R^d$ %with decreasing order induced by 
and $\sigma \in S_d$, with $x_{\sigma(1)} \geq \cdots \geq x_{\sigma(d)}$, 
$f_L(x) := \sum_{k = 1}^d x_{\sigma(k)} F( \sigma(k) \mid S^\sigma_{k-1}).$
%where $S^\sigma_k =\{\sigma(1), \cdots, \sigma(k)\}$.
\end{definition}

We make use of the following well known properties of the Lov\'asz extension; see  e.g. \cite{Bach2013} and \citep[Lemma 1]{Jegelka2011a} for item \ref{itm:Lip}.

\begin{proposition}\label{prop:LEproperties} For a normalized set function $F$, we have:
% following hold:
\begin{enumerate}[label=\alph*), ref=\alph*] %leftmargin=1em, itemindent=1em
\item \label{itm:extension} For all $X \subseteq V, F(X) = f_L(\1_X)$.
\item \label{itm:sum} If $F = G - H$, then $f_L = g_L - h_L$.
\item \label{itm:equiv} $\min_{X \subseteq V} F(X)=\min_{x\in [0,1]^d} f_L(x)$.  %Moreover, if $S^*$ is a minimizer of $F$, then $1_{S^*}$ is a minimizer of $f_L$, and if $s^*$ is a minimizer of $f_L$, then any set $\{i: s^*_i \geq \alpha \}$, obtained by thresholding $s^*$ with any $\alpha \in (0,1)$, is a minimizer of $F$.
\item \label{itm:round} Rounding: Given $x\in [0,1]^d, \sigma \in S_d$  such that  $x_{\sigma(1)}\geq \dots \geq x_{\sigma(d)}$, let $\hat k \in \argmin_{k=0,1,\dots,d}F(S^\sigma_k)$, then
% $\hat A:=\{\sigma(1),\dots,\sigma(\hat k)\}$ for  $\hat k=\argmin_{k=0,1,\dots,d}F(\{\sigma(1),\dots,\sigma(k)\})$. Then
$
F(S^\sigma_{\hat{k}}) \leq f_L(x). 
$
We denote this operation by $S^\sigma_{\hat{k}} = \round(x)$.
\item \label{itm:conv} $f_L$ is convex if and only if $F$ is submodular. 
%\item \label{itm:support} Let $F$ be submodular and define its base polyhedron
%\begin{equation*}
%\resizebox{.95\hsize}{!}{$B(F):=\set{s \in \R^d}{s(V) = F(V),\; s(A) \leq F(A)\;\forall A \subseteq V}.$}
%\end{equation*}
%The Lov\'asz extension $f_L$ is the support function of $B(F)$, i.e., $f_L(x) = \max_{s \in B(F)} \ip{x}{s}$.
\item \label{itm:greedy} %Greedy algorithm
Let $F$ be submodular and define its base polyhedron
\begin{equation*}
\resizebox{.95\hsize}{!}{$B(F):=\set{s \in \R^d}{s(V) = F(V),\; s(A) \leq F(A)\;\forall A \subseteq V}.$}
\end{equation*} 
Greedy algorithm: Given $x\in \R^d, \sigma \in S_d$  such that  $x_{\sigma(1)}\geq \dots \geq x_{\sigma(d)}$, define %$y\in \R^d$ 
$
y_{\sigma(k)}= F( \sigma(k) \mid S^\sigma_{k-1}),
$
then $y$ is a maximizer of $\max_{s \in B(F)} \ip{x}{s}$, $f_L$ is the support function of $B(F)$, i.e., $f_L(x) = \max_{s \in B(F)} \ip{x}{s}$, and $y$ is a subgradient of $f_L$ at $x$.
\item \label{itm:Lip} If $F$ is submodular, then $f_L$ is $\kappa$-Lipschitz, i.e., $|f_L(x) - f_L(y)| \leq \kappa \| x - y \|$ for all $x, y \in \R^d$, with $\kappa = 3 \max_{X \subseteq V} |F(X)|$. If $F$ is also non-decreasing, then $\kappa = F(V)$. 
\end{enumerate}
\end{proposition}
%\begin{proof}
%See \cite{Bach2013} (Prop. 3.1 for \cref{itm:extension,itm:sum}, Prop. 3.7
%for \cref{itm:equiv,itm:round}, Prop 3.6 and 3.7 for \cref{itm:conv,itm:greedy}) and \citep[Lemma 1]{Jegelka2011a} for \cref{itm:Lip}. %The claim therein is stated for submodular function but the proof holds for any set function 
%\end{proof}

%Properties \cref{itm:sum,itm:equiv,itm:conv,itm:Lip} 
These properties imply that Problem \eqref{eq:DS} is equivalent to 
\begin{equation}\label{eq:DS-DC}
\min_{x\in [0,1]^d} f_L(x) = g_L(x) - h_L(x),
\end{equation} 
 with $g_L, h_L \in \Gamma_0$. In paticular, if $X^*$ is a minimizer of \eqref{eq:DS},  then $1_{X^*}$ is a minimizer of \eqref{eq:DS-DC}, and if $x^*$ is a minimizer of \eqref{eq:DS-DC} then $\round(x^*)$ is a minimizer of \eqref{eq:DS}.

\paragraph{DC programming}  For a function $f: \R^d \to \rbar$, its domain is defined as $\dom f = \set{x \in \R^d}{f(x) < + \infty}$, and its Fenchel conjugate as $f^*(y) = \sup_{y \in \R^d} \ip{x}{y} - f(x)$. For $\rho\geq 0$, $f$ is $\rho$-strongly convex if $f - \tfrac{\rho}{2} \| \cdot \|^2$ is convex. We denote by $\rho(f)$ the supremum over such values. We say that $f$ is locally polyhedral convex if every point in its epigraph has a relative polyhedral neighbourhood \cite{Durier1988}.  
For a convex function $f, \epsilon \geq 0$ and $x^0 \in \dom f$, the $\epsilon$-subdifferential of $f$ at $x^0$ is defined by $\p_\epsilon f(x^0) = \set{ y \in \R^d }{f(x) \geq f(x^0)+ \ip{y}{x - x^0} - \epsilon, \forall x \in \R^d},$ while $\p f(x^0)$ stands for the exact subdifferential ($\epsilon=0)$.  We use the same notation to denote the $\epsilon$-superdifferential of a \emph{concave} function $f$ at $x^0$, defined by $\p_\epsilon f(x^0) = \set{ y \in \R^d }{f(x) \leq f(x^0)+ \ip{y}{x - x^0} + \epsilon, \forall x \in \R^d}.$ 
We also define $\dom \p_\epsilon f  =  \set{x \in \R^d}{\p_\epsilon f(x) \not = \emptyset}$.

 %Notethatforanyε≥0,onehasthefollowingrelationbetween
The $\epsilon$-subdifferential of a function $f \in \Gamma_0$ and its conjugate $f^*$ have the 
following relation \citep[Part II, Chap XI, Proposition 1.2.1]{Urruty1993}. % (proof in \cref{sec:FenchelDual-proof}). 
\begin{proposition} \label{prop:FenchelDualPairs}
For any $f \in \Gamma_0, \epsilon \geq 0$, we have
\[y \in \p_\epsilon f(x) \Leftrightarrow f^*(y) + f(x) - \ip{y}{x} \leq \epsilon \Leftrightarrow x \in \p_\epsilon f^*(y).\]
\end{proposition}

% \begin{equation*}
%\resizebox{.98\hsize}{!}{$\p_\epsilon f(x^0) = \set{ y \in \R^d }{f(x) \geq f(x^0)+ \ip{y}{x - x^0} - \epsilon, \forall x \in \R^d}.$}
%\end{equation*} 
A general DC program
%\footnote{Throughout we use the inf-addition rule 
%$
%+\infty-\infty=+\infty=-\infty+\infty.
%$} since in our case only g can take infinite values, no need to introduce this rule
takes the form 
\begin{equation}\label{eq:DC}
\min_{x\in \R^d} f(x):= g(x)-h(x)
\end{equation}
where $g,h\in \Gamma_0$.  We assume throughout the paper that the minimum of \eqref{eq:DC} is finite and denote it by $f^\star$.
The DC dual of \eqref{eq:DC} is given by \cite{Tao1997}
\begin{equation}\label{eq:DCdual}
f^* = \min_{y \in \R^d} h^*(y)-g^*(y).
\end{equation}

The main idea of DCA  is to approximate $h$ at each iteration $k$ by its affine minorization $h(x^k) + \ip{y^k}{x - x^k}$, with $y^{k}\in \p h(x^k)$, and minimize the resulting convex function. DCA can also be viewed as a primal-dual subgradient method. We give in \cref{alg:DCA} an approximate version of DCA with inexact iterates. Note that $ \p g^*(y^k) = \argmin_{x \in \R^d} g(x) - \ip{y^k}{x}$, and any $\epsilon$-solution $x^{k+1}$ to this problem will satisfy $x^{k+1}\in \p_{\epsilon_x} g^*(y^k)$, by \cref{prop:FenchelDualPairs}.
%This computation will not be exact in the case of Problem \eqref{eq:DS-DC}. We present an approximate version of DCA in \cref{alg:DCA}. 
\begin{algorithm}[h] 
\caption{Approximate DCA} \label{alg:DCA}
\begin{algorithmic}[1]
 \STATE  $\epsilon, \epsilon_x, \epsilon_y \geq 0, x^0 \in\dom \p g $, $k \gets 0$.
 \WHILE{$f(x^k) - f(x^{k+1}) > \epsilon$}
 \STATE $y^{k}\in \p_{\epsilon_y} h(x^k)$ 
 \STATE $x^{k+1}\in \p_{\epsilon_x} g^*(y^k)$
 \STATE $k\gets k+1$ 
 \ENDWHILE
\end{algorithmic}
 \end{algorithm}
 
The following lemma, which follows from \cref{prop:FenchelDualPairs}, provides a sufficient condition for DCA to be well defined, i.e, one can construct the sequences $\{x^k\}$ and $\{y^k\}$ from an arbitrary initial point $x^0 \in\dom \p g$.
\begin{lemma}\label{lem:wellDefined}
DCA is well defined if
$$\dom \p g \subseteq \dom \p h \text{ and } \dom \p h^* \subseteq \dom \p g^* $$
\end{lemma}

Since Problem \eqref{eq:DC} is non-convex, we are interested in notions of approximate stationarity.

\begin{definition}\label{def:criticality}
\looseness=-1 Let $g, h \in \Gamma_0$ and $\epsilon, \epsilon_1, \epsilon_2 \geq  0$, a point $x$ is an $(\epsilon_1, \epsilon_2)$-critical point of $g - h$ if $\partial_{\epsilon_1} g(x) \cap \partial_{\epsilon_2} h(x) \neq \emptyset$.
%for some $ \geq 0$. % such that $\epsilon_1 + \epsilon_2 = \epsilon$. 
Moreover, $x$ is an $\epsilon$-strong critical point of $g - h$ if $\partial h(x) \subseteq \partial_\epsilon g(x)$.
\end{definition}
\mtodo{Criticality definitions depend on the DC decomposition, i.e., given $\tilde{g}, \tilde{h} \in \Gamma_0$ such that $\tilde{g} - \tilde{h} = g - h$, a critical point of $\tilde{g} - \tilde{h}$ is not necessarily a critical point of $g - h$. 
Example given by George: $g = \max\{0,x\}, h = \max\{0,-x\}$ and $g' = \log(1+\exp(x)), h' = \log(1+\exp(-x))$. Clearly, $x=0$ is a critical point of $g-h$, but not a critical point of $g'-h'$ (in fact, $g' - h'$ has no critical points).
But if the function $e := \tilde{g} - g$ is in $\Gamma_0$ then critical points of both decompositions are the same, since $\p \tilde{g} = \p g + \p e$ if $\ri \dom  g \cap \ri \dom e \not = \emptyset$, which is not the case for approximate subdifferentials (see for example \citep[Part II, Chap XI, Theorem 3.1.1]{Urruty1993}).}

%It is well known that strong criticality ($\epsilon=0$) is a necessary condition for local minimality, and if $h$ is locally polyhedral convex (which holds for $h = h_L$) it becomes also a sufficient condition; see e.g., \citep[Proposition 3.1]{HiriartUrruty1989} and \citep[Corollary 2]{le1997solving}. We extend this result to %$\epsilon$-strong criticaly and $\epsilon$-local minimality 
%the approximate notions in \cref{prop:LocalOptCondition}.

%The following proposition extends necessary and sufficient conditions for local optimality based on criticality \citep[Theorem 4]{le1997solving}, \citep[Proposition 3.1]{HiriartUrruty1989}, \citep[Corollary 2]{le1997solving}, to the approximate notions.

Note that the definitions of criticality and strong criticality depend on the particular DC decomposition $g - h$ of $f$ \citep[Section 1.1]{le2018dc}. 
The two notions of criticality are equivalent when $h$ is differentiable and $\epsilon_1 = \epsilon, \epsilon_2 = 0$. \blue{When $\epsilon=0$, strong criticality is a necessary condition for local minimality \citep[Proposition 3.1]{HiriartUrruty1989}, and if $h$ is locally polyhedral convex, e.g., when $h = h_L$, it becomes a sufficient condition too \citep[Corollary 2]{le1997solving}. This relation breaks for $\epsilon>0$, %$\epsilon$-strong criticality, 
since $\epsilon$-local minimality is meaningless, as it holds for any point in $\dom g$. Yet $\epsilon$-strong criticality is still necessary for $\epsilon$-global minimality \citep[Proposition 3.2]{HiriartUrruty1989}, and it still implies $\epsilon$-minimality over a restricted set, as outlined in the following proposition.}
%The following proposition relates $\epsilon$-strong criticality to $\epsilon$-minimality over a restricted set.}
%provides necessary and sufficient conditions for approximate local optimality based on approximate criticality.

\begin{restatable}{proposition}{LocalOptCondition}\label{prop:LocalOptCondition} %[DC Local Optimality Condition]
Given $g, h \in \Gamma_0$ and $\epsilon \geq 0$, we have:\footnote{
\blue{\textbf{Erratum:} In the previous version of this paper, \Cref{prop:LocalOptCondition} included a wrong claim that $\epsilon$-strong criticality is necessary for $\epsilon$-local minimality, % (which is meaningless)
 and a vacuous claim that %this condition becomes sufficient
 $\epsilon$-strong criticality is sufficient for $\epsilon$-local minimality 
 when $h$ is locally polyhedral convex. These claims and related vacuous claims in \Cref{thm:ConvCDCA} and \Cref{corr:LocMinCDCA} have been omitted in this version. 
 These revisions do not impact any of the key results of the paper. }}
 \vspace{-0.15in}
\begin{enumerate}[label=\alph*), ref=\alph*]  %
\item \label{itm:general} 
Let $\hat{x}, x$ be two points satisfying $\partial_{\epsilon_1} g(\hat{x}) \cap \partial_{\epsilon_2} h(x) \not = \emptyset$, for some $\epsilon_1, \epsilon_2 \geq 0$ such that $\epsilon_1+ \epsilon_2 = \epsilon$, then $g(\hat{x})-h(\hat{x})\leq g(x)-h(x) + \epsilon$. 
\red{Moreover,  if $\hat{x}$ admits a neighbourhood $U$ such that $\partial_{\epsilon_1} g(\hat{x}) \cap \partial_{\epsilon_2} h(x) \not = \emptyset$ for all $x \in U \cap \dom g$, then $\hat{x}$ is an $\epsilon$-local minimum of $g-h$. Conversely, if $\hat{x}$ is an $\epsilon$-local minimum of $g-h$, then it is also an $\epsilon$-strong critical point of $g-h$.}
%$\partial h(\hat{x}) \subseteq \partial_\epsilon g(\hat{x})$. 
\item \blue{\label{itm:strongCriticalPt}Let $\hat{x}$ be an $\epsilon$-strong critical point of $g-h$, then  $g(\hat{x})-h(\hat{x})\leq g(x)-h(x) + \epsilon$ for all $x$ such that  $\partial h(\hat{x}) \cap \partial h(x) \not = \emptyset$.} \red{\label{itm:polyhedral}If $h$ is locally polyhedral convex, then this becomes a necessary and sufficient condition, i.e., $\hat{x}$ is an $\epsilon$-local minimum of $g-h$ if and only if it is an $\epsilon$-strong critical point of $g-h$.}
%$\partial h(\hat{x}) \subseteq \partial_\epsilon g(\hat{x})$.
\end{enumerate}
\end{restatable} 
\begin{proof}
\blue{\Cref{itm:general} is an extension of \citep[Theorem 4]{le1997solving}. %\citep[Corollary 3.3]{Tao1998}.
Given $y \in  \partial_{\epsilon_1} g(\hat{x}) \cap \partial_{\epsilon_2} h(x)$, we have $g(\hat{x}) + \ip{y}{x - \hat{x}} - \epsilon_1 \leq g(x)$ and $h(x) + \ip{y}{\hat{x} - x} - \epsilon_2 \leq h(\hat{x})$. Hence, $ g(\hat{x}) - h(\hat{x}) \leq g(x) - h(x) + \epsilon$. \Cref{itm:strongCriticalPt} then follows from the definition of an $\epsilon$-strong critical point.} %\Cref{itm:general} and
\red{This extends the conditions for $\epsilon=0$ in \citep[Theorem 4 and Corollary 2]{le1997solving} and \citep[Proposition 3.1]{HiriartUrruty1989} to $\epsilon\geq0$. The proof is given in \cref{sec:LocalOpt-proof}.}
\end{proof}
\mtodo{we don't really use item b anywhere, even for CDCA we use item a. But it might be nice to keep it anw to show relation between criticality and local minimality, specially if we can say that the points satisfying $\partial h(\hat{x}) \cap \partial h(x) \not = \emptyset$ form a neighborhood.}
%\Cref{prop:LocalOptCondition} is well known for $\epsilon=0$; see e.g., \cite[Proposition 3.1]{HiriartUrruty1989} and \cite[Corollary 2]{le1997solving}.
DCA converges in objective values, and in iterates if $g$ or $h$ is strongly convex, to a critical point \citep[Theorem 3]{Tao1997}. We can always make the DC components strongly convex by adding $ \tfrac{\rho}{2} \| \cdot \|^2$ to both $g$ and $h$.
A special instance of DCA, called complete DCA, converges to a strong critical point, but requires solving concave minimization subproblems \citep[Theorem 3]{tao1988duality}.
CDCA picks valid DCA iterates $y^k, x^{k+1}$ that minimize the dual and primal DC objectives, respectively. %, i.e., it makes the following updates at each iteration:
We consider an approximate version of CDCA with the following iterates.  
%\vspace{-2pt}
\begin{subequations}\label{eq:CDCA}
\begin{align}
    y^k &\in \argmin \{ h^*(y) - g^*(y) : y \in \partial h(x^k)\} \nonumber \\ %\mathcal{S}(x^k) &:= 
   & = \argmin \{ \ip{y}{x^k} - g^*(y) : y \in \partial h(x^k)\}, \label{eq:CDCA-y} \\
   x^{k+1} &\in \argmin \{ g(x) - h(x) : x \in \partial_{\epsilon_x} g^*(y^k)\} \nonumber \\ % \mathcal{T}(x^k) &:=  
    &= \argmin \{ \ip{x}{y^k} - h(x) : x \in \partial_{\epsilon_x} g^*(y^k) \}.  \label{eq:CDCA-x} \\ \nonumber 
\end{align}
\end{subequations}

% \begin{equation*}
%\resizebox{.98\hsize}{!}{$\p_\epsilon f(x^0) = \set{ y \in \R^d }{f(x) \geq f(x^0)+ \ip{y}{x - x^0} - \epsilon, \forall x \in \R^d}.$}
%\end{equation*} 
%To make the DC components strongly convex we can add $ \tfrac{\rho}{2} \| \cdot \|^2$ to both $g$ and $h$. This can lead to better performance.

\section{DS Minimization via DCA}\label{sec:DCA}

In this section, we apply DCA to the DC program \eqref{eq:DS-DC} corresponding to DS minimization. 
We consider the DC decomposition $f = g- h$, where 
\begin{equation}\label{eq:DS-DCdecomposition}
g = g_L+\delta_{[0,1]^d}+ \tfrac{\rho}{2} \| \cdot \|^2 \text{ and } h=h_L + \tfrac{\rho}{2} \| \cdot \|^2,
\end{equation}
with $\rho \geq 0$. Starting from $x^0 \in [0,1]^d$, the approximate DCA iterates (with $\epsilon_y = 0$) are then given by 
\begin{subequations}\label{eq:DCASet}
\begin{align}
& y^k   \in  \rho x^k  + \p h_L(x^k), \label{eq:DCASet-y} \\
& x^{k+1} \text{ is an $\epsilon_x$-solution of }  \nonumber \\
& \!\!\!\! \min_{x \in [0,1]^d} g_L(x) - \ip{x}{y^k} + \tfrac{\rho}{2} \| x \|^2  \label{eq:DCASet-x}
\end{align}
%\begin{align}
%y^k  & \in  \rho x^k  + \p h_L(x^k), \label{eq:DCASet-y} \\
%x^{k+1} & \in \text{$\epsilon_x$-solutions of } % \nonumber \\
%\min_{x \in [0,1]^d} g_L(x) - \ip{x}{y^k} + \tfrac{\rho}{2} \| x \|^2  \label{eq:DCASet-x}
%\end{align}
\end{subequations}

Note that the minimum $f^* = F^*$ of \eqref{eq:DS-DC} is finite, since $f$ is finite. DCA is clearly well defined here; we discuss below how to obtain the iterates efficiently. One can also verify that the condition in \cref{lem:wellDefined} holds: $\dom \p g = [0,1]^d \subseteq \dom \p h = \R^d$ by \cref{prop:LEproperties}-\ref{itm:greedy}, and $\dom \p h^* = B(H)$ if $\rho = 0$, $\R^d$ otherwise, hence in both cases $\dom \p h^* \subseteq \dom \p g^* = \R^d$, by  \cref{prop:LEproperties}-\ref{itm:sum},\ref{itm:equiv}.

\paragraph{Computational complexity} \label{sec:DCA-complexity}
A subgradient of $h_L$ can be computed as described in \cref{prop:LEproperties}-\ref{itm:greedy} in $O(d \log d + d ~\text{EO}_H)$ with $\text{EO}_H$ being the time needed to evaluate $H$ on any set. 
An $\epsilon_x$-solution of Problem \eqref{eq:DCASet-x}, for $\epsilon_x>0$, can be computed using the projected subgradient method (PGM) in $O(d \kappa^2 / \epsilon_x^2)$ iterations when $\rho=0$ and in $O(2 (\kappa + \rho \sqrt{d})^2 / \rho \epsilon_x)$ when $\rho>0$ \citep[Theorems 3.1 and 3.5]{Bubeck2014}, where $\kappa$ is the Lipschitz constant of $ g_L(x) - \ip{x}{y^k}$; see \cref{prop:LEproperties}-\ref{itm:Lip}. The time per iteration of PGM is $O(d \log d + d ~\text{EO}_G)$. 
%Note that, if $x^{k+1}$ is an $\epsilon_x$-solution of Problem \eqref{eq:DCASet-x}, then $x^{k+1} \in \p_{\epsilon_x} g^*(y^k)$, by \cref{prop:FenchelDualPairs}.

%. The convergence rate of PGM when $\rho=0$ is $D L /\sqrt{T}$ where $D = \sqrt{d}$ is the diameter of $[0,1]^d$, and $L \leq G(V) + \| y^k\|_2$; see e.g., \citep[Theorem 3.1]{Bubeck2014}, while the convergence rate when $\rho > 0$ is $2 L^2 / \rho (T + 1)$, with $L = G(V) + \| y^k\|_2 +  \rho D$; see for e.g., \cite[Theorem 3.5]{Bubeck2014}. 
%which consists of updates $\tilde{x}^{t+1} = \Pi_{[0,1]^d}(\tilde{x}^{t+1} - \eta \tilde{y}^t)$, where $ \tilde{y}^t \in \p g_L(x) - \ip{x}{y^k}$, and $\Pi_{[0,1]^d}$ is the projection onto $[0,1]^d$.

When $\rho=0$, %a subgradient of $g^*(y^k)$ can also be computed by 
Problem \eqref{eq:DCASet-x} is equivalent to a submodular minimization problem, since $\min_{x \in [0,1]^d} g_L(x) - \ip{x}{y^k} = \min_{X \subseteq V} G(X) - y^k(X)$ by \cref{prop:LEproperties}-\ref{itm:sum},\ref{itm:equiv}. Then we can take $x^{k+1} = \1_{X^{k+1}}$ where $X^{k+1} \in \argmin_{X \subseteq V} G(X) - y^k(X)$.
Several algorithms have been developed for minimizing a submodular function in polynomial time, exactly or within arbitrary accuracy $\epsilon_x > 0$. 
%These include combinatorial algorithms \cite{Schrijver2000, Iwata2001, Iwata2009}, and algorithms based on convex optimization \cite{Fujishige2011, Bach2013, Lee2015, Chakrabarty2017, Axelrod2019}.
Inexact algorithms are more efficient, with the current best runtime $\tilde{O}(d ~\text{EO}_G/\epsilon_x^2)$ achieved by \cite{Axelrod2019}. %, with EO being the time needed to evaluate $F$ on any set.
%the algorithm proposed in \cite{Axelrod2019} based on stochastic projected subgradient method (PGM) 
In this case, DCA reduces to the SubSup procedure of \cite{Narasimhan2005a} and thus satisfies the same theoretical guarantees; see \cref{sec:SubSupDCA}.

In what follows, we extend these guarantees to the general case where $x^k$ is not integral and $\rho \geq 0$, by leveraging convergence properties of DCA.

\paragraph{Theoretical guarantees} 
Existing convergence results of DCA in \cite{Tao1997, le1997solving, An2005} consider exact iterates and exact convergence, i.e., $f(x^k) = f(x^{k+1})$, which may require an exponential number of iterations, as shown in \citep[Theorem 3.4]{Byrnes2015} for SubSup.
%show that DCA obtains a critical point at exact convergence%, which is a special case of DCA.
%We extend these results showing that DCA obtains an approximate critical point at approximate convergence, i.e., $f(x^k) - f(x^{k+1}) \leq \epsilon$.
We extend these results to handle inexact iterates and approximate convergence. 
%Theoretical results about DCA were developed in \cite{Tao1997, An2005}. We are interested in the convergence  We extend these results to 

\begin{theorem}\label{thm:convergence} %[Convergence of approximate DCA]
\looseness=-1 Given any $f=g-h$, where $g,h\in \Gamma_0$, 
let $\{x^k\}$ and $\{y^k\}$ be generated by approximate DCA (\Cref{alg:DCA}). %where $y^k \in \p_{\epsilon_y} h(x^k)$ and $x^{k+1} \in \p_{\epsilon_x} g^*(y^k)$  and for some $\epsilon_x, \epsilon_y \geq 0$, 
%and $f^\star = \min_{x \in \R^d} f(x)$.
%, and let $T_g(x^{k+1}) = g(x^k) - g(x^{k+1}) -  \langle y^k, x^k - x^{k+1} \rangle$  and $T_h(x^{k+1}) = h(x^k) - h(x^{k+1}) -  \langle y^k, x^k - x^{k+1} \rangle$.
% We say that $x^k$ is an $\epsilon$-critical point if it satisfies $\min\{T_g(x^{k+1}), -T_h(x^{k})\} \leq \epsilon$. 
Then for all $t_x, t_y \in (0,1], k\in\bN$, let $\bar{\rho} = \rho(g)(1-t_x)+\rho(h)(1-t_y)$ and $\bar{\epsilon} =  \tfrac{\epsilon_x}{t_x} + \tfrac{\epsilon_y}{t_y}$, we have:
\begin{enumerate}[label=\alph*),ref=\alph*]  %[leftmargin=1em, itemindent=1em]   
\item \label{itm:descent} \(\begin{aligned} f(x^k) - f(x^{k+1}) \geq \frac{\bar{\rho}}{2}\|x^k-x^{k+1}\|^2 - \bar{\epsilon} \end{aligned} \).
    
\item \label{itm:convergence} For $\epsilon \geq 0$, if $f(x^k) - f(x^{k+1}) \leq \epsilon$, then $x^k$ is an $(\epsilon',  \epsilon_y)$-critical point of $g - h$ with $y^k \in \partial_{\epsilon'} g(x^k) \cap \p_{\epsilon_y} h(x^k)$, $x^{k+1}$ is an $(\epsilon_x, \epsilon')$-critical point of $g - h$ with $y^k \in \p_{\epsilon_x} g(x^{k+1}) \cap \p_{\epsilon'} h(x^{k+1})$, where $\epsilon' = \epsilon +  \epsilon_x + \epsilon_y$, %for some $\epsilon_1 + \epsilon_2 = \epsilon, \epsilon_1 \geq - \epsilon_x, \epsilon_2 \geq  -\epsilon_y$, 
and $\frac{\bar{\rho}}{2}\|x^k-x^{k+1}\|^2 \leq \bar{\epsilon}  + \epsilon$. % for all $t_x, t_y \in (0,1], k\in\bN$.
%Conversely, if  $x^k \in\p_{\epsilon_x + \epsilon_1} g^*(y^k),\;y^k\in \p_{\epsilon_y + \epsilon_2} h(x^{k+1})$ then $f(x^k) - f(x^{k+1})  \leq \epsilon_x + \epsilon_y + \epsilon$. <-- we don't really use the converse anywhere..
% using g^* in converse as it's more intuitive what this means, if x^k, y^k are valid solutions for next iterations then we have converged..

\item \label{itm:obj-rate} \(\begin{aligned}\min_{k \in \{0,1,\dots,K-1\}} f(x^k)-f(x^{k+1}) \leq \frac{f(x^0) - f^\star}{K}. \end{aligned} \)

\item \label{itm:iterates-rate} If $\rho(g) + \rho(h) >0$, then \hspace{-5pt} $$ \min_{k \in \{0,1,\dots,K-1\}} \|x^k-x^{k+1}\|\leq \sqrt{\frac{ 2}{\bar{\rho}} \big( \frac{f(x^0) - f^\star}{K} + \bar{\epsilon} \big) } .$$
%\(\begin{aligned} \end{aligned} \).

%\item \label{itm:complexity} If we use $f(x^k) - f(x^{k+1}) \leq \epsilon$ as a stopping criterion,  DCA stops after ${(f(x^0) - f^\star})~/~{\epsilon}$ iterations. If we use $f(x^k) (1 + \epsilon) \leq f(x^{k+1})$ as a stopping criterion and start with $f(x^0) \leq 0$,  DCA stops after $O\big({\log(|f^\star|/|f(x^1)|)} ~/~ {\epsilon}\big)$ iterations.

%\item \label{itm:finite} If $\epsilon'=0$, $g$ or $h$ is polyhedral convex, and we use $f(x^k) = f(x^{k+1})$ as a stopping criterion, DCA with fixed choice of subgradients stops after finitely many iterations. 
\end{enumerate}
\end{theorem} 
\begin{proof}[Proof sketch] 
\looseness=-1 \Cref{itm:descent,itm:convergence} with $\epsilon\!=\!\epsilon_x\!=\!\epsilon_y\!=\!0$ are proved in \citep[Theorem 3]{Tao1997}. %, and \cref{itm:finite} in \citep[Theorem 8]{le1997solving}. 
We extend them to $\epsilon, \epsilon_x, \epsilon_y \!\geq\! 0$ by leveraging properties of approximate subgradients.
%a property of $\epsilon'$-subgradients of a strongly convex function (see \cref{lem:eps-subdiff-strcvx}). 
\Cref{itm:obj-rate} is obtained by telescoping sum. %\Cref{itm:iterates-rate} follows from \cref{itm:descent,itm:obj-rate}.
%The full proof is provided in \cref{sec:DCAconv-proof}.
\end{proof}
\looseness=-1  \Cref{thm:convergence} shows that approximate DCA decreases the objective $f$ almost monotonically (up to $\bar{\epsilon}$), and converges in objective values with rate $O(1/k)$, and in iterates with rate $O(1/\sqrt{k})$ if $\rho>0$, to an approximate critical point of $g - h$.

\looseness=-1 We present in \cref{sec:DCAconv-proof} a more detailed version of \cref{thm:convergence} and its full proof. In particular, we relate $f(x^k) - f(x^{k+1})$ to a weaker measure of non-criticality, recovering the convergence rate provided in \citep[Corollary 4.1]{abbaszadehpeivasti2021rate} on this measure.
Approximate DCA with $\epsilon=0, \epsilon_x = \epsilon_y \geq 0$ was considered in \citep[Theorem 1.4]{Vo2015} showing that any limit points $\hat{x}, \hat{y}$ of $\{x^k\}, \{y^k\}$  satisfy $\hat{y} \in  \p_{3\epsilon_x} g(\hat{x})  \cap \p_{\epsilon_x} h(\hat{x})$ in this case. Our results are more general and tighter (at convergence $y^K \in \partial_{2 \epsilon_x} g(x^K) \cap \p_{\epsilon_x} h(x^K)$ in this case). 
For DS minimization, $y^k$ can be easily computed exactly ($\epsilon_y=0$).  We consider $\epsilon_y>0$ to provide convergence results of FW on the concave subproblem required in CDCA (see Section \ref{sec:CDCA}).
%as it will be needed later for results of FW..

The following corollary relates criticality on the DC problem \eqref{eq:DS-DC} to local minimality on the DS problem \eqref{eq:DS}.

\begin{restatable}{corollary}{LocMinDCA} \label{corr:LocMinDCA} 
%For $g, h$ given by \eqref{eq:DS-DCdecomposition},
Given $f= g -h$ as defined in \eqref{eq:DS-DCdecomposition}, let $\{x^k\}$ and $\{y^k\}$ be generated by a variant of approximate DCA \eqref{eq:DCASet}, where %$x^{k+1} \in \p_{\epsilon_x} g^*(y^k)$ 
$x^{k}$ is integral, i.e.,  $x^{k} = \1_{X^{k}}$ for some $X^{k} \subseteq V$, and $y^k - \rho x^k$ is computed as in \cref{prop:LEproperties}-\ref{itm:greedy}. Then for all $k \in \bN, \epsilon \geq 0$, we have 
\begin{enumerate}[label=\alph*),ref=\alph*]  %[leftmargin=1em, itemindent=1em]   
\item \label{itm:noPerms} If 
$f(x^k) - f(x^{k+1}) \leq \epsilon$, then
\begin{equation}\label{eq:ChainMin}
F(X^k) \leq F(S^{\sigma}_\ell) + \epsilon' \text{ for all $\ell \in V$},
\end{equation}
where 
\begin{equation}\label{eq:epsilon'-main}
\epsilon' = \begin{cases} \sqrt{2 \rho d (\epsilon + \epsilon_x)} &\text{ if $\epsilon + \epsilon_x \leq \tfrac{\rho d}{2}$} \\ \tfrac{\rho d}{2} + \epsilon + \epsilon_x &\text{otherwise}.
\end{cases}
\end{equation}
and $\sigma \in S_d$ is the permutation used to compute $y^k - \rho x^k$ in \cref{prop:LEproperties}-\ref{itm:greedy}.
%$\epsilon' = \sqrt{2 \rho d (\epsilon + \epsilon_x)}$ if $\epsilon + \epsilon_x \leq \tfrac{\rho d}{2}$ and $\tfrac{\rho d}{2} + \epsilon + \epsilon_x$ otherwise.
\item \label{itm:Perms} %If $x^{k+1}$ is  chosen in the following way:
Given $d$ permutations $\sigma_1, \cdots, \sigma_d \in S_d$, corresponding to decreasing orders of $x^k$ with different elements at $\sigma(|X^k|)$ or $\sigma(|X^k|+1)$, and the corresponding subgradients $y^k_{\sigma_1}, \cdots, y^k_{\sigma_d} \in \p h(x^k)$ chosen as in \cref{prop:LEproperties}-\ref{itm:greedy}, 
\looseness=-1 if we choose $$x^{k+1} \in \argmin \{ f(x^{k+1}_{\sigma_{i}}) : x^{k+1}_{\sigma_{i}} \in \p_{\epsilon_x} g^*(y^k_{\sigma_{i}}), i \in V \},$$ %for some $\epsilon' \geq 0$. 
 then if $f(x^k) - f(x^{k+1}) \leq \epsilon$, \cref{eq:ChainMin} holds with $\sigma = \sigma_i$ for all $i \in V$. Hence, $X^k$ is an $\epsilon'$-local minimum of $F$.
 %For $\epsilon \geq 0$, if 
%$f(x^k) - f(x^{k+1}) \leq \epsilon$, then $X^k$ is an $\epsilon''$-local minimum of $F$.
%, where $\epsilon'' = \sqrt{2 \rho d (\epsilon + \epsilon')}$ if $\epsilon + \epsilon' \leq \tfrac{\rho d}{2}$ and $\tfrac{\rho d}{2} + \epsilon + \epsilon'$ otherwise.
\end{enumerate}
\end{restatable}
\begin{proof}[Proof sketch] \looseness=-1
 We observe that $y^k - \rho x^k \in \p h_L(\1_{S^{\sigma}_\ell})$ for all $\ell \in V$. 
% Similarly, for $x = \1_{X^k \setminus i}, i \in X^k$ and $x = \1_{X^k \cup i}, i \in V \setminus X^k$, there exists  $j \in V$ such that $y^k_{\sigma_j} - \rho x^k \in \p h_L(x)$. %Similarly for any $x = \1_{X^k \cup i}, i \in V \setminus X^k$.
\Cref{itm:noPerms} then follows from \cref{thm:convergence}-\ref{itm:convergence}, \cref{prop:LEproperties}-\ref{itm:extension},\ref{itm:greedy}, \cref{prop:LocalOptCondition}-\ref{itm:general}, and the relation between the  $\epsilon$-subdifferentials of $g$ and $g - \frac{\rho}{2} \| \cdot \|^2$. 
\Cref{itm:Perms} follows from \cref{itm:noPerms}.
 %by observing that for any $x = \1_{X^k \setminus i}, i \in X^k$ and any $x = \1_{X^k \cup i}, i \in V \setminus X^k$, there exists $j \in V$, such that $S^{\sigma_j}_{|X^k|-1} = X^k \setminus i$ and $S^{\sigma_j}_{|X^k|+1} = X^k \cup i$ respectively.
 See \cref{sec:LocMinDCA-proof}. % for details.
\end{proof}
\looseness=-1 \Cref{thm:convergence} and \cref{corr:LocMinDCA} show that DCA with integral iterates $x^k$  decreases the objective $F$ almost monotonically (up to $\bar{\epsilon}$), and
converges to an $\epsilon'$-local minimum of $F$ after at most $(f(x^0) - f^\star)/\epsilon$ iterations, if we consider $O(d)$ permutations for computing $y^k$.
By a similar argument, we can further guarantee that the returned solution cannot be improved, by more than $\epsilon'$, by adding or removing any $c$ elements, if we consider $O(d^c)$ permutations for computing $y^k$. % then pick the best.

Taking $\epsilon_x=0, \rho=0$ in \cref{thm:convergence} and \cref{corr:LocMinDCA}, we recover all the theoretical properties of SubSup given in \cite{Narasimhan2005a,Iyer2012a}. %; namely that SubSup monotonically reduces the objective value $F$ at each iteration and converges to a local minimum of $F$.
%, and show in addition that the objective values converge with $O(1/k)$ rate. => I'll skip mentioning this since the value of this result is showing a computation bound which is already done in SubSup (multiplicative bound in e)

\paragraph{Effect of regularization}\label{sec:reg}

\Cref{thm:convergence} shows that using a non-zero regularization parameter  $\rho >0$ ensures convergence in iterates. Regularization also affects the complexity of solving Problem \eqref{eq:DCASet-x}; as discussed earlier $\rho>0$ leads to a faster convergence rate (except for very small $\rho$). On the other hand, \Cref{corr:LocMinDCA} shows that for fixed $\epsilon$ and $\epsilon_x$, a larger $\rho$ may lead to a poorer solution. In practice, we observe that a larger $\rho$ leads to slower convergence in objective values $f(x^k)$, but more accurate $x^k$ iterates, with $\rho>0$ always yielding the best performance with respect to $F$ (see \cref{sec:regEffect}). 
%, which ensures monotonically decreasing objective values .

%how the regularization parameter $\rho$ affects the convergence of DCA. 
%seem to suggest that choosing a larger regularization parameter $\rho$ in regularized DCA leads to larger descent per iteration, but that's not necessarily the case since $\|x^k - x^{k+1}\|$ could be smaller for larger $\rho$.

Note that when $\rho >0$ we can't restrict $x^k$ to be integral, since the equivalence in \cref{prop:LEproperties}-\ref{itm:equiv} does not hold in this case. It may also be advantageous to not restrict $x^k$ to be integral even when $\rho=0$, as we observe in our numerical results (\cref{sec:SubSupDCAexp}). 
%<-- Note that this is true for DCA (even without local min) but not DCAR 
A natural question arises here: can we still obtain an approximate local minimum of $F$ in this case? Given a fractional solution $x^K$ returned by DCA we can easily obtain a set solution with a smaller objective $F(X^K) = f_L(\1_{X^K})\leq f_L(x^K)$ by rounding; $X^K = \round(x^K)$ as described in \cref{prop:LEproperties}-\ref{itm:round}.  %In practice, this strategy seems to work well \ref{}. <-- e.g., in speech-lbd1.0e+00-7fc4444-48639887 (no need to mention this though)
%in general 
However, rounding a fractional solution $x^K$ returned by DCA will not necessarily yield an approximate local minimum of $F$, even if $x^K$ is a local minimum of $f_L$, as we show in \cref{ex:localMinRound}. 
A simple workaround would be to explicitly check if the rounded solution is an $\epsilon'$-local minimum of $F$. If not, we can restart the algorithm from $x^{K} = \1_{\hat{X}^{K}}$ where $\hat{X}^{K} = \argmin_{|X \Delta X^{K}| = 1} F(X)$, similarly to what was proposed in \citep[Algorithm 1]{Byrnes2015} for SubSup. This will guarantee that DCA converges to an $\epsilon'$-local minimum of $F$ after at most $(f(x^0) - f^\star)/\epsilon$ iterations (see \cref{prop:convergence-localmin}). Such strategy is not feasible though if we want to guarantee convergence to an approximate strong local minimum of $F$, as we do in \cref{sec:CDCA} with CDCA. 
%\mtodo{add comment here about this making a difference or not in practice} <-- it does it significantly improves performance of baselines, and a bit our methods if we use obj stopping criterion, but not if we use iterates stopping criterion since in that case our methods often do not converge.. 
We thus propose an alternative approach.  We introduce a variant of DCA, which we call \DCAR, where we round $x^k$ at each iteration.
%To circumvent this issue, we propose a  variant of DCA, which we call \DCAR, where we round $x^k$ at each iteration.

\paragraph{DCA with rounding}  Starting from $x^0 \in \{0,1\}^d$, the approximate \DCAR iterates are  given by 
\begin{subequations}\label{eq:DCAround}
\begin{align}
&y^k, \tilde{x}^{k+1} \text{ as in \eqref{eq:DCASet-y} and  \eqref{eq:DCASet-x} respectively,} \\
%y^k & \in  \rho x^k  + \p h_L(x^k),  \label{eq:DCAround-y} \\ 
%\tilde{x}^{k+1} & \in  \argmin_{x \in [0,1]^d} g_L(x) - \ip{x}{y^k} + \tfrac{\rho}{2} \| x \|^2, \label{eq:DCAround-x} \\ 
&x^{k+1} \gets \1_{X^{k+1}} \text{ where } X^{k+1}=\round(\tilde{x}^{k+1}).
%\text{where $X^{k+1}$ is obtained by rounding $\tilde{x}^{k+1}$ via \cref{prop:LEproperties}-\ref{itm:round}}
\end{align}
\end{subequations}

Since $y^k, \tilde{x}^{k+1}$ are standard approximate DCA iterates, then the properties in \cref{thm:convergence} apply to them, with $\epsilon_y = 0$ and $x^{k+1}$ replaced by $\tilde{x}^{k+1}$. %even for items c-f since c follows by telescoping sum, d from c and a, e from c and a similar argument in multiplicative fashion, f uses the fact that with rounding we get $f(x^{k+1}) \leq f(\tilde x^{k+1}) \leq f(x^k) \leq  f(\tilde x^{k})$.
%By \cref{prop:LEproperties}-\ref{itm:extension},\ref{itm:round}, the same properties also hold with $f(x^k) - f(\tilde x^{k+1})$ replaced by $F(X^k) - F(X^{k+1})$. % (except for the converse direction in item \ref{itm:convergence}). %Moreover, finite convergence (item \ref{itm:finite}) holds here without requiring fixed choices of subgradients. 
See \cref{them:convergence-round-app} for details.
%The following convergence properties then easily follow from \cref{thm:convergence} and \cref{prop:LEproperties}-\ref{itm:extension},\ref{itm:round}.
Since $x^k$ is integral in \DCAR, \cref{corr:LocMinDCA} also holds. In particular, \DCAR 
converges to an $\epsilon'$-local minimum of $F$ after at most $(f(x^0) - f^\star)/\epsilon$ iterations, if we consider $O(d)$ permutations for computing $y^k$, with $\epsilon'$ defined in \eqref{eq:epsilon'-main}.
%In particular, if we consider $O(d)$ permutations for choosing $y^k$, then pick the one yielding the best objective $F(X^{k+1})$,  then at convergence $F(X^k) - F(X^{k+1}) \leq \epsilon$, $X^k$ is an $\epsilon'$-local minimum of $F$.

\section{DS Minimization via CDCA}\label{sec:CDCA}

As discussed in \cref{sec:prelim}, CDCA is a special instance of DCA which is guaranteed to converge to a strong critical point. 
%In the next section, we show that  reasonable to expect that the stronger guarantee on the DC program \ref{eq:DS-DC} translates into a stronger guarantee 
In this section, we apply CDCA to the DC program \eqref{eq:DS-DC} corresponding to DS minimization, and show that the stronger guarantee on the DC program translates into a stronger guarantee on the DS problem. We use the same decomposition in \eqref{eq:DS-DCdecomposition}. %as in \cref{sec:DCA}. 

\paragraph{Computational complexity} \looseness=-1 CDCA requires solving a concave minimization problem for each iterate update. The constraint polytope $\partial h(x^k) = \rho x^k  + \p h_L(x^k)$ in Problem \eqref{eq:CDCA-y} can have a number of vertices growing exponentially with the number of equal entries in $x^k$. Thus, it is not possible to efficiently obtain a global solution of Problem \eqref{eq:CDCA-y} in general. However, we can efficiently obtain an approximate critical point.  Denote the objective
\begin{equation}\label{eq:phik}
\phi_k(w) = \ip{w}{x^k} - g^*(w).
\end{equation}
We use an approximate version of the FW algorithm, which 
starting from $w^0 \in \p h(x^k)$, has the following iterates:
%the approximate FW iterates are given by:
%Starting from $w^0 \in \rho x^k  + \p h_L(x^k)$, FW iterates are given by:
%. 
 \begin{subequations}\label{eq:FW-concave}
\begin{align}
  s^{t} &\in \p_{\epsilon} \phi_k(w^t) \supseteq  x^k - \p_\epsilon g^*(w^t),\\
v^t &\in \argmin \{ \ip{s^{t}}{w} : w \in \partial h(x^k)\}, \label{eq:FW-LO} \\
w^{t+1} &= (1 - \gamma_t) w^t + \gamma_t v^t, 
\end{align}
\end{subequations}
where $\epsilon \geq 0$ and we use the greedy step size $\gamma_t = \argmin_{\gamma \in [0,1]} \phi_k((1 - \gamma) w^t + \gamma v^t) = 1$.
%For concave objectives, 
We observe that with this step size, FW is a special case of DCA (with DC components $g' = \delta_{\p h(x^k)}$ and $h' = - \phi_k$). Hence, \cref{thm:convergence} applies to it (with $\epsilon_x=0, \epsilon_y=\epsilon$). In paticular, FW converges to a critical point with rate $O(1/t)$. Convergence results of FW for nonconvex problems are often presented in terms of the FW gap defined as $\mathrm{gap}(w^t) := \max_{w \in \p h(x^k)} \ip{s^t}{w^t - w}$ \cite{LacosteJulien2016}. Our results imply the following bound on the FW gap (see \cref{sec:FWconv-proof} for details).

%FW converges to a stationary point with rate $O(1/k)$ \citep[Lemma 2.1]{Yurtsever2022}\footnote{The result therein is stated for $\phi_k$ continuously differentiable, but it does not actually require differentiability.}.  We extend this result to handle approximate supergradients of $\phi_k$. Non-stationarity can be measured by 

 \begin{restatable}{corollary}{FWconvergence} \label{corr:FWconvergence}
Given any $f=g-h$, where $g,h\in \Gamma_0$, and $\phi_k$ as defined in \eqref{eq:phik},
let $\{w^{t} \}$ be generated by approximate FW \eqref{eq:FW-concave} %where $s^{t} \in \p_\epsilon \phi_k(w^t)$ for some $\epsilon \geq 0$, 
with $\gamma_t=1$. Then for all $T \in \bN$, we have
\[ \min_{t \in \{0, \cdots, T-1\} } \mathrm{gap}(w^t) \leq \frac{\phi_k(w^0) - \min_{ w \in \partial h(x^k)} \phi_k(w)}{T} + \epsilon\]
 \end{restatable}
% \begin{proof}[Proof sketch]
%This result follows from \cref{thm:convergence}-\ref{itm:obj-rate} by observing that FW with step size $\gamma_t=1$ is a special case of DCA. See \cref{sec:FWconv-proof} for the detailed proof. %convergence properties of DCA 
% \end{proof}
\cref{corr:FWconvergence} extends the result of \citep[Lemma 2.1]{Yurtsever2022}\footnote{The result therein is stated for $\phi_k$ continuously differentiable, but it does not actually require differentiability.} to handle approximate supergradients of $\phi_k$.
A subgradient of $h_L$ and an approximate subgradient of $g^*$ can be computed as discussed in \cref{sec:DCA}. The following proposition shows that the linear minimization problem \eqref{eq:FW-LO} can be exactly solved in $O(d \log d + d ~\text{EO}_H)$ time. %, i.e., similar cost needed for $\p h_L$.

 \begin{restatable}{proposition}{FWLO} \label{prop:FWLO}
 Given $s, x \in \R^d$, let $a_1 >  \cdots > a_m$ denote the unique values of $x$ taken at sets $A_1 \cdots, A_m$, i.e., $A_1 \cup \cdots \cup A_m = V$ and for all $ i \in \{1, \cdots, m\},  j \in A_i$, $x_j = a_i$, and 
% Given a vector $x \in \R^d$ with unique values $a_1 >  \cdots > a_m$ taken at sets $A_1, \cdots, A_m$, i.e., $\forall i \in \{1, \cdots, m\}, \forall k \in A_i, x_k = a_i$, and $s \in \R^d$, 
let  $\sigma \in S_d$ be a decreasing order of $x$, where we break ties according to $s$, i.e., $x_{\sigma(1)} \geq \cdots \geq x_{\sigma(d)}$ and $s_{\sigma(|C_{i-1}| + 1)} \geq \cdots \geq s_{\sigma(|C_{i}|)}$, where $C_i = A_1 \cup \cdots \cup A_i$ for all $i \in \{1, \cdots, m\}$.
Define $w_{\sigma(k)} = H(\sigma(k) \mid S^\sigma_{k-1})$ for all $k \in V$, then $w$ is a maximizer of $\max_{w \in \p h_L(x)} \ip{s}{w}$.
 \end{restatable}
\begin{proof}[Proof sketch]
By \cref{prop:LEproperties}-\ref{itm:greedy}, we have that $w \in \p h_L(x)$ and that any feasible solution 
%$w' \in \p h_L(x)$, $w'$ 
is a maximizer of $\max_{w \in  B(H)} \ip{w}{s}$. The claim then follows by the optimality conditions of this problem given in \citep[Proposition 4.2]{Bach2013}.
% , hence it must satisfy $w'(C_i) = H(C_i)$ for all $i \in \{1, \cdots, m\}$ 
The full proof is in \cref{sec:FWLO-proof}.
\end{proof}

Note that Problem \eqref{eq:CDCA-x} reduces to a unique solution $x^{k+1} = \nabla g^*(y^k)$ when $\rho>0$, since $g^*$ is differentiable in this case. When $\rho=0$, the constraint $\p g^*(y^k) = \argmin_{x \in [0,1]^d} g_L(x) - \ip{y^k}{x}$ is the convex hull of minimizers of $g_L(x) - \ip{y^k}{x}$ on $\{0,1\}^d$ \citep[Proposition 3.7]{Bach2013}, which can be exponentially many. One such trivial example is when the objective is zero so that the set of minimizers is $\{0,1\}^d$, in which case Problem \eqref{eq:CDCA-x} is as challenging as the original DC problem. Fortunately, in what follows we show that solving Problem \eqref{eq:CDCA-x} is not necessary to obtain an approximate strong local minimum of $F$; it is enough to pick any approximate subgradient of $g^*(y^k)$ as in DCA. %$x^{k+1} = \p_{\epsilon'} g^*(y^k)$ 
% it is however needed to obtain a strong critical point of the dual DC problem

\looseness=-1 \paragraph{Theoretical guarantees} Since CDCA is a special case of DCA,  all the guarantees discussed in \cref{sec:DCA} 
%properties outlined in \ref{thm:convergence} 
apply. In addition, CDCA is known to converge to a strong critical point \citep[Theorem 3]{tao1988duality}. We extend this to the variant with inexact iterates and approximate convergence. 
%of the latter apply to the former.

\begin{restatable}{theorem}{ConvCDCA} \label{thm:ConvCDCA}
Given any $f=g-h$, where $g,h\in \Gamma_0$, let $\{x^k\}$ and $\{y^k\}$ be generated by variant of approximate CDCA \eqref{eq:CDCA}, where $x^{k+1}$ is any point in $\p_{\epsilon_x} g^*(y^k)$ (not necessarily a solution of Problem \eqref{eq:CDCA-x}).
Then, for $\epsilon \geq 0$, 
if $f(x^k) - f(x^{k+1}) \leq \epsilon$, $x^k$ is an $(\epsilon + \epsilon_x)$-strong critical point of $g - h$. %, with $\p h(x^k) \subseteq \p_{\epsilon_x + \epsilon} g(x^k)$. 
\red{Moreover, if $h$ is locally polyhedral, then $x^k$ is also an $(\epsilon + \epsilon_x)$-local minimum of $f$. This is the case for $h$ given by \eqref{eq:DS-DCdecomposition} when $\rho=0$.} 
\end{restatable}
\begin{proof}[Proof sketch]
    We extend a result in \citep[Theorem 2.3]{tao1988duality} which shows that if $x^k \in \p_\epsilon g^*(y^k)$ where $y^k$ is a solution of Problem \eqref{eq:CDCA-y} then $x^k$ is an $\epsilon$-strong critical point of $g - h$, from $\epsilon = 0$ to any $\epsilon>0$. The theorem then follows from \cref{thm:convergence}-\ref{itm:convergence}.
\end{proof}
%\mtodo{The final sentence is vacuous, should be removed.}

The full proof is given in \cref{sec:ConvCDCA-proof}. 
%It only requires that $x^{k+1} \in \p_{\epsilon_x} g^*(y^k)$ and not necessarily that it 
It does not require that $x^{k+1}$ is a solution of Problem \eqref{eq:CDCA-x}. However it does require that $y^k$ is a solution of Problem \eqref{eq:CDCA-y}. Whether a similar result holds when $y^k$ is only an approximate critical point is an interesting question for future work.

%\mtodo{check if in numerical results CDCA is converging to an approximate discrete strong local min. If not modify this sentence to simply say that when $y^k$ is only a stationary point the result doesn't necessarily hold anymore, and indeed in practice it does not...}<-- it's not possible to really check this, other than checking if it holds for random sets..
%Note that \cref{thm:ConvCDCA} does not require $x^{k+1}$ to be 

%Note that when $\rho=0$, $h$ in the DC program \eqref{eq:DS-DC} is polyhedral, hence the point $x^k$ reached at convergence in  

\looseness=-1 The next corollary relates strong criticality on the DC problem \eqref{eq:DS-DC} to strong local minimality on the DS problem \eqref{eq:DS}.

\begin{restatable}{corollary}{LocMinCDCA}\label{corr:LocMinCDCA} \looseness=-1  Given $f= g -h$ as defined in \eqref{eq:DS-DCdecomposition},  $\varepsilon \geq 0$, let $\hat{X} \subseteq V$ and $\hat{x}=\1_{\hat{X}}$. If $\hat{x}$ is an $\varepsilon$-strong critical point of $g - h$, then
%such that $\p h(\hat{x}) \subseteq \p_\varepsilon g(\hat{x})$, 
%if $\hat{x}$ is integral, i.e., $\hat{x}=\1_{\hat{X}}$ for some $\hat{X} \subseteq V$, 
$\hat{X}$ is an $\varepsilon'$-strong local minimum of $F$, where $\varepsilon' = \sqrt{2 \rho d \varepsilon}$ if $\varepsilon \leq \tfrac{\rho d}{2}$ and $\tfrac{\rho d}{2} + \varepsilon $ otherwise. 
\red{Conversely, if $\hat{X}$ is an $\varepsilon$-strong local minimum of $F$, then $\hat{x}$ is an $\varepsilon$-local minimum of $f$, and  hence also an $\varepsilon$-strong critical point of $g - h$.}
\end{restatable}
%\mtodo{The final sentence is vacuous, should be removed.}
\begin{proof}[Proof sketch]
We observe that for any $x=\1_X$ corresponding to $X \subseteq \hat{X}$ or $X \supseteq \hat{X}$, we have $\p h_L(\hat{x}) \cap \p h_L(x) \not = \emptyset$. The proof \red{of the forward direction}then follows from \cref{prop:LocalOptCondition}-\ref{itm:general} and the relation between the  $\varepsilon$-subdifferentials of $g$ and $g - \frac{\rho}{2} \| \cdot \|^2$. 
\red{For the converse direction, we argue that there exists a neighborhood $B_\delta(\hat{x})$ of $\hat{x}$, 
such that any $X = \round(x)$ for $x \in B_\delta(\hat{x})$, satisfies $X \subseteq \hat{X}$ or $X \supseteq \hat{X}$. The claim then follows from \cref{prop:LEproperties}-\ref{itm:round},\ref{itm:extension} and \cref{prop:LocalOptCondition}-\ref{itm:general}.}
See \cref{sec:LocMinCDCA-proof} for details. 
\end{proof}

\looseness=-1 \Cref{thm:ConvCDCA} and \cref{corr:LocMinCDCA} imply that CDCA with integral iterates $x^k$ 
converges to an $\epsilon'$-strong local minimum of $F$ after at most $(f(x^0) - f^\star)/\epsilon$ iterations, with $\epsilon'$ as in \eqref{eq:epsilon'-main}. %, and to an $(\epsilon + \epsilon_x)$-local minimum of $f$ if $\rho=0$,

\paragraph{Effect of regularization} The parameter $\rho$ has the same effect on CDCA as discussed in \cref{sec:reg} for DCA (\cref{corr:LocMinCDCA} shows, like in \cref{corr:LocMinDCA},  that for fixed $\epsilon$ and $\epsilon_x$, a larger $\rho$ may lead to a poorer solution). 
Also, as in DCA, when $\rho >0$ we can't restrict $x^k$ in CDCA to be integral.  Moreover, rounding only once at convergence is not enough to obtain even an approximate local minimum of $F$, as shown in \cref{ex:localMinRound}. Checking if a set is an approximate strong local minimum of $F$ is computationally infeasible, thus it cannot be explicitly enforced. Instead, we propose a variant of CDCA, which we call \CDCAR, where we round $x^k$ at each iteration.

\paragraph{CDCA with rounding}  Starting from $x^0 \in \{0,1\}^d$, the approximate \CDCAR iterates are given by 
%Similar to \DCAR, at each iteration $k$, \CDCAR updates $y^k, \tilde{x}^{k+1}$ as done in standard CDCA, and sets $x^{k+1} \gets \1_{X^{k+1}} \text{ where } X^{k+1}=\round(\tilde{x}^{k+1})$. 
%rounds at each iteration the iterate $\tilde{x}^{k+1}$ obtained by a standard CDCA iteration to
\begin{subequations}\label{eq:CDCAR}
\begin{align}
&y^k, \tilde{x}^{k+1} \text{ as in \eqref{eq:CDCA-y} and  \eqref{eq:CDCA-x} respectively,} \\
%    y^k &\in  \argmin \{ \ip{y}{x^k} - g^*(y) : y \in \partial h(x^k)\}, \label{eq:CDCAR-y} \\
%   x^{k+1} &\in \argmin \{ \ip{x}{y^k} - h(x) : x \in \partial g^*(y^k) \}, \label{eq:CDCAR-x} \\
& x^{k+1} \gets \1_{X^{k+1}} \text{ where } X^{k+1}=\round(\tilde{x}^{k+1}).
\end{align}
\end{subequations}
\looseness=-1 Since \CDCAR is a special case of \DCAR, all the properties of \DCAR discussed in \cref{sec:DCA} apply. In addition, since $y^k, \tilde{x}^{k+1}, $ are standard approximate CDCA iterates, \cref{thm:ConvCDCA} applies to them, with $x^{k+1}$ replaced by $\tilde{x}^{k+1}$. 
Since $x^k$ is integral in \CDCAR, \cref{corr:LocMinCDCA} holds. In particular,  
\DCAR converges to an $\epsilon'$-strong local minimum of $F$ after at most $(f(x^0) - f^\star)/\epsilon$ iterations, with $\epsilon'$ defined in \eqref{eq:epsilon'-main}. %, and to an $(\epsilon + \epsilon_x)$-local minimum of $f$ if $\rho=0$,
%at convergence $F(X^k) - F(X^{k+1}) \leq \epsilon$, we have that $X^k$ is an $\epsilon'$-strong local minimum of $F$, where $\epsilon' = \sqrt{2 \rho d (\epsilon + \epsilon_x)}$ if $\epsilon + \epsilon_x \leq \tfrac{\rho d}{2}$, and $\tfrac{\rho d}{2} + \epsilon + \epsilon_x$ otherwise. If $\rho=0$, $x^k$ is also an $(\epsilon + \epsilon_x)$-local minimum of $f$. 
See \cref{corr:CDCARconv} for details.

\looseness=-1 The guarantees of DCA and CDCA are equivalent when $F$ is submodular and similar when $F$ is supermodular. 
As stated in \cref{sec:prelim}, if $F$ is supermodular then any $\epsilon'$-{local minimum} of $F$ is also an $\epsilon' d$-{strong local minimum}. 
And when $h$ is differentiable, which is the case in DS minimization only if $H$ is modular and thus $F$ is submodular, then approximate weak and strong criticality of $f$ are equivalent. In this case, both DCA and CDCA return an $\epsilon'$-global minimum of $F$ if $x^k$ is integral; see \Cref{sec:specialCases}.
%any $(\epsilon, 0)$-critical point is also an $\epsilon$-strong critical point. 
%Hence, the solutions returned by DCA and CDCA have the same guarantees in this case (both are $\epsilon$-global minima in this case, if $x^k$ is integral; see \Cref{sec:specialCases}). 
%\Cref{corr:LocMinCDCA} then implies that the solutions returned by DCA and CDCA, when $\rho=0$ and $x^k$ is integral, and DCAR are $\epsilon'$-strong local minimum
However, in general the objective value achieved by a set satisfying the guarantees in \cref{corr:LocMinDCA} can be arbitrarily worse than any strong local minimum of $F$ as illustrated in \cref{ex:strongLocalMin}. This highlights the importance of the stronger guarantee achieved by CDCA.

\section{Experiments}\label{sec:exps}

\begin{figure*}
%\vspace{-25pt}
\centering
\begin{subfigure}{.37\textwidth}
 \centering
\includegraphics[scale=0.34]{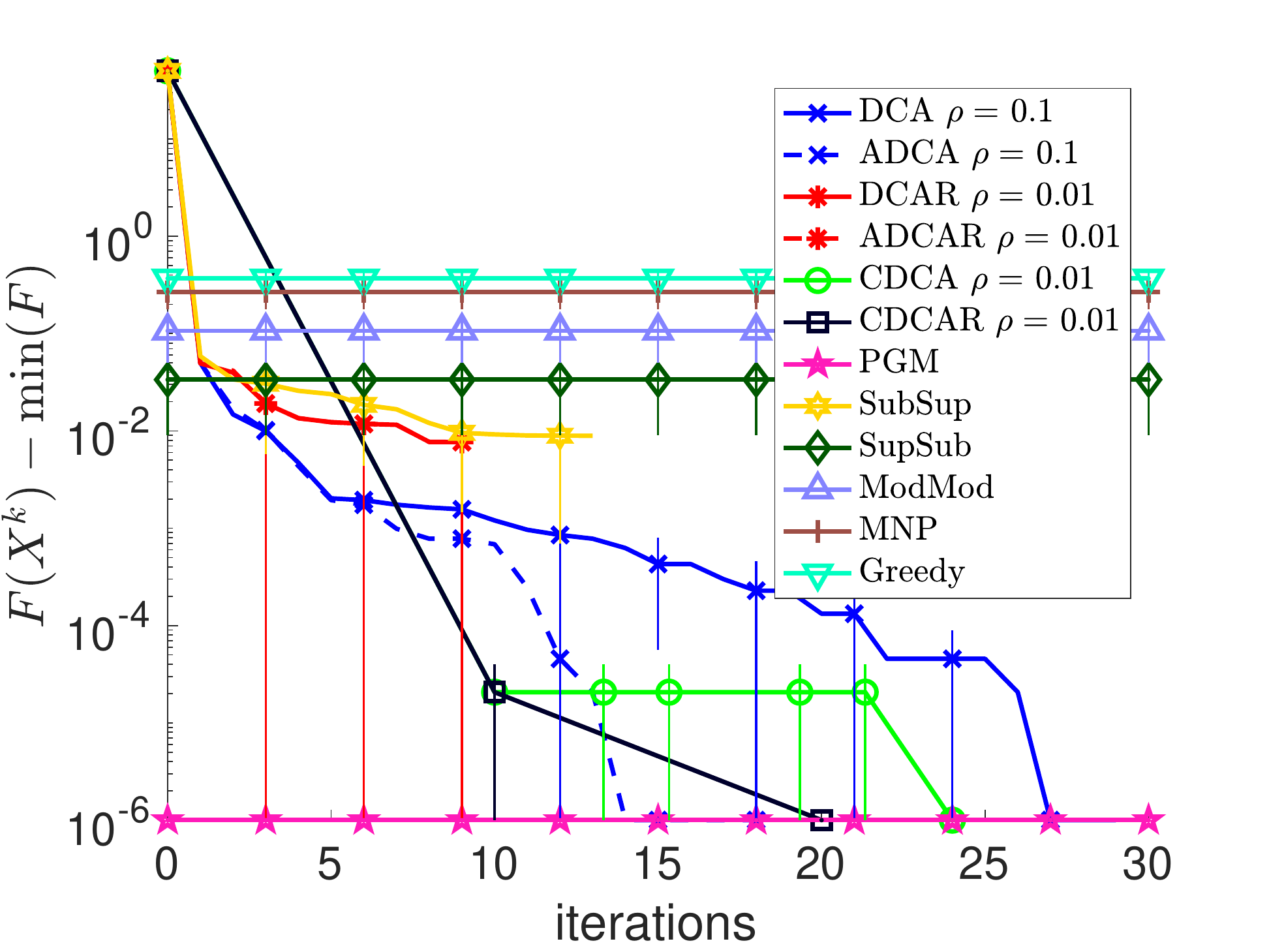}
\end{subfigure}\hspace{15pt}
\begin{subfigure}{.37\textwidth}
 \centering
\includegraphics[scale=0.34]{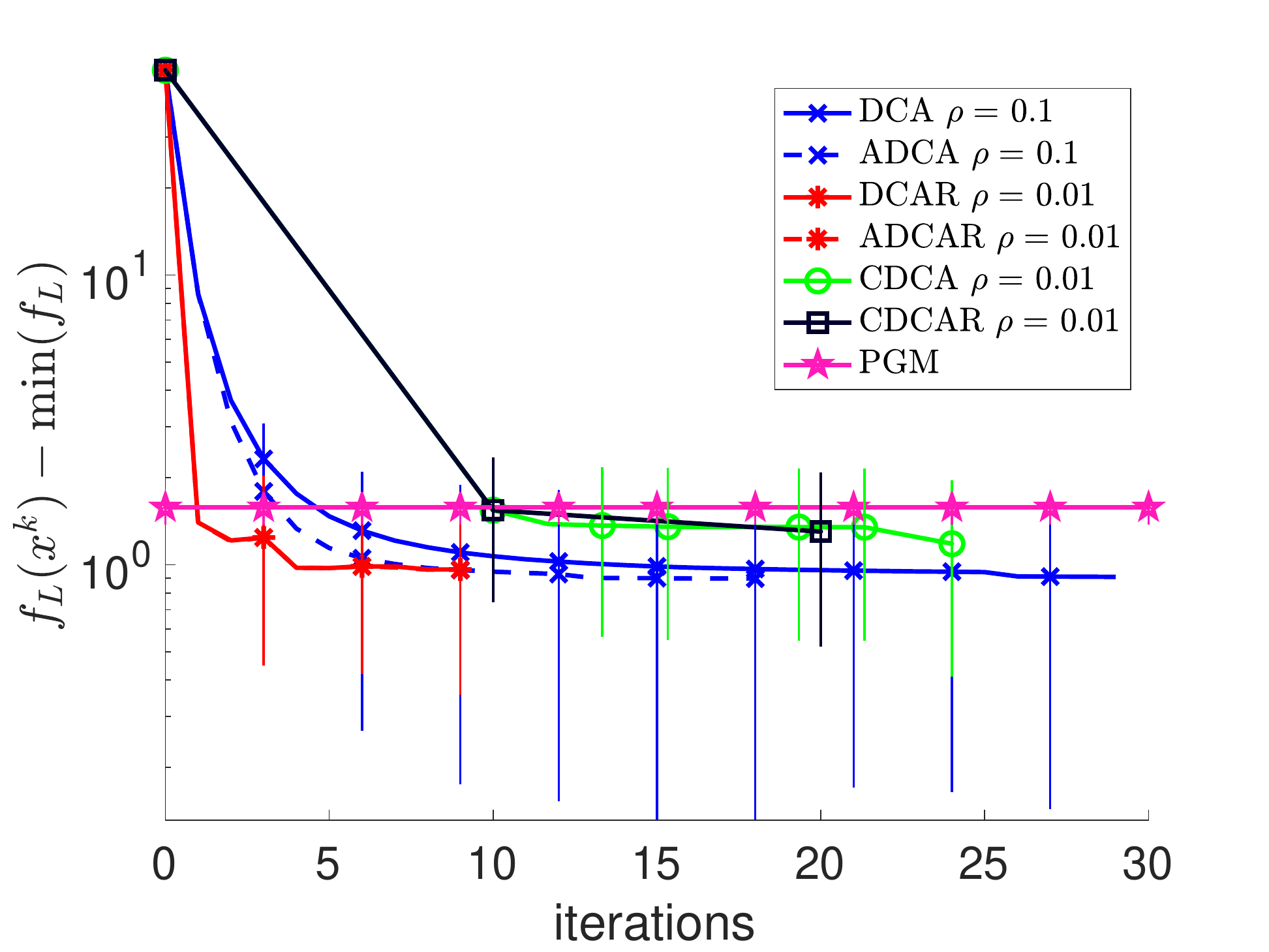}
\end{subfigure}\\ 
\begin{subfigure}{.37\textwidth}
 \centering
\includegraphics[scale=0.34]{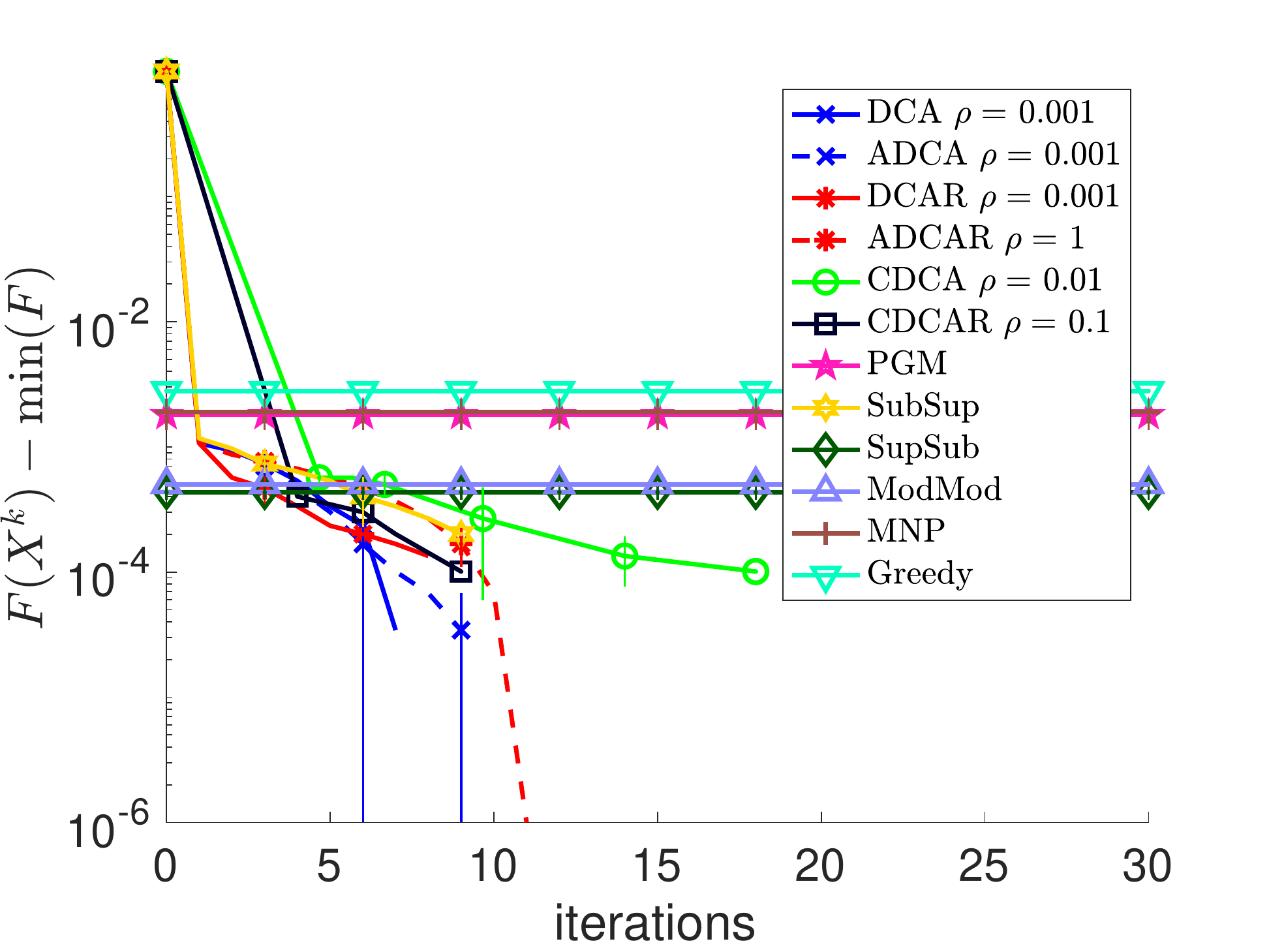}
\end{subfigure}\hspace{15pt}
\begin{subfigure}{.37\textwidth}
 \centering
\includegraphics[scale=0.34]{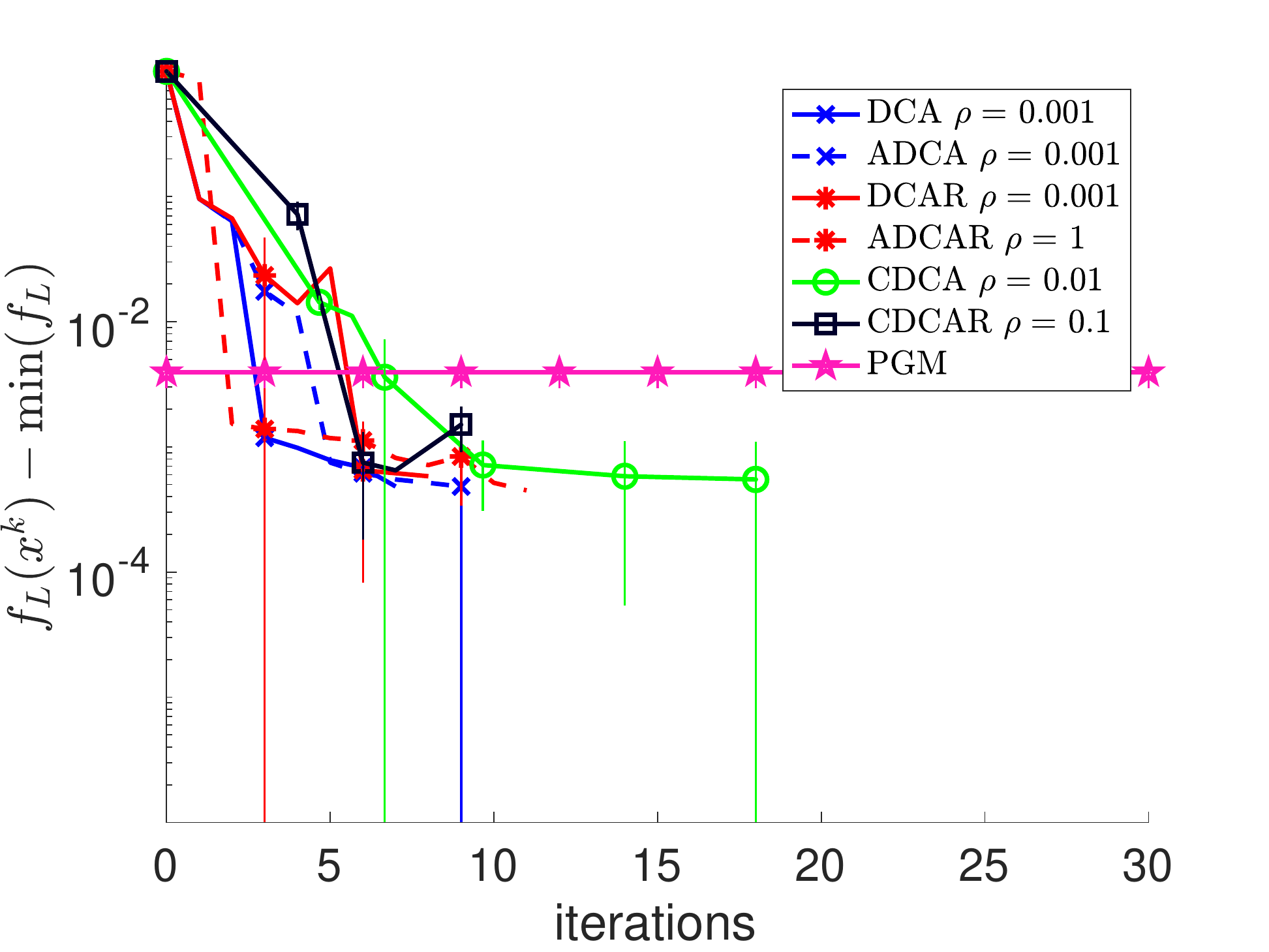}
\end{subfigure}
\caption{\label{fig:obj-iter-bestrho} Discrete and continuous objective values (log-scale) vs iterations on speech (top) and mushroom (bottom) datasets.}
\end{figure*}

\looseness=-1 In this section, we evaluate the empirical performance of our proposed methods on two applications: speech corpus selection and feature selection. 
%In particular, we address: (1)
%We address the following questions:
%\begin{enumerate}[itemindent=0em]  %[leftmargin=1em, itemindent=1em]  
%\item How do our proposed methods compare to other baselines and to each other?
%\item What is the effect of regularization?
%\item What is the effect of acceleration?
%\end{enumerate}
%%We focus on the first question here, and address the rest in Appendix \ref{}.
We compare our proposed methods DCA, \DCAR, CDCA and \CDCAR to the state-of-the-art  methods for DS minimization, SubSup, SupSub and ModMod \cite{Narasimhan2005a, Iyer2012a}. 
%We also apply the acceleration proposed in \cite{Nhat2018}
We also include an accelerated variant of DCA (ADCA) and \DCAR (\ADCAR), with the acceleration proposed in \cite{Nhat2018}. 
We use the minimum norm point (MNP) algorithm \cite{Fujishige2011} for submodular minimization in SubSup and the optimal Greedy algorithm of \citep[Algorithm 2]{Buchbinder2012} for submodular maximization in SupSub. We also compare with the MNP, PGM, and Greedy algorithms applied directly to the DS problem \eqref{eq:DS}.

We do not restrict $\rho$ to zero or the iterates to be integral in DCA and CDCA (recall that DCA in this case reduces to SubSup). Instead, we vary $\rho$ between $0$ and $10$, and %use the heuristic of rounding <-- this is not a heuristic for DCA
round only once at convergence (though for evaluation purposes we do round at each iteration, but we do not update $x^{k+1}$ with the rounded iterate).
We also do not consider $O(d)$ permutations for choosing $y^k$ in DCA, \DCAR, SubSup and ModMod, as required in \cref{corr:LocMinDCA} and  \cite{Iyer2012a} to guarantee convergence to an approximate local minimum of $F$, as this is too computationally expensive (unless done fully in parallel). %(unless enough workers are available to try the $O(d)$ choices fully in parallel).
 Instead, we consider as in \cite{Iyer2012a} three permutations to break ties in $x^k$: a random permutation, a permutation ordered according to the decreasing marginal gains of $G$, i.e., $G(i \mid X^k \setminus i)$, or according to  
the decreasing marginal gains of $F$, i.e., $F(i \mid X^k \setminus i)$, which we try in parallel at each iteration, then pick the one yielding the best objective $F$. We also apply this heuristic in CDCA and \CDCAR to choose an initial feasible point  $w^0 \in \rho x^k + \partial h_L(x^k)$ for FW \eqref{eq:FW-concave}; we pick the permutation yielding the smallest objective $\phi_k(w^0)$. 

\looseness=-1 We use $f(x^k) - f(x^{k+1}) \leq 10^{-6}$ as a stopping criterion in our methods, and $X^{k+1} = X^k$ in SubSup, SupSub and ModMod as in \cite{Iyer2012a}, and stop after a maximum number of iterations.  
To ensure convergence to a local minimum of $F$, we explicitly check for this as an additional stopping criterion in all methods except MNP, PGM and Greedy, and restart from the best neighboring set if not satisfied, as discussed in \cref{sec:reg}.
%as proposed in \citep[Algorithm 1]{Byrnes2015} for SubSup. If local minimality is not satisfied, we find  $X = \argmin_{|X \Delta X^{k+1}| = 1} F(X)$ and resume the algorithm after setting $x^{k+1} = \1_{X}$.  
%We summarize the stopping criteria used in all methods and their subsolvers in \cref{table:stopCriteria}.   
For more details on the experimental set-up, see \cref{sec:expSetup}.
%All experiments were conducted on cluster nodes with ...
The code is available at \url{https://github.com/SamsungSAILMontreal/difference-submodular-min.git}.
%We stopped DCA, \DCAR and SubSup after $30$ iterations or when the two stopping criteria are satisfied. For their subproblems, we stopped PGM and MNP after $10^3$ iterations or when the duality gap reached $10^{-6}$. Since one iteration of FW in CDCA, \CDCAR has similar cost to an iteration of DCA, we stopped them after $30$ total FW iterations,  or when the two stopping criteria are satisfied. In each outer iteration, we stop FW after $10$ iterations or when the FW gap reaches $10^{-6}$. We again used PGM to compute approximate subgradients of $g^*$, and stopped it after $10^3$ iterations or when the duality gap reached $10^{-6}$. We stopped MNP, ModMod and SupSub after  $3 \times 10^4$ iterations 

\paragraph{Speech corpus selection} \looseness=-1 
The goal of this problem is to find a subset of a large speech data corpus to rapidly evaluate new and expensive speech recognition algorithms. 
One approach is to select a subset of utterances $X$ from the corpus $V$ that simultaneously minimizes the vocabulary size and maximizes the total value of data \cite{Lin2011, Jegelka2011}. %value of data
Also, in some cases, some utterances' importance decrease when they are selected together.
This can be modeled by minimizing $F(X) = \lambda \sqrt{|\cN(X)|} - \sum_{i=1}^r \sqrt{m(X \cap V_i)}$,
%$F(X) = \lambda \sqrt{|\cN(X)|} - m(X)$, 
where $\cN(X)$ is the set of distinct words that appear in utterances $X$, $m$ is a non-negative modular function, with the  weight $m_j$ representing the importance of utterance $j$, and $V_1 \cup \cdots \cup V_r = V$. We can write $F$ as the difference of two non-decreasing submodular functions $G(X) = \lambda \sqrt{|\cN(X)|}$ and  $H(X) = \sum_{i} \sqrt{m(X \cap V_i)}$. 
%This can be modeled by minimizing $F(X) = \lambda \sqrt{w(\cN(X))} - m(X)$, where $\cN(X)$ is the set of distinct words that appear in utterances $X$, and $w, m$ are modular functions, with the non-negative weights $w_j$ indicating the unimportance of word $j$ and $m_i$ the importance of utterance $i$. The function $G(X) = \lambda \sqrt{w(\cN(X))}$ is non-decreasing submodular. 
%In some cases, however, some utterances' importance decrease when they are selected together. This can be expressed by replacing $m$ by the non-decreasing submodular  utility $H(X) = \sum_{i} \sqrt{m(X \cap V_i)}$, where $V$ is partitionned into groups $V_i$.
Moreover, this problem is a special case of DS minimization, where $H$ is \emph{approximately modular}. In particular, $H$ is $(1, \beta)$-\emph{weakly DR-modular} (see \cref{def:WDR}) with\footnote{The proof follows similarly to \citep[Lemma 3.3]{Iyer2013b}} $$\beta \geq \min_{i \in [r]} \min_{j \in V_i} \tfrac{1}{2} \sqrt{\tfrac{m(j)}{m(V_i)}}.$$ The parameter $\beta$ characterizes how close $H$ is to being supermodular.
This DS problem thus fits under the setting considered in \cite{Halabi20} (with $\alpha=1$), for which PGM was shown to achieve the optimal approximation guarantee $F(\hat{X}) \leq G(X^*) - \beta H(X^*) + \epsilon$ for some $\epsilon>0$, where $X^*$ is a minimizer of $F$ (see Corollary 1 and Theorem 2 therein). We show in \cref{sec:ApproxSub} that any variant of DCA and CDCA obtains the same approximation guarantee as PGM (see \cref{prop:localMinWDRSub-weaker} and discussion below it).\\
We use the same dataset used by \citep[Section 12.1]{Bach2013}, with $d = |V|= 800$ utterances and $1105$ words. 
We choose $\lambda=1$, the non-negative weights $m_i$ randomly, and partition $V$ into $r=10$ groups of consecutive indices. %The results are shown in ...
%Both $G(X) = \lambda \sqrt{\sum_{j \in \cN(X)} w_j}$ and $H(X) = \sum_{i} \sqrt{m(X \cap V_i)}$ are non-decreasing submodular functions.

\paragraph{Feature selection} \looseness=-1 Given a set of features $U_V = \{U_1, U_2, \cdots, U_d\}$, the goal is to find a small subset of these features $U_X = \{U_i: i \in X\}$ 
that work well when used to classify a class $C$. We thus want to select the subset
which retains the most information from the original set $U_V$ about $C$. This can be modeled by minimizing $F(X) = \lambda |X| - \mathrm{I}(U_X; C)$. The mutual information  $\mathrm{I}(U_X; C)$ can be written as the difference of the entropy $\mathcal{H}(U_X)$ and conditional entropy $\mathcal{H}(U_X \mid C)$, both of which are non-decreasing submodular. Hence $F$ can be written as the difference of two non-decreasing submodular functions $G(X) =  \lambda |X| + \mathcal{H}(U_X \mid C)$ and $H(X) =  \mathcal{H}(U_X)$. We estimate the mutual information from the data. %using empirical probabilities
We use the Mushroom data set from \cite{Dua:2019}, which has 8124 instances with 22 categorical attributes, which we convert to $d = 118$ binary features. We randomly select $70\%$ of the data as training data for the feature selection, and set $\lambda=10^{-4}$. %The results are shown in ...

\paragraph{Results:} \looseness=-1 We plot in \cref{fig:obj-iter-bestrho}, the discrete objective values $F(X^k) - \min(F)$ and continuous objective values  $f_L(x^k) - \min(f_L)$, per iteration $k$, where $\min(F)$ and $\min(f_L)$ are the smallest values achieved by all compared methods. We only plot the continuous objective of the methods which minimize the continuous DC problem \eqref{eq:DS-DC}, instead of directly minimizing the DS problem \eqref{eq:DS}, i.e., our methods and PGM. For DCAR and CDCAR, we plot the continuous objective values before rounding, i.e., $f_L(\tilde{x}^{k})$, since the continuous objective after rounding  is equal to the discrete one, i.e., $f_L({x}^k) = F(X^{k})$. 
Results are averaged over 3 random runs, with standard deviations shown as error bars.
%Figure \ref{} shows the difference between the discrete objective values $F(X)$ achieved at every iteration by each algorithm and the smallest one $\min(F)$ obtained by all of them. 
For clarity, we only include our methods with the $\rho$ value achieving the smallest discrete objective value. We show the results for all $\rho$ values in \cref{sec:regEffect}. 
%\cref{fig:obj-iter-allrhos-speech} and \ref{fig:obj-iter-allrhos-mushroom}. 
For a fair implementation-independent comparison, we use the number of FW \eqref{eq:FW-concave} iterations as the x-axis for CDCA and \CDCAR, since one iteration of FW has a similar cost to an iteration of DCA variants. We only show the minimum objective achieved by SupSub, ModMod, MNP, PGM and Greedy, since their iteration time is significantly smaller than the DCA and CDCA variants. 
We show the results
%the plots of discrete and continuous objective values 
with respect to time in \cref{sec:runtimes}. %\cref{fig:obj-time-bestrho}.

\looseness=-1 We observe that, as expected, PGM obtains the same discrete objective value as the best variants of our methods on the speech dataset, where PGM and our methods achieve the same approximation guarantee, but worse on the adult dataset, where PGM has no guarantees. Though in terms of continuous objective value, PGM is doing worse than our methods on both datasets. Hence, a better $f_L$ value does not necessarily yield a better $F$ value after rounding.
In both experiments, our methods reach a better $F$ value than all other baselines, except SubSup which gets the same value as DCAR on the speech dataset, and a similar value to our non-accelerated methods on the mushroom dataset. 

The complete variants of our methods, CDCA and CDCAR, perform better in terms of $F$ values, than their simple counterparts, DCA and DCAR, on the speech dataset. But, on the mushroom dataset,  CDCAR perform similarly to DCAR, while CDCA is worse that DCA. Hence, using the complete variant is not always advantageous. 
In terms of $f_L$ values, CDCA and CDCAR perform worse than DCA and DCAR, respectively, on both datasets. Again this illustrates than a better $f_L$ value does not always yield a better $F$ value.
%Note then that a better $f_L$ value does not necessarily yield a better $F$ value after rounding.
%We also note that 

\looseness=-1 Rounding at each iteration helps for CDCA on both datasets; CDCAR converges faster than CDCA in $F$, but not for DCA; DCAR reaches worse $F$ value than DCA. Note that unlike $f_L({x}^k)$, 
the objective values $f_L(\tilde{x}^k)$ of DCAR and CDCAR are not necessarily approximately non-increasing (\cref{them:convergence-round-app}-\ref{itm:descent-round-app} does not apply to them), which we indeed observe on the mushroom dataset. 
%satisfy $f_L(\tilde{x}^k) - f_L(\tilde{x}^{k+1}) \geq $
%This is perhaps not surprising since DCA and DCAR
%Using the complete form of CDCA improves performance in terms of $F$
%does not give a significant advantage, other than a faster convergence in $F$ on the speech dataset, 

\looseness=-1 Finally, we observe that adding regularization leads to better $F$ values; the best $\rho$ is non-zero for all our methods (see \cref{sec:regEffect} for a more detailed discussion on the effect of regularization). Acceleration helps in most cases but not all; DCAR and ADCAR perform the same on the speech dataset. 
%Note that a better continuous objective does not necessarily yield a better discrete objective after rounding (e.g., DCAR has smaller $f_L$ values than DCA in the first 10 iterations, but larger $F$ values, on the speech dataset).

%\begin{figure}
%%\vspace{-25pt}
%\begin{subfigure}{.21\textwidth}
% \centering
%\includegraphics[trim=10 30 0 350, clip, scale=0.2]{speech-lbd1.0e+00-d268453-53048204/discrete-obj-iter-bestrho.pdf}
%\end{subfigure}\hspace{15pt}
%\begin{subfigure}{.21\textwidth}
% \centering
%\includegraphics[trim=0 30 0 350, clip, scale=0.2]{speech-lbd1.0e+00-d268453-53048204/cont-obj-iter-bestrho.pdf}
%\end{subfigure}\\ 
%\begin{subfigure}{.21\textwidth}
% \centering
%\includegraphics[trim=10 30 0 350, clip, scale=0.2]{speech-lbd1.0e+00-d268453-53048204/discrete-obj-iter-bestrho.pdf}
%\end{subfigure}\hspace{15pt}
%\begin{subfigure}{.21\textwidth}
% \centering
%\includegraphics[trim=10 30 0 350, clip, scale=0.2]{speech-lbd1.0e+00-d268453-53048204/cont-obj-iter-bestrho.pdf}
%\end{subfigure}\\
%\caption{\label{fig:results} results}
%\end{figure}

\section{Conclusion}

 We introduce variants of DCA and CDCA for minimizing the DC program equivalent to DS minimization. We establish novel links between the two problems, which allow us to match the theoretical guarantees of existing algorithms using DCA, and to achieve stronger ones using CDCA. Empirically, our proposed methods perform similarly or better than all existing methods.

%We establish novel links between optimality guarantees of the two problems, which allow us to match the theoretical guarantees of existing algorithms for DS minimization using DCA variants, and to achieve a stronger optimality guarantee using CDCA variants. We also show that adding regularization to DCA and CDCA can improve performance, both theoretically and empirically.

%One limitation of CDCA is that its stronger optimality guarantee requires solving a concave minimization subproblem which is intractable. We show how to  approximately solve this subproblem efficienlty. However the resulting algorithm is still more expensive than DCA, does not necessarily satisfy the stronger optimality guarantee, and in practice it does not always perform better than DCA.

\section*{Acknowledgements}

This research was enabled in part by support provided by 
Calcul Quebec (\url{https://www.calculquebec.ca/})
% for beluga
and the Digital Research Alliance of Canada (\url{https://alliancecan.ca/}). 
George Orfanides was partially supported by NSERC CREATE INTER-MATH-AI.
Tim Hoheisel was partially supported by the NSERC discovery grant RGPIN-2017-04035.

\bibliographystyle{icml2023}
\bibliography{biblio}

%%%%%%%%%%%%%%%%%%%%%%%%%%%%%%%%%%%%%%%%%%%%%%%%%%%%%%%%%%%%%%%%%%%%%%%%%%%%%%%
%%%%%%%%%%%%%%%%%%%%%%%%%%%%%%%%%%%%%%%%%%%%%%%%%%%%%%%%%%%%%%%%%%%%%%%%%%%%%%%
% APPENDIX
%%%%%%%%%%%%%%%%%%%%%%%%%%%%%%%%%%%%%%%%%%%%%%%%%%%%%%%%%%%%%%%%%%%%%%%%%%%%%%%
%%%%%%%%%%%%%%%%%%%%%%%%%%%%%%%%%%%%%%%%%%%%%%%%%%%%%%%%%%%%%%%%%%%%%%%%%%%%%%%
\newpage
\appendix
\onecolumn
\section{Subsup as a Special Case of DCA}\label{sec:SubSupDCA}
We show that 
the SubSup procedure proposed in \cite{Narasimhan2005a} is a special case of DCA. SubSup starts from $X^0 \subseteq V$, and makes the following updates at each iteration:
 \begin{equation}
\begin{aligned}\label{eq:SubSup}
%\sigma &\gets \text{random permutation such that } S^\sigma_{|X^k|} = X^k \\
 y^k_{\sigma(i)} &\gets H(\sigma(i) \mid S^\sigma_{i-1}) ~\forall i \in V, \text{for } \sigma \in S_d \text{ such that } S^\sigma_{|X^k|} = X^k \\
 X^{k+1} &\gets \argmin_{X \subseteq V} G(X) - y^k(X) \\
\end{aligned}
\end{equation}

Note that $y^k \in \p h_L(\1_{X^k})$ %is a subgradient of $h_L$ at $\1_{X^k}$ 
by \cref{prop:LEproperties}-\ref{itm:greedy} and $\1_{X^{k+1}} \in \argmin_{x \in [0,1]^d} g_L(x) - \ip{x}{y^k}$ as discussed in \cref{sec:DCA}, thus they are valid  updates of DCA in \cref{eq:DCASet} with $\rho = \epsilon_x = 0$. %when applied to Problem \ref{eq:DS-DC} 

\section{Experimental Setup Additional Details}\label{sec:expSetup}

\begin{table}
\caption{Stopping criteria}  \label{table:stopCriteria}
\begin{center}
\begin{adjustbox}{width=1\textwidth}
\begin{tabular}{c | c | c | c | c | c | c | c  }
\toprule
 DCA, \DCAR,  & CDCA, \CDCAR& SubSup & SupSub, ModMod & MNP, PGM & PGM in DCA and  & MNP in SubSup & FW in CDCA  \\
  ADCA, \ADCAR & & & & &   CDCA variants & & variants\\
 \midrule
$f(x^k) - f(x^{k+1}) \leq 10^{-6}$ & $f(x^k) - f(x^{k+1}) \leq 10^{-6}$ & $X^{k+1} = X^k$ & $X^{k+1} = X^k$ & & $\text{gap} \leq 10^{-6}$ & $\text{gap} \leq 10^{-6}$ & $\text{gap} \leq 10^{-6}$ \\
$k \leq 30$ & $k + \text{\# FW iterations} \leq 30$ & $k \leq 30$ & $k \leq 3 \times 10^4$ & $k \leq 3 \times 10^4$ & $k \leq 10^3$ & $k \leq 10^3$ & $k \leq 10$ \\
% $F(X^{k+1}) \leq \min_{|X \Delta X^{k+1}| = 1} F(X)$ 
$X^{k+1}$ local minimum of $F$ & $X^{k+1}$ local minimum of $F$ & $X^{k+1}$ local minimum of $F$ & $X^{k+1}$ local minimum of $F$ & & & & \\
\bottomrule
\end{tabular}
\end{adjustbox}
\end{center}
\end{table}
In this section, we provide additional details on our experimental setup. As in \cite{Iyer2012a}, we consider in ModMod and SupSub two modular upper bounds on $G$, which we try in parallel and pick the one which yields the best objective $F$. We set the parameter $q$  in ADCA and \ADCAR to $5$ as done in \cite{Nhat2018}.
We summarize the stopping criteria used in all methods and their subsolvers in \cref{table:stopCriteria}. We pick the maximum number of iterations according to the complexity per iteration. %; since one iteration of FW in CDCA, \CDCAR has similar cost to an iteration of DCA, we set the maximum number of iterations, $10^3$ iterations of MNP has similar cost to one iteration of the DCA variants, SupSub and ModMod are also quite cheaper than SubSup so we allow max number of iters 
We use the random seeds 42, 43, and 44.  We use the implementation of MNP from the Matlab code provided in \citep[Section 12.1]{Bach2013} and implement the rest of the methods in Matlab.

\section{Additional Empirical Results}

In this section, we present some additional empirical results of the experiments presented in \cref{sec:exps}. 

\subsection{Effect of regularization}\label{sec:regEffect}

\mtodo{Recall here the theoretical effects of $\rho$. I think we can show that critical points of $g-h$ with $rho>0$ are critical points of $g - h$ without $\rho$ with $\epsilon$ depending on $\rho$ in the same was as for local minimality w.r.t $F$, which would explain maybe why overall larger $\rho$ leads to slower convergence.}

We report the discrete and continuous objective values per iteration of our proposed methods, for all $\rho$ values, on the speech dataset in \cref{fig:obj-iter-allrhos-speech} and the mushroom dataset in \cref{fig:obj-iter-allrhos-mushroom}. We observe that the variants without rounding at each iteration converge slower in $f_L$ for larger $\rho$ values, though not always, e.g., DCA with $\rho=0.001$ converges faster than with $\rho=0$ on the speech dataset, and CDCA with $\rho = 0.01$  converges faster than with $\rho=0.1$ on the mushroom dataset. The effect of $\rho$ on the rounded variants is less clear; in most cases the methods with small $\rho$ values are performing worse, but for CDCAR on the speech dataset the opposite is true.
We again observe that better performance w.r.t $f_L$ does not necessarily translate to better performance w.r.t $F$. The effect of $\rho$ on performance w.r.t $F$ varies with the different methods and datasets. But in all cases, the best $F$ values is obtained with $\rho>0$.
%On the speech dataset, methods with larger $\rho$ values are doing better, except for DCAR and ADCAR where all $\rho$ values do similarly. On the mushroom dataset, some methods do better with smaller $\rho$ values (DCA, ADCA, CDCA?), other do better larger ones (CDCAR)...

Recall that we use PGM to compute an $\epsilon_x$-subgradient of $g^*$ to update $x^k$ in DCA variants \eqref{eq:DCASet-x} and CDCA variants \eqref{eq:CDCA-x}, as well as in each iteration of FW \eqref{eq:FW-concave} to update $y^k$ in CDCA variants \eqref{eq:CDCA-y}.
%Recall that in all our methods we use PGM to compute an $\epsilon_x$-subgradient of $g^*$, which 
As discussed in \cref{sec:DCA}, PGM requires $O(d \kappa^2 / \epsilon_x^2)$ iterations when $\rho=0$ and $O(2 (\kappa + \rho \sqrt{d})^2 / \rho \epsilon_x)$ when $\rho>0$, where $\kappa$ is the Lipschitz constant of $ g_L(x) - \ip{x}{y^k}$. \Cref{fig:gap-iter-allrhos} shows the gap reached by PGM at each iteration of DCA variants, and the worst gap reached by PGM over all the approximate subgradient computations done at each iteration of CDCA variants. As expected, a larger $\rho$ leads to a more accurate solution (smaller gap), for a fixed number of PGM iterations (we used $1000$). Though, the accuracy at $\rho=0$ is better than the very small non-zero values $\rho = 0.01, 0.001$, for which the complexity $O(2 (\kappa + \rho \sqrt{d})^2 / \rho \epsilon_x)$ becomes larger than $O(d \kappa^2 / \epsilon_x^2)$.

%in both DCA variants (Problem \ref{eq:DCASet-x}) and CDCA variants (Problem \ref{eq:CDCA-x} and \eqref{eq:FW-LO})

\begin{figure}
\vspace{-7pt}
\centering
\begin{subfigure}{\textwidth}
 \centering
\includegraphics[scale=0.38]{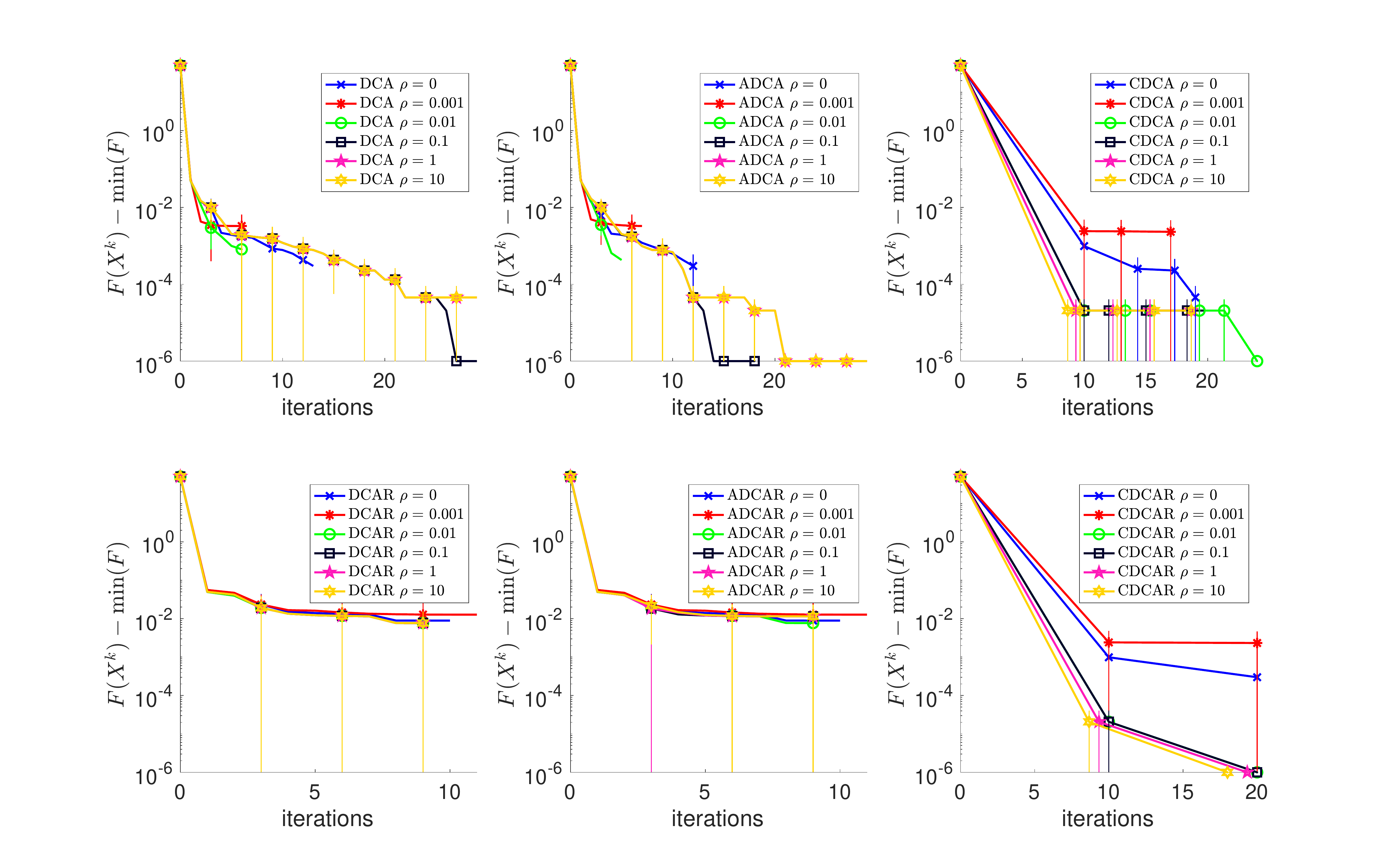}
\end{subfigure}\\
\vspace{-7pt}
\begin{subfigure}{\textwidth}
 \centering
\includegraphics[scale=0.38]{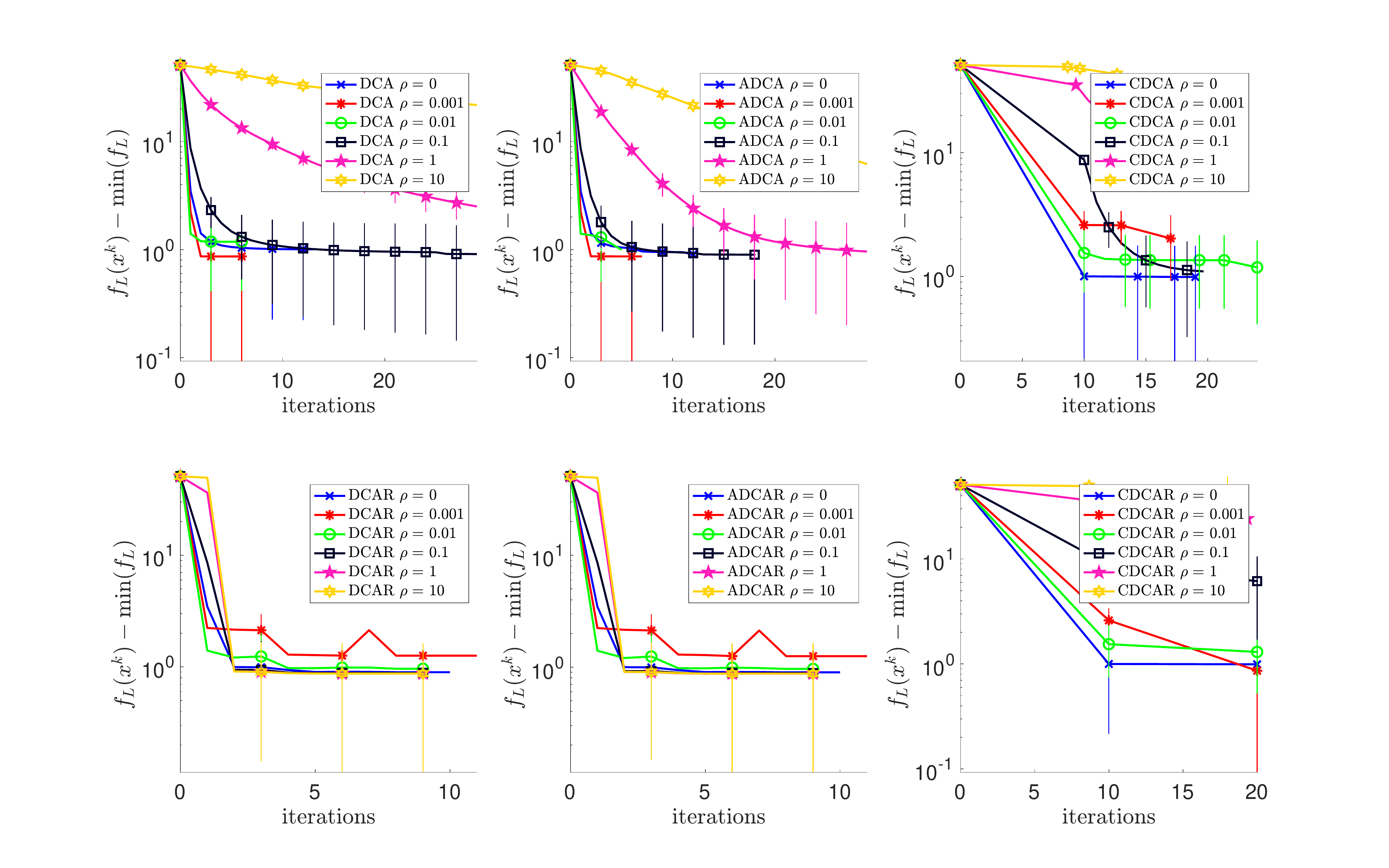}
\end{subfigure}
\caption{\label{fig:obj-iter-allrhos-speech} Discrete and continuous objective values (log-scale) of our proposed methods for all $\rho$ values vs iterations on speech dataset.}
\end{figure}

\begin{figure}
\vspace{-7pt}
\begin{subfigure}{\textwidth}
 \centering
\includegraphics[scale=0.38]{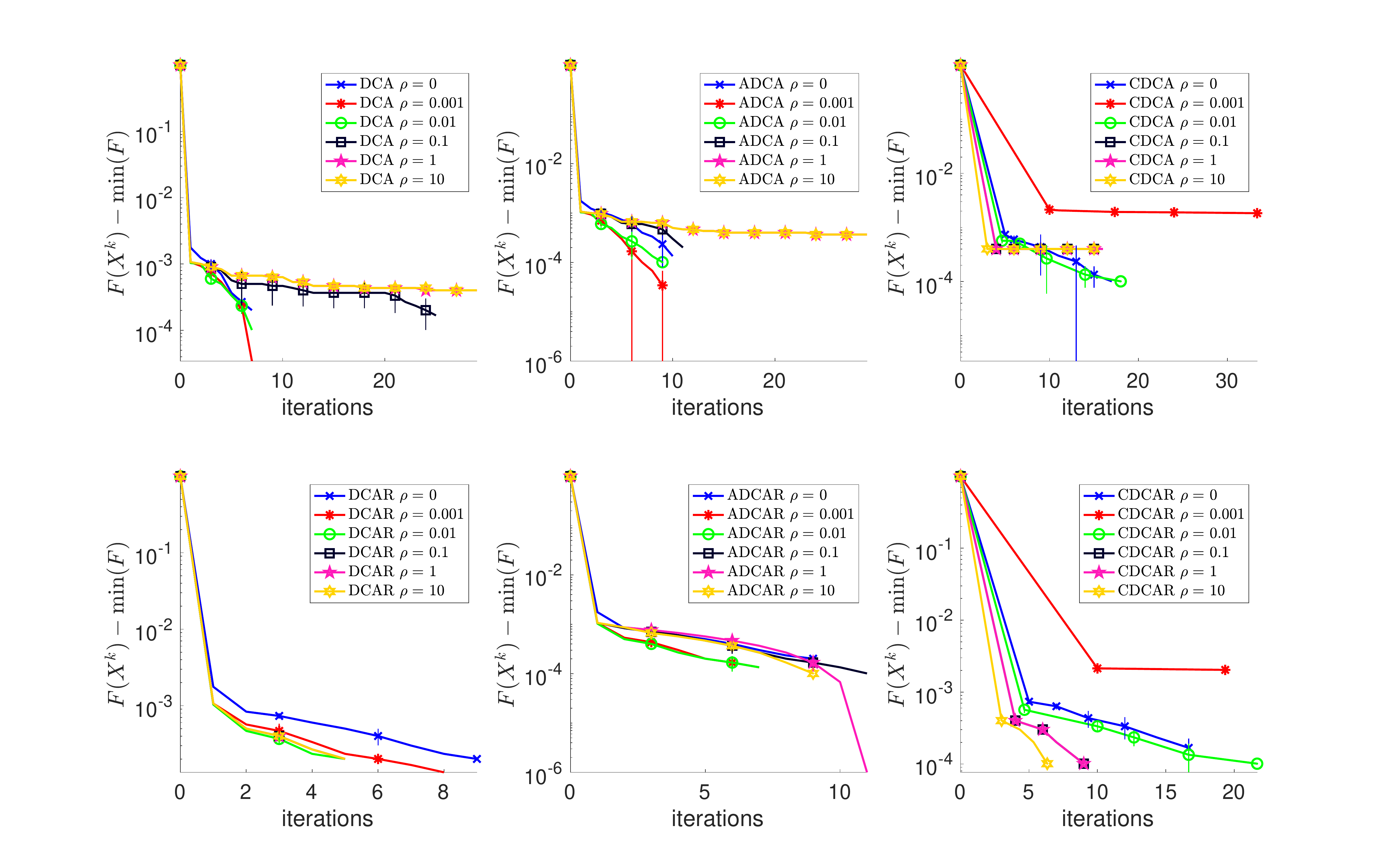}
\end{subfigure}\\
\vspace{-7pt}
\begin{subfigure}{\textwidth}
 \centering
\includegraphics[scale=0.38]{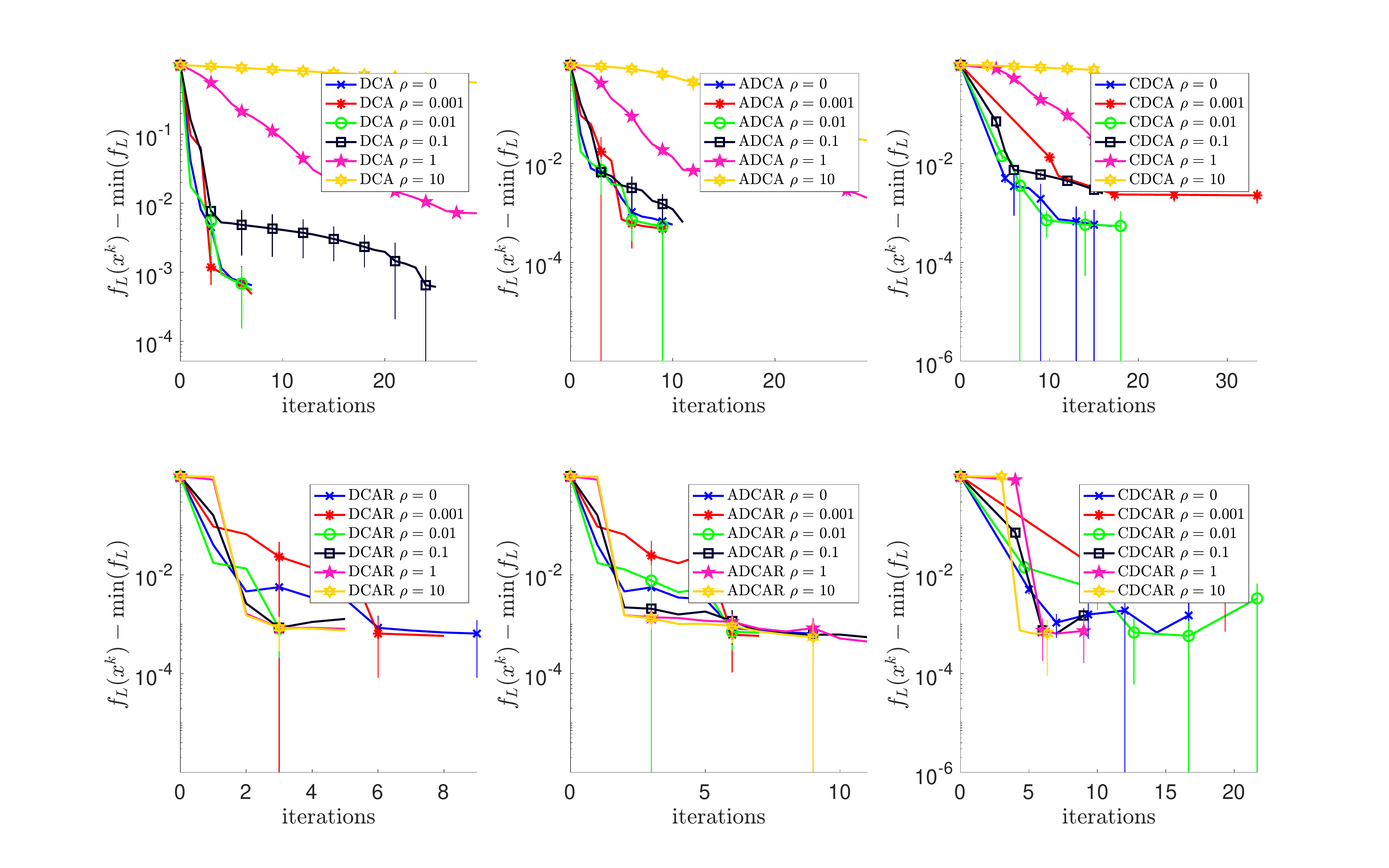}
\end{subfigure}
\caption{\label{fig:obj-iter-allrhos-mushroom} Discrete and continuous objective values (log-scale) of our proposed methods for all $\rho$ values vs iterations on mushroom dataset.}
\end{figure}

\begin{figure}
\vspace{-7pt}
\begin{subfigure}{\textwidth}
 \centering
\includegraphics[scale=0.38]{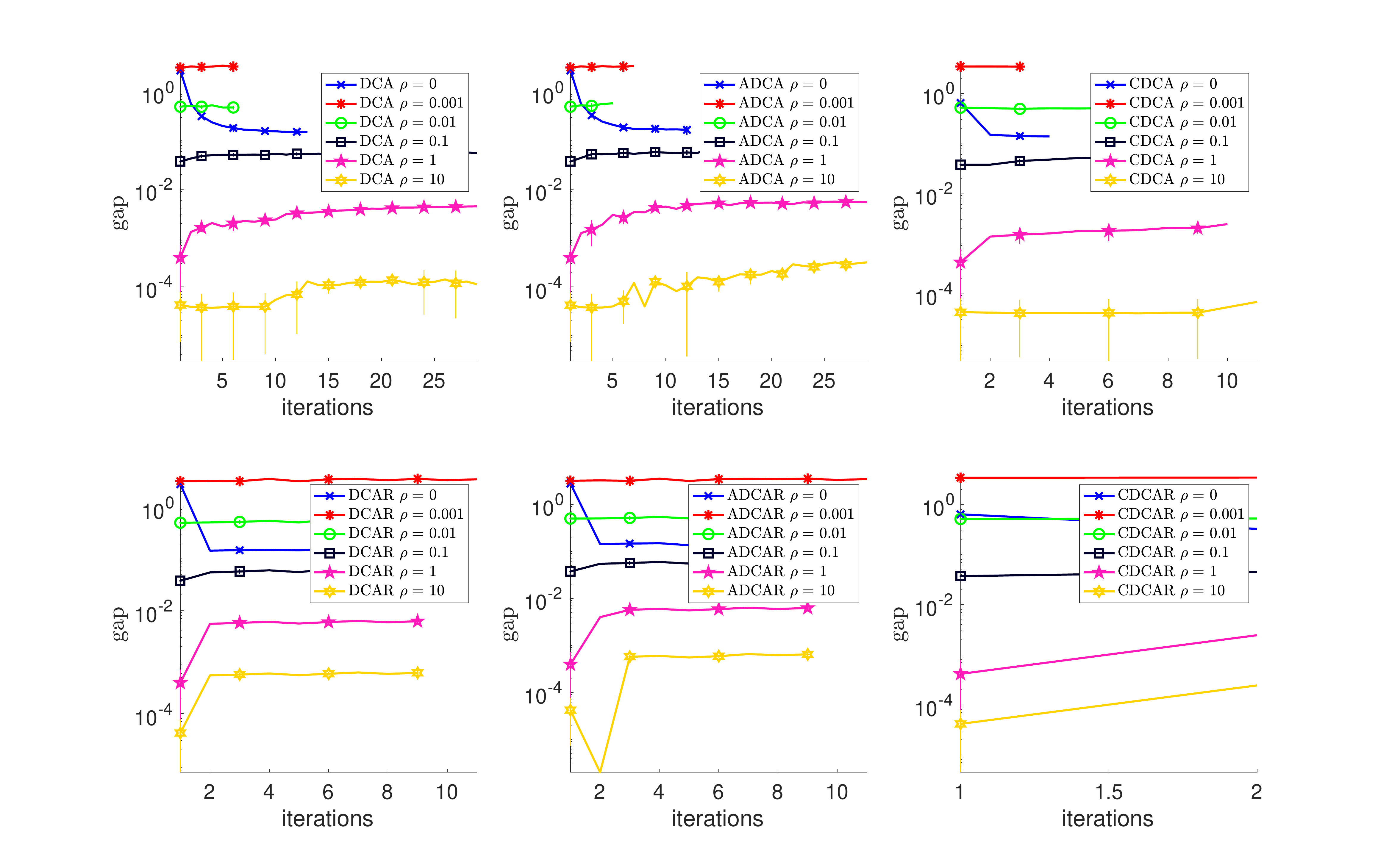}
\end{subfigure}\\
\vspace{-7pt}
\begin{subfigure}{\textwidth}
 \centering
\includegraphics[scale=0.38]{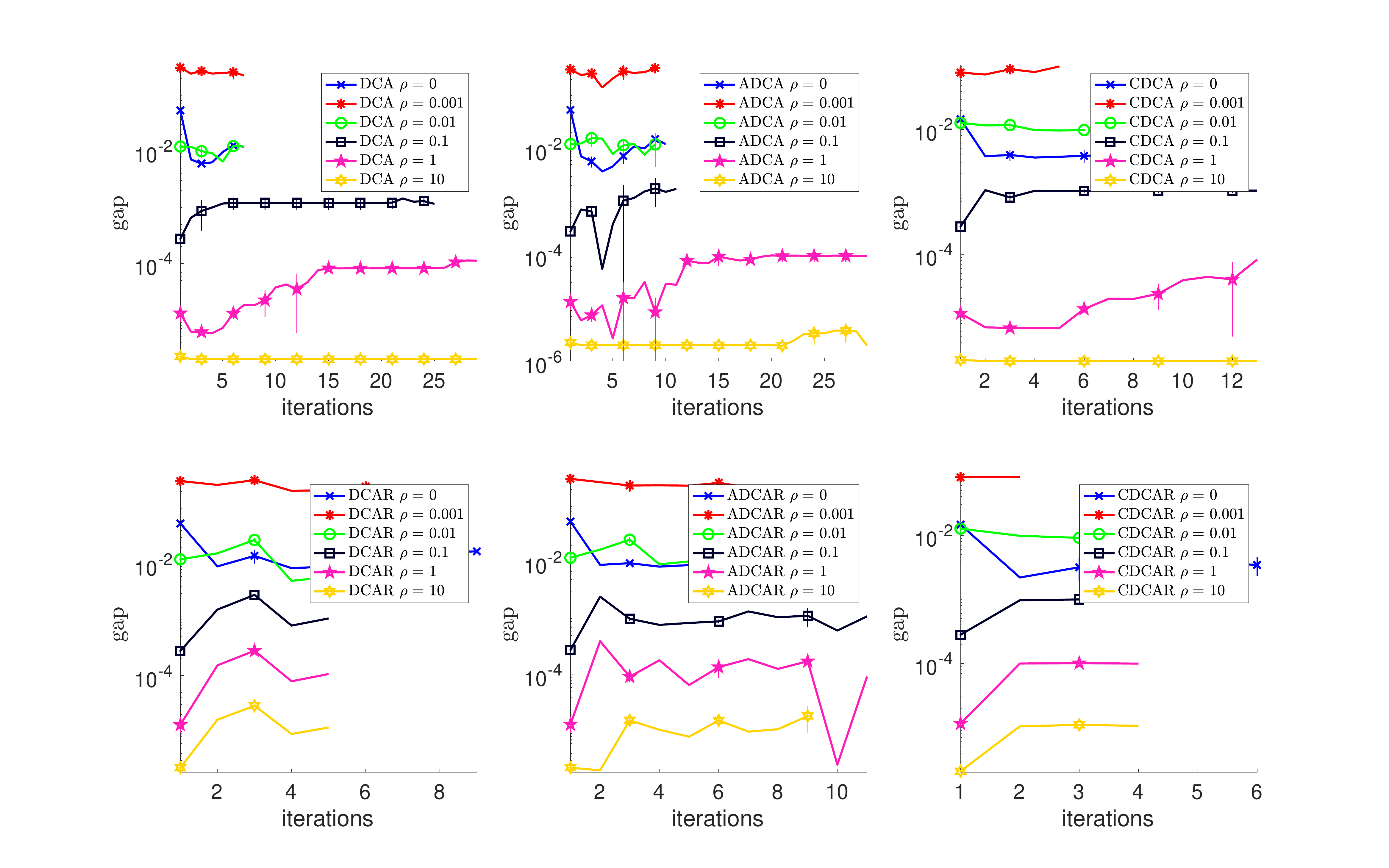}
\end{subfigure}
\caption{\label{fig:gap-iter-allrhos} PGM gap values (log-scale) of our proposed methods for all $\rho$ values vs iterations on speech (top two rows) and mushroom (bottom two rows) datasets.}
\end{figure}

\subsection{Running times}\label{sec:runtimes}

\begin{figure}
%\vspace{-25pt}
\centering
\begin{subfigure}{.33\textwidth}
 \centering
\includegraphics[scale=0.32]{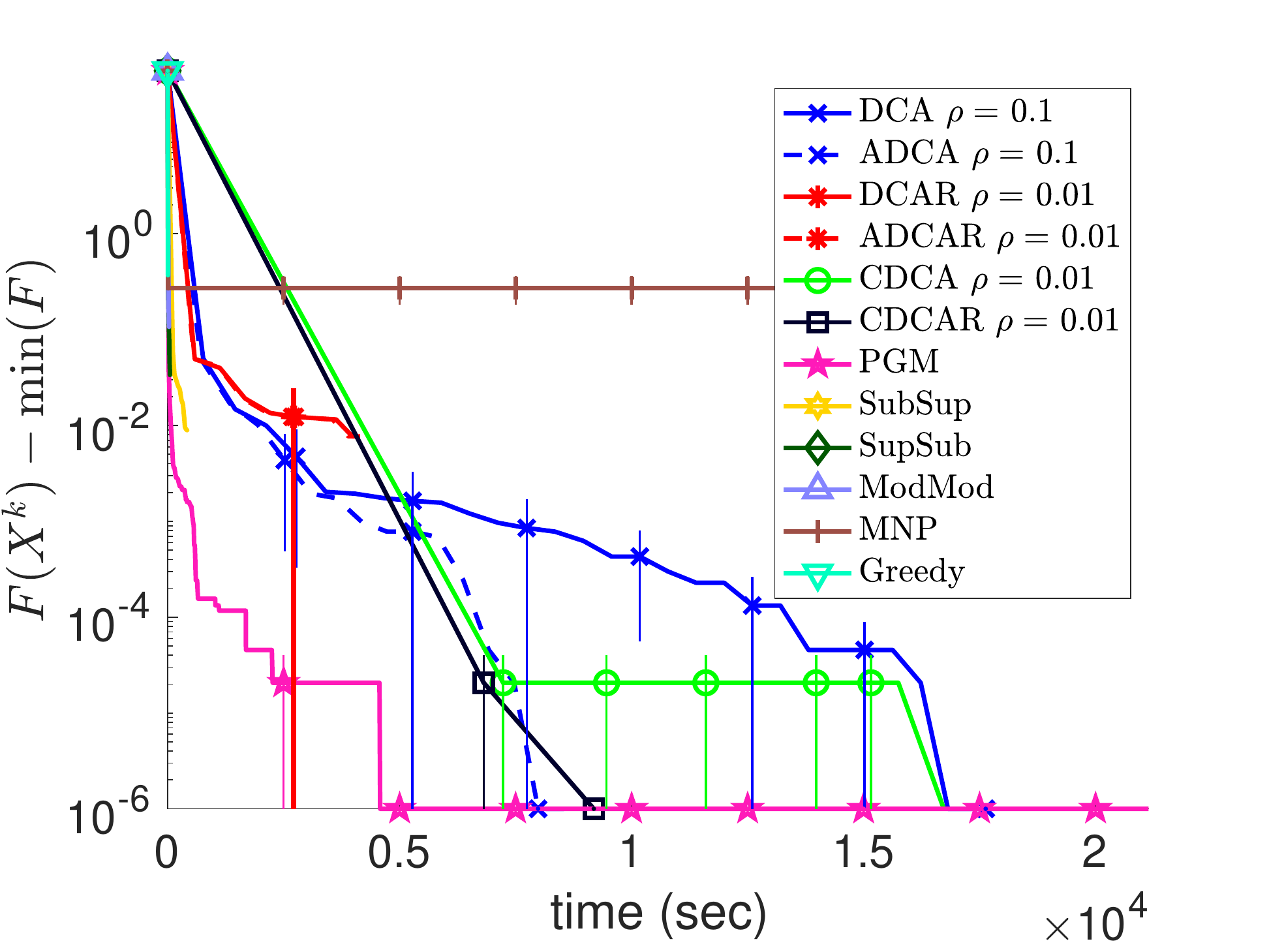}
\end{subfigure}
\begin{subfigure}{.32\textwidth}
 \centering
\includegraphics[scale=0.3]{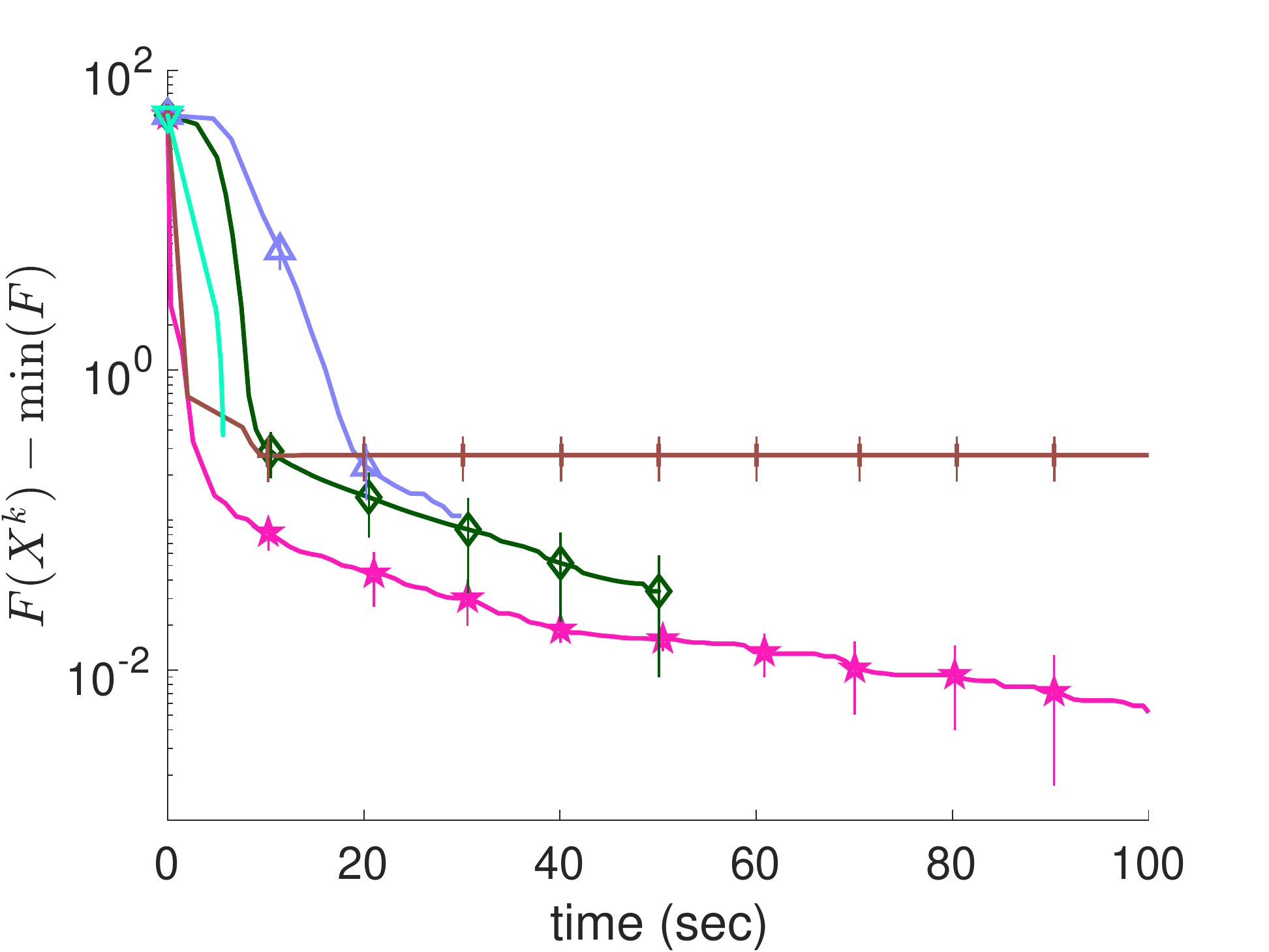}
\end{subfigure}%\hspace{10pt}
\hspace{5pt}
\begin{subfigure}{.33\textwidth}
 \centering
\includegraphics[scale=0.32]{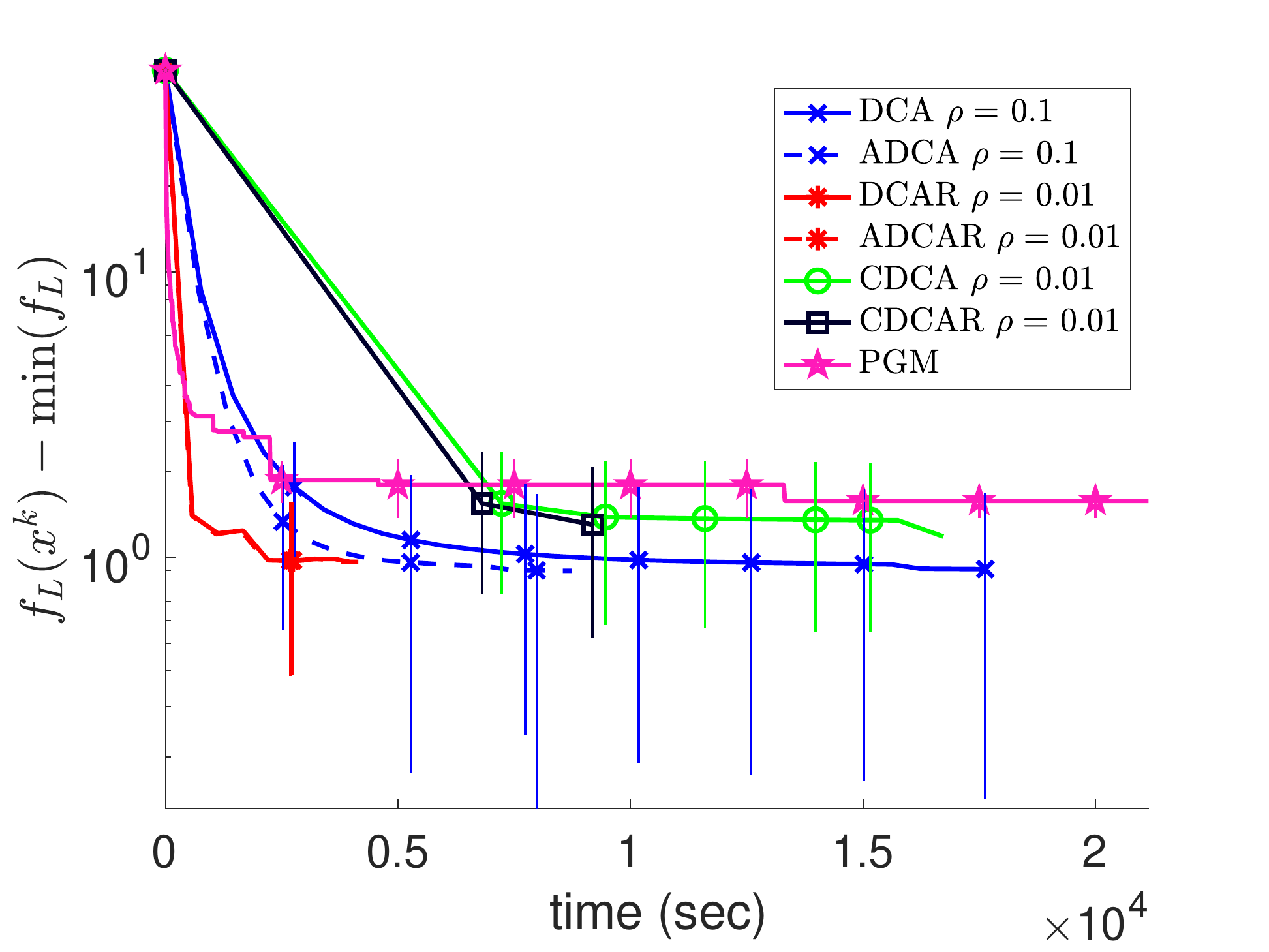}
\end{subfigure}\\
\begin{subfigure}{.33\textwidth}
 \centering
\includegraphics[scale=0.32]{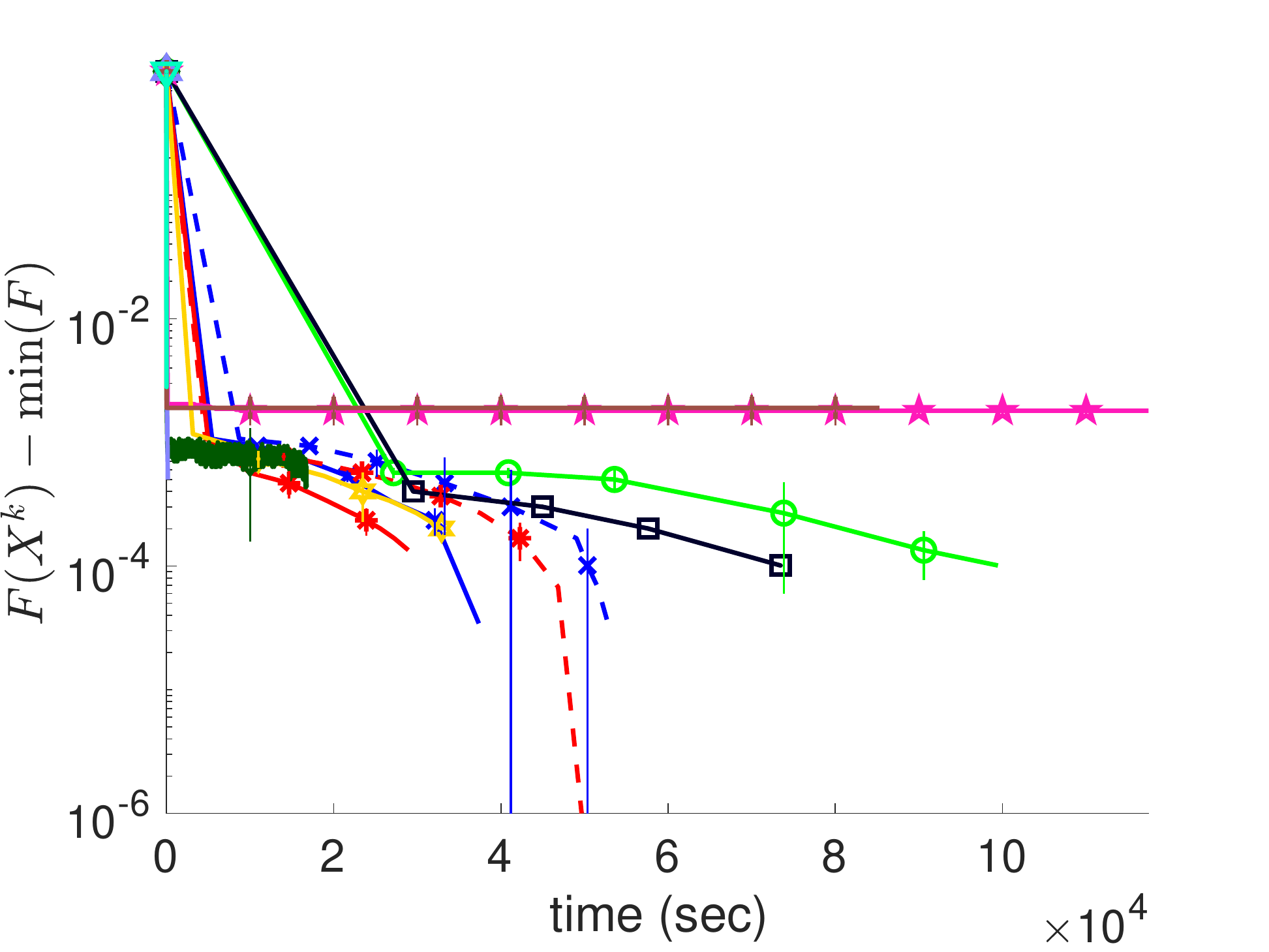}
\end{subfigure}
\begin{subfigure}{.32\textwidth}
 \centering
\includegraphics[scale=0.3]{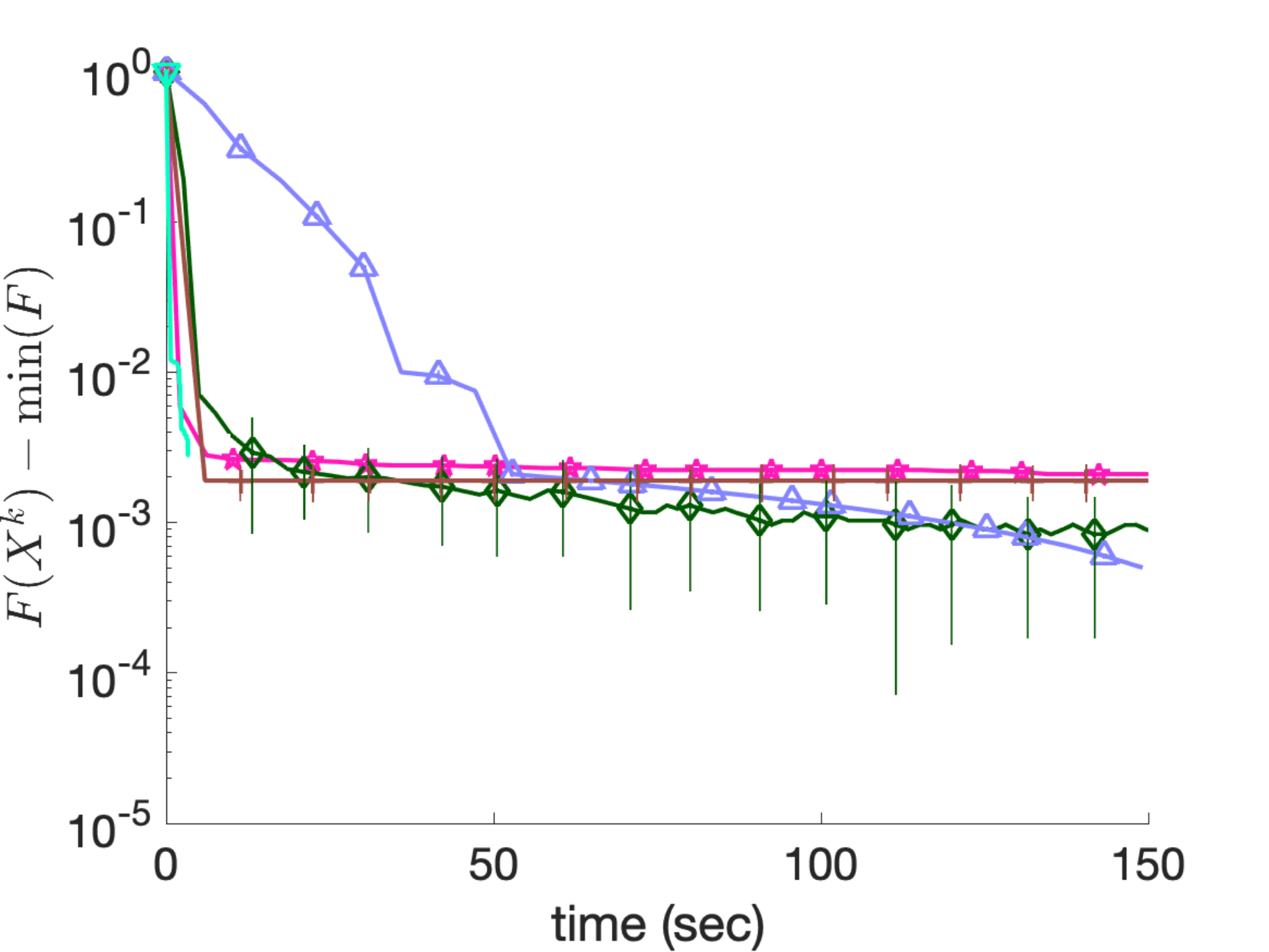}
\end{subfigure}%\hspace{10pt}
\hspace{5pt}
\begin{subfigure}{.33\textwidth}
 \centering
\includegraphics[scale=0.32]{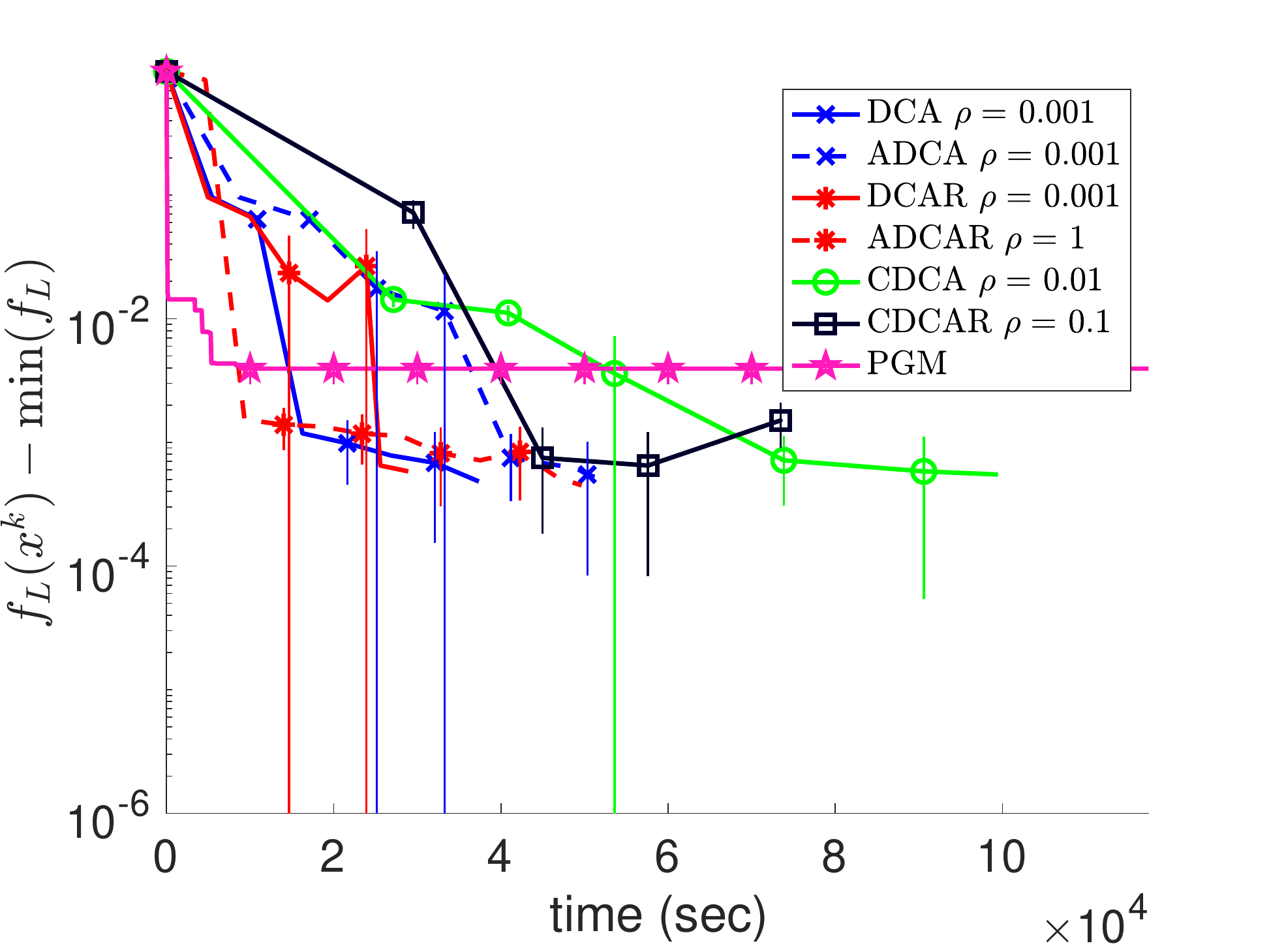}
\end{subfigure}\\
\caption{\label{fig:obj-time-bestrho} Discrete and continuous objective values (log-scale) vs time on speech (top) and mushroom (bottom) datasets. We include separate plots for non-DCA variants for visibility.}
\end{figure}

We report in \cref{fig:obj-time-bestrho} the discrete and continuous objective values with respect to time. We again only include our methods with the $\rho$ value achieving the smallest discrete objective. As expected, DCA variants (including SubSup) have a significantly higher computational cost compared to other baselines.  
%SupSub oscillates between 2 objective values after 11 iterations because generated sets are not local min, and since submax is not exact objective actually increases slightly

\looseness=-1 Recall that SubSup is a special case of DCA with $\rho=0$ and $x^k$ chosen to be integral (see \cref{sec:SubSupDCA} and the computational complexity discussion in \cref{sec:DCA-complexity}), so theoretically the cost of SubSup is the same as DCA  with $\rho=0$. In our experiments, we are using the MNP algorithm for the submodular minimization in SubSup $\min_{X \subseteq V} G(X) - y^k(X) = \min_{x \in [0,1]^d} g_L(x) - \ip{x}{y^k}$, and PGM to solve Problem \eqref{eq:DCASet-x} $\min_{x \in [0,1]^d} g_L(x) - \ip{x}{y^k} + \tfrac{\rho}{2} \| x \|^2 $ in DCA (MNP cannot be used for this problem when $\rho >0$). MNP requires $O({d ~\diam(B(G - y^k))^2}/{\epsilon_x^2})$ iterations to obtain an $\epsilon_x$-solution to $\min_{X \subseteq V} G(X) - y^k(X)$ \citep[Theorems 4 and 5]{Chakrabarty2014}. We can bound $\diam(B(G - y^k)) \leq 2 \kappa$, where recall that $\kappa$ is the Lipschitz constant of $ g_L(x) - \ip{x}{y^k}$ given in \cref{prop:LEproperties}-\ref{itm:Lip}. Hence MNP requires the same number of iterations $O(d \kappa^2 /\epsilon_x^2)$ as PGD with $\rho=0$, and the time per iteration of MNP is $O(d^2 + d \log d + d ~\text{EO}_G)$ \citep[Proof of Theorem 1]{Chakrabarty2014}, which is larger than PGD $O(d \log d + d ~\text{EO}_G)$ (see the computational complexity discussion in \cref{sec:DCA-complexity}).
Nevertheless, in our experiments, we observe that SubSup actually has a lower running time per iteration than DCA on the speech dataset; this is true even for DCA with $\rho=0$ (see \cref{fig:SubSupDCA}), but similar on the mushroom dataset.

\subsection{SubSup vs DCA and DCAR with $\rho=0$}\label{sec:SubSupDCAexp}

\begin{figure*}
%\vspace{-25pt}
\centering
\begin{subfigure}{.37\textwidth}
 \centering
\includegraphics[scale=0.34]{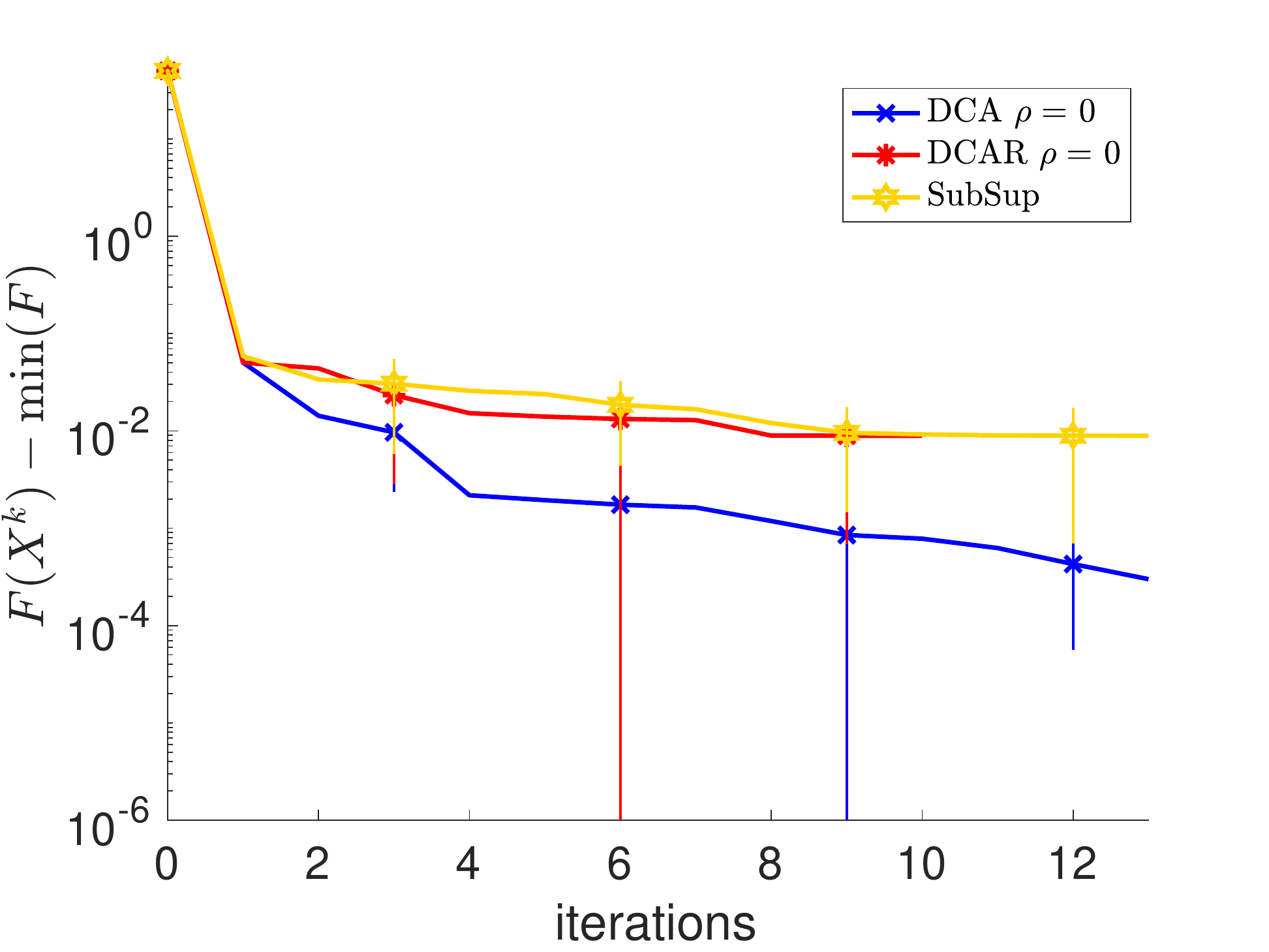}
\end{subfigure}\hspace{15pt}
\begin{subfigure}{.37\textwidth}
 \centering
\includegraphics[scale=0.34]{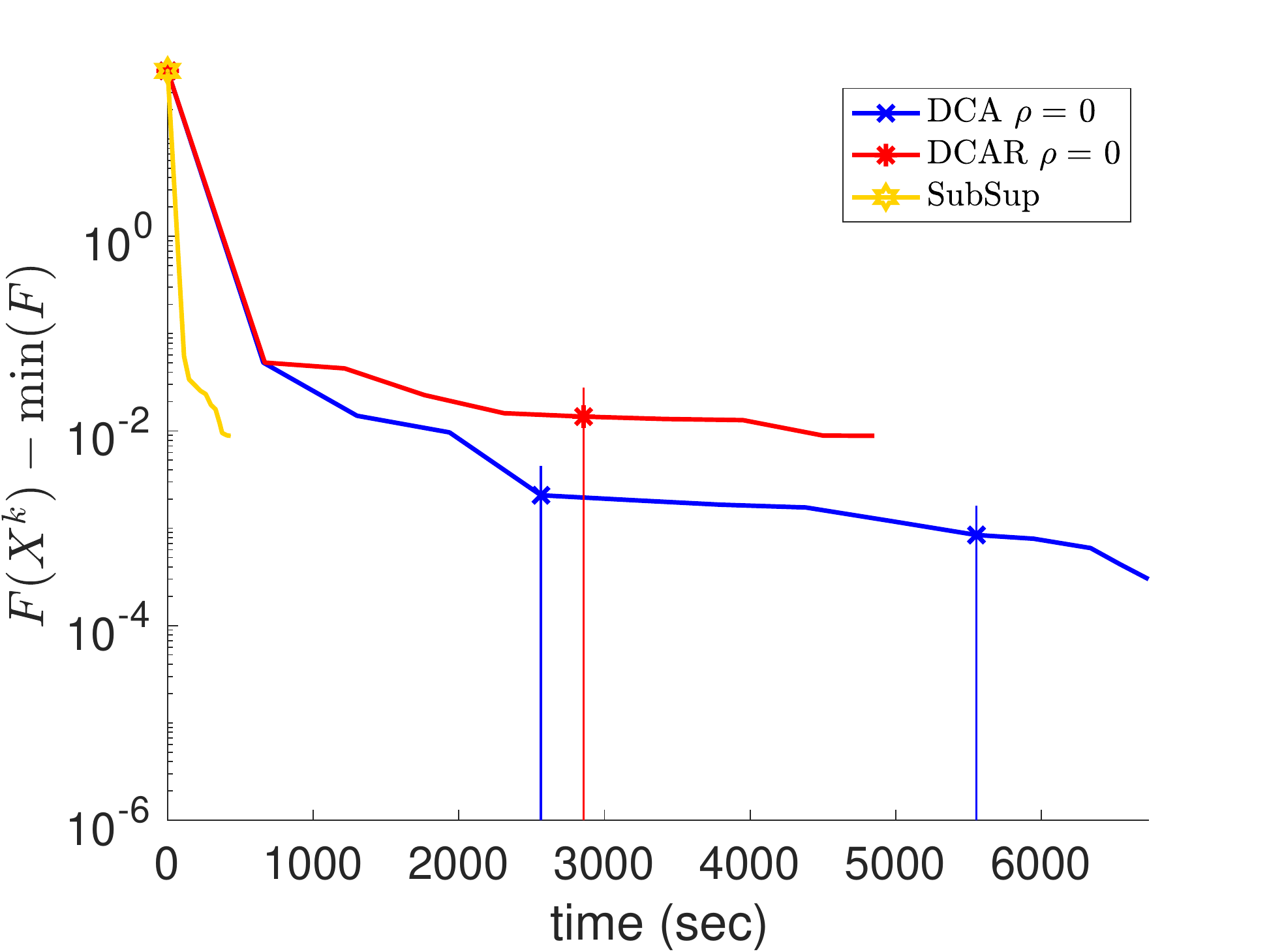}
\end{subfigure}\\ 
\begin{subfigure}{.37\textwidth}
 \centering
\includegraphics[scale=0.34]{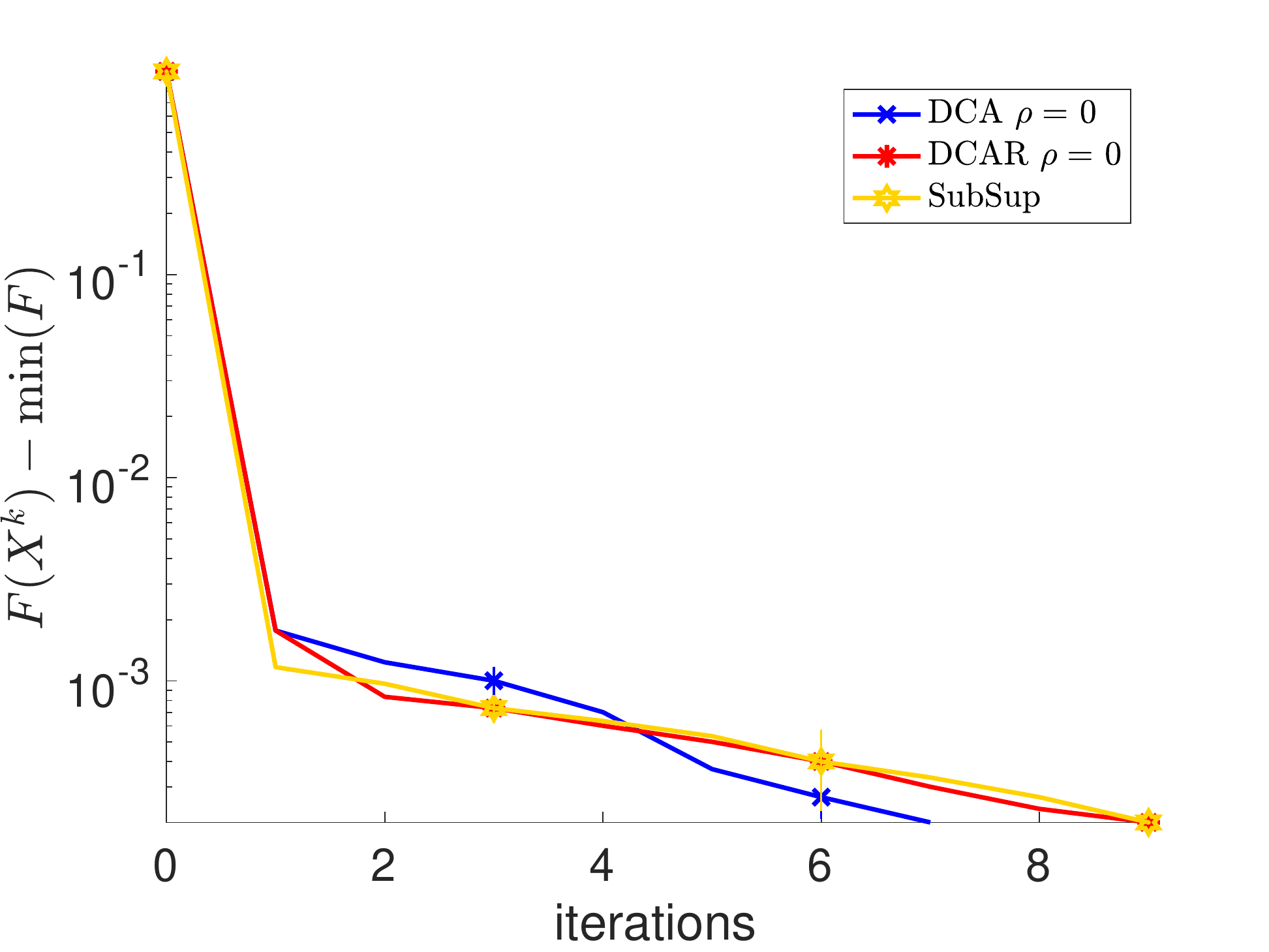}
\end{subfigure}\hspace{15pt}
\begin{subfigure}{.37\textwidth}
 \centering
\includegraphics[scale=0.34]{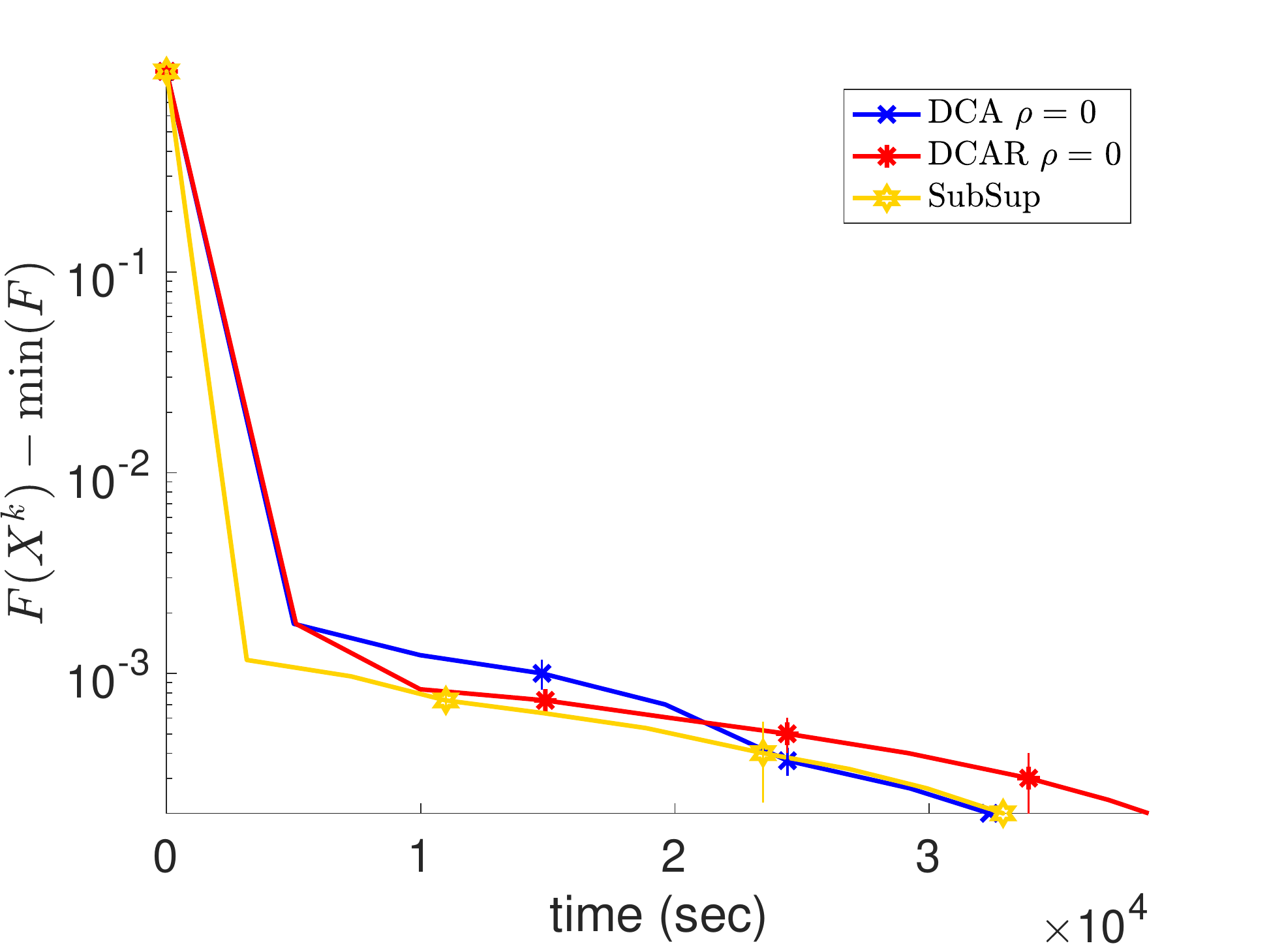}
\end{subfigure}
\caption{\label{fig:SubSupDCA} Discrete objective values (log-scale) of three DCA variants with $\rho=0$, vs iterations (left) and time (right), on speech (top) and mushroom (bottom) datasets.}
\end{figure*}

\looseness=-1 In this section, we compare the performance of SubSup with DCA and DCAR with $\rho=0$. We plot the discrete objective values of these three algorithms with respect to both iterations and time in \cref{fig:SubSupDCA}. We observe that SubSup performs similarly to DCAR with $\rho=0$ in terms of $F$ values, while DCA with $\rho=0$ obtains a bit better $F$ values on the speech dataset. 
Note that the only difference between SubSup and DCA with $\rho=0$ is that SubSup is choosing an integral solution in Problem \eqref{eq:DCASet-x}, using the MNP algorithm, while DCA chooses a possibly non-integral solution using the PGM algorithm. %well they also use different stopping criteria, but even in speech-lbd1.0e+00-f16a488-52507411 where we were using iterates stopping criterion in DCA (though with epsilon=1e-6, not 0 like in SubSup, but that should not make a big difference), DCA was still doing better
Hence, it seems that 
there is some advantage to not restricting the $x^k$ iterates of DCA to be integral in some cases. In terms of running time, SubSup has a lower iteration time than the other two algorithms on the speech dataset, and a similar one on the mushroom dataset (see discussion in \cref{sec:runtimes}).

\section{Proofs of \cref{sec:DCA}} \label{sec:DCA-proofs}

\subsection{Proof of \cref{thm:convergence}} \label{sec:DCAconv-proof}
Before proving \cref{thm:convergence}, we need the following lemma.

\begin{lemma}[Lemma 5 in \cite{pham2022}] \label{lem:eps-subdiff-strcvx}
Let $\Phi$ be a $\rho$-strongly convex function with $\rho \geq 0$, then for any  $\epsilon \geq 0$, $t \in (0,1], x \in \dom \Phi$ and $y \in \p_\epsilon \Phi(x)$, we have \[ \Phi(z) \geq \Phi(x) + \ip{y}{z - x} + \frac{\rho (1- t)}{2} \| z - x \|^2 - \frac{\epsilon}{t}, \quad \forall z  \in \R^d. \]
\end{lemma}

We now present a more detailed version of \cref{thm:convergence} and its proof.

\begin{theorem}\label{thm:convergence-app} Given any $f=g-h$, where $g,h\in \Gamma_0$, 
let $\{x^k\}$ and $\{y^k\}$ be generated by approximate DCA (\Cref{alg:DCA}), %$f^\star = \min_{x \in \R^d} f(x)$, 
and define $T_\Phi(x^{k+1}) = \Phi(x^k) - \Phi(x^{k+1}) -  \langle y^k, x^k - x^{k+1} \rangle$ for any $\Phi \in \Gamma_0$, %$T_h(x^{k+1}) = h(x^k) - h(x^{k+1}) -  \langle y^k, x^k - x^{k+1} \rangle$
% We say that $x^k$ is an $\epsilon$-critical point if it satisfies $\min\{T_g(x^{k+1}), -T_h(x^{k})\} \leq \epsilon$. 
Then for all $t_x, t_y \in (0,1], k\in\bN$, let $\bar{\rho} = \rho(g)(1-t_x)+\rho(h)(1-t_y)$ and $\bar{\epsilon} =  \tfrac{\epsilon_x}{t_x} + \tfrac{\epsilon_y}{t_y}$, we have:
\begin{enumerate}[label=\alph*),ref=\alph*] 

\item \label{itm:criticality-app} 
$T_g(x^{k+1}) \geq \tfrac{\rho(g) (1-t_x)}{2} \| x^k - x^{k+1} \|^2 - \tfrac{\epsilon_x}{t_x} \text{ and } T_h(x^{k+1}) \leq -\tfrac{\rho(h)(1-t_y)}{2} \| x^k - x^{k+1} \|^2 + \tfrac{\epsilon_y}{t_y}.$  %for all $t \in (0,1], k\in\bN$. 

Moreover, for any $\epsilon \geq 0$, if $T_g(x^{k+1}) \leq \epsilon$, then $x^k$ is an $(\epsilon+ \epsilon_x, \epsilon_y)$-critical point of $g - h$, with $y^k \in \partial_{\epsilon + \epsilon_x} g(x^k) \cap \p_{\epsilon_y} h(x^k)$, and
$\tfrac{\rho(g)(1-t_x)}{2} \| x^k - x^{k+1} \|^2 \leq \tfrac{\epsilon_x}{t_x} + \epsilon $.
Conversely, if $x^k \in \partial_{\epsilon + \epsilon_x} g^*(y^k)$, then  $T_g(x^{k+1}) \leq \epsilon_x + \epsilon$.

Similarly,  
%for all $t \in (0,1], k\in\bN,$ for any $\epsilon \geq 0$, 
if $T_h(x^{k+1}) \geq -\epsilon$, then $x^{k+1}$ is an $(\epsilon_x, \epsilon + \epsilon_y)$-critical point of $g - h$, with  $y^k \in \p_{\epsilon_x} g(x^{k+1}) \cap \p_{\epsilon+\epsilon_y} h(x^{k+1})$, and $\tfrac{\rho(h)(1-t_y)}{2} \| x^k - x^{k+1} \|^2 \leq \tfrac{\epsilon_y}{t_y} + \epsilon $.
Conversely, if $y^k \in \partial_{\epsilon + \epsilon_y} h(x^{k+1})$, then $T_h(x^{k+1}) \geq - \epsilon_y - \epsilon$.
    
\item \label{itm:descent-app} $\begin{aligned} f(x^k) - f(x^{k+1})  =  T_g(x^{k+1}) - T_h(x^{k+1}) \geq \frac{\bar{\rho}}{2}\|x^k-x^{k+1}\|^2 - \bar{\epsilon}.\end{aligned}$ % for all $t_x, t_y \in (0,1], k\in\bN$.
    
\item \label{itm:convergence-app} For any $\epsilon \geq 0$, $f(x^k) - f(x^{k+1}) \leq \epsilon$ if and only if $T_g(x^{k+1}) - T_h(x^{k+1}) \leq \epsilon$. In this case,  
$x^k$ is an $(\epsilon_1 + \epsilon_x , \epsilon_y)$-critical point of $g - h$, with $y^k \in \partial_{\epsilon_1 + \epsilon_x} g(x^k) \cap \p_{\epsilon_y} h(x^k)$, $x^{k+1}$ is an $(\epsilon_x, \epsilon_2 + \epsilon_y)$-critical point of $g - h$, with  $y^k \in \p_{\epsilon_x} g(x^{k+1}) \cap \p_{\epsilon_2+\epsilon_y} h(x^{k+1})$, for some $\epsilon_1 + \epsilon_2 = \epsilon, \epsilon_1 \geq - \epsilon_x, \epsilon_2 \geq  -\epsilon_y$, and $\frac{\bar{\rho}}{2}\|x^k-x^{k+1}\|^2 \leq \bar{\epsilon} + \epsilon$. % for all $t_x, t_y \in (0,1], k\in\bN$. 
%^with $y^k \in \p g(x^k) \cap \p h(x^k)$ and $y^k \in \p g(x^{k+1}) \cap \p h(x^{k+1})$. 
Conversely, if $x^k \in\p_{\epsilon_x + \epsilon_1} g^*(y^k)$ and $y^k\in \p_{\epsilon_y + \epsilon_2} h(x^{k+1})$, then
$T_g(x^{k+1}) - T_h(x^{k+1}) \leq \epsilon_x + \epsilon_y + \epsilon$ and $f(x^k) - f(x^{k+1})  \leq \epsilon_x + \epsilon_y + \epsilon$.

\item \label{itm:obj-rate-app} $ \begin{aligned} \min_{k \in \{0,1,\dots,K-1\}} T_g(x^{k+1}) - T_h(x^{k+1}) =  \min_{k \in \{0,1,\dots,K-1\}} f(x^k)-f(x^{k+1}) \leq \frac{f(x^0) - f^\star}{K}. \end{aligned}$

\item \label{itm:iterates-rate-app} If $\rho(g) + \rho(h) >0$, then
%for all $t_x, t_y \in (0,1]$, we have
    \[
    \min_{k \in \{0,1,\dots,K-1\}}\|x^k-x^{k+1}\|\leq \sqrt{\frac{ 2}{\bar{\rho}} \big( \frac{f(x^0) - f^\star}{K} + \bar{\epsilon} \big). } 
    \]

%\item  If we use $f(x^k) - f(x^{k+1}) \leq \epsilon$ as a stopping criterion,  DCA stops after ${(f(x^0) - f^\star})~/~{\epsilon}$ iterations. If we use $f(x^k) (1 + \epsilon) \leq f(x^{k+1})$ as a stopping criterion and start with $f(x^0) \leq 0$,  DCA stops after $O\big({\log(|f^\star|/|f(x^1)|)} ~/~ {\epsilon}\big)$ iterations.

%\item If $\epsilon_x=\epsilon_y=0$ and we use $f(x^k) = f(x^{k+1})$ as a stopping criterion, and $g$ or $h$ is polyhedral convex, the sequences $\{x^k\}$ and $\{y^k\}$ contain finitely many elements, and DCA stops after finitely many iterations. 

%    \item If $f$ is $L$-Lipschitz continuous on its domain, then we have
%      \[
%    \min_{k \in \{0,1,\dots,K-1\}} f(x^k)-f(x^{k+1}) \leq L \sqrt{\frac{2(f(x^0)-f^\star)}{K(\rho(g)+\rho(h))}}
%    \]

\end{enumerate}
\end{theorem} 
\begin{proof}
\begin{enumerate}[label=\alph*)]
\item  
Since $x^{k+1} \in \partial_{\epsilon_x} g^*(y^k)$, then $y^k \in \partial_{\epsilon_x} g(x^{k+1})$ by \cref{prop:FenchelDualPairs}. By \cref{lem:eps-subdiff-strcvx} we have for all $x \in \R^d$ %, t \in (0,1]$
\begin{equation}\label{eq:Tg-lb}
g(x) \geq g(x^{k+1}) + \ip{y^k}{x - x^{k+1}} + \tfrac{\rho(g) (1 - t_x)}{2} \| x - x^{k+1} \|^2 - \tfrac{\epsilon_x}{t_x},
\end{equation}
hence $T_g(x^{k+1}) \geq \tfrac{\rho(g)(1-t_x)}{2} \| x^k - x^{k+1} \|^2 - \tfrac{\epsilon_x}{t_x}$. If $T_g(x^{k+1}) \leq \epsilon$, taking $t_x=1$ in \eqref{eq:Tg-lb}, we have for all $x \in \R^d$
\[ g(x) \geq g(x^{k}) -  \langle y^k, x^k - x^{k+1} \rangle  - \epsilon + \ip{y^k}{x - x^{k+1}}  - \epsilon_x \green{=}  g(x^{k}) +  \langle y^k, x - x^{k} \rangle - \epsilon - \epsilon_x, \]
so $y^k \in \partial_{\epsilon + \epsilon_x} g(x^k) \cap \partial_{\epsilon_y} h(x^k)$ and $x^k$ is an $(\epsilon + \epsilon_x, \epsilon_y)$-critical point.
Similarly, since $\green{y^k} \in \p_{\epsilon_y} h(x^k)$, we have for all $x \in \R^d$ %, t \in (0,1]$
\begin{equation}\label{eq:Th-ub}
h(x) \geq h(x^k) + \ip{y^k}{x - x^k} + \tfrac{\rho(h) (1 - t_y)}{2} \| x - x^k \|^2 - \tfrac{\epsilon_y}{t_y},
\end{equation}
hence $T_h(x^{k+1}) \leq -\tfrac{\green{\rho(h)} (1 - t_y)}{2} \| x^k - x^{k+1} \|^2 + \tfrac{\epsilon_y}{t_y}$. If $T_h(x^{k+1}) \geq -\epsilon$, taking $t_y=1$ in \eqref{eq:Th-ub}, we have for all $x \in \R^d$
\[ h(x) \geq h(x^{k+1}) + \langle y^k, x^k - x^{k+1} \rangle  - \epsilon + \ip{y^k}{x - x^{k}} -\epsilon_y \green{=} h(x^{k+1}) +  \langle y^k, x - x^{k+1} \rangle - \epsilon -\epsilon_y, \]
so $y^k \in \p_{\epsilon_x} g(x^{k+1}) \cap \p_{\epsilon+\epsilon_y} h(x^{k+1})$  and $x^{k+1}$ is an $(\epsilon_x, \epsilon + \epsilon_y)$-critical point. The converse directions follow directly from the definitions of approximate subdifferentials and $T_\Phi$.

\item  We have
\begin{align*}\label{eq:descent}
f(x^k) - f(x^{k+1})  &= g(x^k)  - g(x^{k+1}) + \ip{y^k}{x^k - x^{k+1}}  
- (h(x^k) - h(x^{k+1}) + \ip{y^k}{x^k - x^{k+1}}) \nonumber \\
&=  T_g(x^{k+1}) - T_h(x^{k+1}) \nonumber \\
& \geq \tfrac{\bar{\rho}}{2} \| x^k - x^{k+1} \|^2 - \bar{\epsilon},
\end{align*}
where the inequality follows from \cref{itm:criticality-app}.
\item % This follows from Items \ref{itm:criticality-app} and \ref{itm:descent-app}. 
By \cref{itm:descent-app}, we directly get that $f(x^k) - f(x^{k+1}) \leq \epsilon$ if and only if $T_g(x^{k+1}) - T_h(x^{k+1}) \leq \epsilon$, and that $\frac{\bar{\rho}}{2} \|x^k - x^{k+1}\|^2 \leq \bar{\epsilon} + \epsilon$ in this case.
The rest of the claim follows from \cref{itm:criticality-app}. \\
In particular, if $T_g(x^{k+1}) - T_h(x^{k+1}) \leq \epsilon$, we can find some $\epsilon_1 \geq T_g(x^{k+1})$, $\epsilon_2 \geq - T_h(x^{k+1})$ with $\epsilon_1 + \epsilon_2 = \epsilon$ (in the worst case, we might choose $\epsilon_1 = T_g(x^{k+1})$, $\epsilon_2 = - T_h(x^{k+1})$). 
By \cref{itm:criticality-app}, since $T_g(x^{k+1}) \leq \epsilon_1$, we have that $x^{k+1}$ is an $(\epsilon_1 + \epsilon_x, \epsilon_y)$-critical point with $y^k \in \partial_{\epsilon_1 + \epsilon_x}g(x^k) \cap \partial_{\epsilon_y}h(x^k)$. Similarly, since $T_h(x^{k+1}) \geq - \epsilon_2$, we have that $x^{k+1}$ is an $(\epsilon_x, \epsilon_2 + \epsilon_y)$-critical point of $g-h$ with $y^k \in \partial_{\epsilon_x}g(x^{k+1}) \cap \partial_{\epsilon_2 + \epsilon_y}h(x^{k+1})$. Moreover,
taking $t_x = t_y = 1$ in the bounds on $T_g(x^{k+1})$ and $ T_h(x^{k+1})$ in \cref{itm:criticality-app}, we observe that $\epsilon_1 \geq -\epsilon_x$, $\epsilon_2  \geq -\epsilon_y$.\\
% By our assumption that $T_g(x^{k+1}) - T_h(x^{k+1}) \leq \epsilon$, we may use \cref{itm:descent-app} to conclude that
% %    \begin{equation*}
%         \epsilon \geq T_g(x^{k+1}) - T_h(x^{k+1}) \geq \frac{\bar{\rho}}{2} \|x^k - x^{k+1}\|^2 - \bar{\epsilon}
%     \end{equation*}
% hence $\frac{\bar{\rho}}{2} \|x^k - x^{k+1}\|^2 \leq \bar{\epsilon} + \epsilon$.
The converse direction follows from the converse direction of  \cref{itm:criticality-app}, whereby $x^k \in \partial_{\epsilon_x + \epsilon_1} g^*(y^k)$ implies that $T_g(x^{k+1}) \leq  \epsilon_x + \epsilon_1$, and $y^k \in \partial_{\epsilon_y +\epsilon_2}h(x^{k+1})$ implies that $-T_h(x^{k+1}) \leq \epsilon_y +\epsilon_2$.

\item This follows from \cref{itm:descent-app} by telescoping sum:
\begin{align*}
\min_{k \in \{0,1,\dots,K-1\}} T_g(x^{k+1}) - T_h(x^{k+1}) &=  \min_{k \in \{0,1,\dots,K-1\}} f(x^k)-f(x^{k+1}) \\
&\leq \tfrac{1}{K} \sum_{k=0}^{K-1} T_g(x^{k+1}) - T_h(x^{k+1})\\
&= \tfrac{1}{K} \sum_{k=0}^{K-1} f(x^k)-f(x^{k+1})  \\
&= \frac{f(x^0) - f(x^K)}{K} \leq \frac{f(x^0) - f^\star}{K}.
\end{align*}

\item This follows from Items \ref{itm:descent-app} and \ref{itm:obj-rate-app}.
\[ \min_{k \in \{0,1,\dots,K-1\}}\|x^k-x^{k+1}\|^2 \leq  \min_{k \in \{0,1,\dots,K-1\}} \frac{ 2}{\bar{\rho}} (f(x^k)-f(x^{k+1}) + \bar{\epsilon}) \leq \frac{ 2}{\bar{\rho}}  \big( \frac{f(x^0) - f^\star }{K} + \bar{\epsilon}  \big).\]

%\item The first complexity follows from \ref{itm:obj-rate}, and the second one follows in the same way as in \citep[Theorem 5.5]{Iyer2012a}. 

%\item See \cite{An2005} (but no proof is included). \citep[Theorem 9]{le1997solving} proves this claim for DCA with fixed choice of subgradients. This also holds if we use iterates convergence as stopping criterion, i.e., $x^{k+1} = x^k$, if $\rho >0$, or if the choice of subgradients is natural \citep[Theorem 8 and 9]{le1997solving}. {Since neither cases would apply in case of DSM problem (if $\rho>0, \tih$ and $\tig$ are no longer polyhedral, and our choice of subgradients does not satisfy $y^k \in \mathrm{ri} h(x^k)$) no need to state this.}

\end{enumerate}
\end{proof}
%A slightly improved rate on the iterates of $\sqrt{\frac{2(f(x^0)-f^\star)}{K(\rho(g)+\rho(h)) + \max\{\rho(g),\rho(h)\}}}$ was obtained in \citep[Theorem 3.1]{abbaszadehpeivasti2021rate} (see discussion after Eq. 15 therein). %by applying their convergence result for smooth DC components to the dual of the DC problem with strongly convex DC components. <-- No need to mention this

Note that $f(x^k) - f(x^{k+1})$ acts as a measure of non-criticality, since $f(x^k) = f(x^{k+1})$ implies that $x^k$ and $x^{k+1}$ are critical points, when $\epsilon_x = \epsilon_y = 0$. 
\Cref{thm:convergence-app} also motivates $\min\{T_g(x^{k+1}), T_h(x^{k})\}$ as a weaker measure of non-criticality, since $\min\{T_g(x^{k+1}), T_h(x^{k})\} = 0$ implies that $x^k$ is a critical point, when $\epsilon_x = \epsilon_y = 0$. %-\ref{itm:criticality-app}  
\Cref{itm:obj-rate-app,itm:criticality-app} imply the following bound $$\min_{k \in \{0,1,\dots,K-1\}} \min\{T_g(x^{k+1}), T_h(x^{k})\} \leq \min\{\tfrac{f(x^0) - f^\star}{K} + \epsilon_y, \tfrac{f(x^0) - f^\star}{K-1} + \epsilon_x\},$$ which recovers the convergence rate provided in \citep[Corollary 4.1]{abbaszadehpeivasti2021rate} on $T_g(x^{k+1})$, with $\epsilon_x = \epsilon_y = 0$.
The criterion $T_g(x^{k+1}) \leq \epsilon$ has also been used as a stopping criterion of FW for nonconvex problems; see  \cref{sec:FWconv-proof} and \citep[Eq. (2.6)]{Ghadimi2019}. 
%(taking $h = g, f= -h, y_{k-1}= x^{k}, x_k=x^{k+1}$).

\subsection{Proof of \cref{corr:LocMinDCA}} \label{sec:LocMinDCA-proof}

Before proving \cref{corr:LocMinDCA}, we need the following lemma.

%\begin{lemma}\label{prop:criticalpt}
%Let $x^*, x \in [0,1]^d$ be two points satisfying $\partial_{\epsilon_1} g(x^*) \cap \partial_{\epsilon_2} h(x) \not = \emptyset$ for some $\epsilon_1, \epsilon_2 \geq 0$, we have $g(x^*)-h(x^*)\leq g(x)-h(x) + \epsilon_1 + \epsilon_2$.
%\end{lemma}
%\begin{proof}
%This is an extension of \citep[Theorem 4]{le1997solving} \citep[Corollary 3.3]{Tao1998}.
%Given $y \in  \partial_{\epsilon_1} g(x^*) \cap \partial_{\epsilon_2} h(x)$, we have $g(x^*) + \ip{y}{x - x^*} - \epsilon_1 \leq g(x)$ and $h(x) + \ip{y}{x^* - x} - \epsilon_2 \leq h(x^*)$. Hence, $ g(x^*) - h(x^*) \leq g(x) - h(x) + \epsilon_1 + \epsilon_2$
%\end{proof}

\begin{lemma} \label{lem:eps-subdiff-reg}
Let $\Phi$ be a convex function with bounded domain of diameter $D$, i.e., $\|x - z \| \leq D$ for all $x, z \in \dom \Phi$, and $\tilde{\Phi} = \Phi + \frac{\rho}{2} \| \cdot \|^2$ for some $\rho \geq 0$. Then for any $x \in \dom \Phi$, if $y - \rho x \in \p_\epsilon \Phi(x)$, then $y  \in \p_\epsilon \tilde{\Phi}(x)$. Conversely, if $y  \in \p_\epsilon \tilde{\Phi}(x)$, then $y - \rho x \in \p_{\epsilon'} \Phi(x)$, where $\epsilon' = \sqrt{2 \rho \epsilon} D$ if $\epsilon \leq \tfrac{\rho D^2}{2}$, and $\tfrac{\rho D^2}{2} + \epsilon$ otherwise.
\end{lemma}
\begin{proof}

If $y - \rho x \in \p_\epsilon \Phi(x)$, we have 
\begin{align*}
\Phi(z) &\geq \Phi(x) + \ip{y - \rho x}{z - x} - \epsilon \\
\Leftrightarrow \Phi(z) &\geq \Phi(x) + \ip{y}{z - x} + \rho \| x \|^2 - \ip{\rho x}{z} +\tfrac{\rho}{2} \| z \|^2 - \tfrac{\rho}{2}\| z \|^2 - \epsilon \\
\Leftrightarrow \tilde{\Phi}(z) &\geq \tilde{\Phi}(x) + \ip{y}{z - x} +\tfrac{\rho}{2} \| x - z \|^2 - \epsilon \\
\Rightarrow  \tilde{\Phi}(z) &\geq \tilde{\Phi}(x) + \ip{y}{z - x}  - \epsilon 
\end{align*}
Hence, $y  \in \p_\epsilon \tilde{\Phi}(x)$. Conversely, if $y  \in \p_\epsilon \tilde{\Phi}(x)$, then by \cref{lem:eps-subdiff-strcvx}, we have for all $t \in (0,1], z \in \dom \Phi$
\begin{align*}
\tilde{\Phi}(z) &\geq \tilde{\Phi}(x) + \ip{y}{z - x} + \frac{\rho (1- t)}{2} \| z - x \|^2 - \frac{\epsilon}{t} \\
&\geq \tilde{\Phi}(x) + \ip{y}{z - x} + \frac{\rho}{2} \| z - x \|^2  - \frac{\rho t}{2} D^2 - \frac{\epsilon}{t} \\
\Leftrightarrow {\Phi}(z) &\geq {\Phi}(x) + \ip{y - \rho x}{z - x} - \frac{\rho t}{2} D^2 - \frac{\epsilon}{t}\\
\end{align*}
Hence, $y - \rho x \in \p_{\epsilon'} \Phi(x)$ with
\begin{equation*}
    \green{\epsilon' = \min_{t \in (0,1]} \frac{\rho t}{2} D^2 + \frac{\epsilon}{t} = \begin{cases}
                            \sqrt{2 \rho \epsilon} D & \text{if } \epsilon \leq \frac{\rho D^2}{2},\\
                            \frac{\rho D^2}{2} + {\epsilon} & \text{otherwise}.
                        \end{cases}}
\end{equation*}         
\end{proof}

\LocMinDCA*
\begin{proof}
\begin{enumerate}[label=\alph*),ref=\alph*]  %[leftmargin=1em, itemindent=1em]   
\item If $f(x^k) - f(x^{k+1}) \leq \epsilon$, we have by \cref{thm:convergence}-\ref{itm:convergence} (with $\epsilon_y=0$)  that $y^k \in \partial_{\epsilon_x+ \epsilon} g(x^k)$, which by \cref{lem:eps-subdiff-reg} implies that $y^k - \rho x^k \in \partial_{\epsilon'} (g_L + \delta_{[0,1]^d}) (x^k)$, by taking $D = \max_{x, z\in \dom(g_L + \delta_{[0,1]^d})} \|x - z\| = \sqrt{d}$. We observe that for any $\ell \in V$, we have $y^k - \rho x^k \in \p h_L(\1_{S^{\sigma}_\ell})$ by \cref{prop:LEproperties}-\ref{itm:greedy}. Hence, $\partial_{\epsilon'} (g_L + \delta_{[0,1]^d})(x^k) \cap \p h_L(\1_{S^{\sigma}_\ell}) \not = 0$, and by \cref{prop:LocalOptCondition}-\ref{itm:general} $f(x^k) \leq f(\1_{S^{\sigma}_\ell}) + \epsilon'$. The statement then follows by \cref{prop:LEproperties}-\ref{itm:extension}.
\item 
Note that $y^k_{\sigma_{i}}, x^{k+1}_{\sigma_{i}}$ for any $i \in V$ are valid iterates for approximate DCA, so \cref{itm:noPerms} apply to them.   
If $f(x^k) - f(x^{k+1}) \leq \epsilon$, then $f(x^{k}) - f(x^{k+1}_{\sigma_{i}}) \leq \epsilon$ since $f(x^{k+1}) \leq f(x^{k+1}_{\sigma_{i}})$ for all $i \in V$. Hence, by \cref{itm:noPerms} we have $F(X^k) \leq F(S^{\sigma_i}_\ell) + \epsilon'$ for all $i, \ell \in V$.
We now observe that for any $j \in X^k$ there exists $\sigma_i$ for some $i \in V$, such that $\sigma_i(|X^k|) = j$, and $S^{\sigma_i}_{|X^k|-1} = X^k \setminus j$. Similarly for any $j \in V \setminus X^k$, there exists $\sigma_i$ for some $i \in V$, such that $\sigma_i(|X^k|+1) = j$, and $S^{\sigma_i}_{|X^k|+1} = X^k \cup j$. Then $X^k$ is an $\epsilon'$-local minimum of $F$.

%Hence, by \cref{thm:convergence}-\ref{itm:convergence} we have that $y^k \in \partial_{\epsilon_x+ \epsilon} g(x^k)$, which by \cref{lem:eps-subdiff-reg} implies that $y^k - \rho x^k \in \partial_{\epsilon'} (g_L + \delta_{[0,1]^d}) (x^k)$. We now observe that for any $x = \1_{X^k \setminus i}, i \in X^k$ there exists $\sigma_j$ for some $j \in V$, such that $\sigma_j(|X^k|) = i$, hence $\sigma_j$ is a decreasing order of $x$, and $y^k_{\sigma_j} - \rho x^k \in \p h_L(x)$. Then $\partial_{\epsilon'} (g_L + \delta_{[0,1]^d})(x^k) \cap \p h(x) \not = 0$, and by \cref{prop:LocalOptCondition}-\ref{itm:general} $f(x^k) \leq f(x) + \epsilon'$. Similarly for any $x = \1_{X^k \cup i}, i \in V \setminus X^k$. The statement then follows from the fact that the Lov\'asz extension $f_L$ coincide with $F$ on the vertices of $[0,1]^d$.
\end{enumerate}
\end{proof}

\subsection{Convergence properties of DCA variants}\label{sec:DCARconv-proof}

In this section, we present convergence properties of the DCA variants discussed in \cref{sec:DCA}. We start by the DCA variant presented in \cref{alg:DCASetwithLocalMinCheck}, where at convergence we explicitly check if rounding the current iterate yields an $\epsilon'$-local minimum of $F$, and if not we restart from the best neighboring set. 
\begin{algorithm}[h] 
\caption{Approximate DCA with local minimality stopping criterion} \label{alg:DCASetwithLocalMinCheck}
\begin{algorithmic}[1]
 \STATE  $\epsilon, \epsilon', \epsilon_x \geq 0, x^0 \in\dom \p h $, $k \gets 0$.
 \FOR{$k = 1, 2, \cdots, K$}
 \STATE $y^{k}\in \p h(x^k)$ 
 \STATE $x^{k+1}\in \p_{\epsilon_x} g^*(y^k)$
\IF{$f(x^k) - f(x^{k+1}) \leq \epsilon$}
\STATE $X^{k+1} = \round(x^{k+1})$
\IF{$X^{k+1}$ is an $\epsilon'$-local minimum of $F$}
\STATE Stop
\ELSE
\STATE $x^{k+1} = \1_{\hat{X}^{k+1}}$ where $\hat{X}^{k+1} = \argmin_{|X \Delta X^{k+1}| = 1} F(X)$
\ENDIF 
\ENDIF 
 \ENDFOR
\end{algorithmic}
 \end{algorithm}

\begin{proposition}\label{prop:convergence-localmin}
Given $f= g -h$ as defined in \eqref{eq:DS-DCdecomposition} and $\epsilon' \geq \epsilon + \epsilon_x$, \cref{alg:DCASetwithLocalMinCheck}
%$\epsilon' = \sqrt{2 \rho d (\epsilon + \epsilon_x)}$ if $\epsilon + \epsilon_x \leq \tfrac{\rho d}{2}$ and $\tfrac{\rho d}{2} + \epsilon + \epsilon_x$ otherwise. This variant of DCA 
converges to an $\epsilon'$-local minimum of $F$ after at most $(f(x^0) - f^\star)/\epsilon$ iterations. 
\end{proposition}
\begin{proof}
Note that between each restart (line 10), \cref{alg:DCASetwithLocalMinCheck} is simply running approximate DCA, so \cref{thm:convergence} applies.
For any iteration $k \in \bN$, if the algorithm did not terminate, then either  $f(x^k) - f(x^{k+1}) > \epsilon$ or $X^{k+1}$ is not an $\epsilon'$-local minimum of $F$ and thus $F(X^{k+1}) > F(\hat{X}^{k+1}) + \epsilon'$. In the second case, we have 
\begin{align*}
f(\1_{\hat{X}^{k+1}}) = F(\hat{X}^{k+1}) &< F(X^{k+1})- \epsilon' &\text{(by \cref{prop:LEproperties}-\ref{itm:extension})}\\
&\leq f(x^{k+1})- \epsilon' &\text{(by \cref{prop:LEproperties}-\ref{itm:round})} \\
&\leq f(x^k) + \epsilon_x - \epsilon'   &\text{(by \cref{thm:convergence}-\ref{itm:descent} with $t_x = 1, \epsilon_y = 0$)} \\
&\leq  f(x^k) - \epsilon   &\text{(since $\epsilon' \geq \epsilon + \epsilon_x$)}
\end{align*}
Hence, the new $x^{k+1} = \1_{\hat{X}^{k+1}}$ will satisfy
$f(x^k) - f(x^{k+1}) > \epsilon$. Thus $f^\star < f(x^k) < f(x^0) - k \epsilon$ and $k < (f(x^0) - f^\star)/\epsilon$.
\end{proof}

Next we present convergence properties of approximate \DCAR \eqref{eq:DCAround}.

\begin{theorem} \label{them:convergence-round-app}
Given $f=g-h$ as defined in \eqref{eq:DS-DCdecomposition}, 
let $\{x^k\}, \{X^k\}, \{\tilde x^{k}\}$ and $\{y^k\}$ be generated by approximate \DCAR \eqref{eq:DCAround}, %$f^\star = \min_{x \in \R^d} f(x)$, 
and define $T_\Phi(x^{k+1}) = \Phi(x^k) - \Phi(x^{k+1}) -  \langle y^k, x^k - x^{k+1} \rangle$ for any $\Phi \in \Gamma_0$, %$T_h(x^{k+1}) = h(x^k) - h(x^{k+1}) -  \langle y^k, x^k - x^{k+1} \rangle$
% We say that $x^k$ is an $\epsilon$-critical point if it satisfies $\min\{T_g(x^{k+1}), -T_h(x^{k})\} \leq \epsilon$. 
Then for all $t_x \in (0,1], k\in\bN$, we have:
% I removed things with  T_{g}(x^{k+1}) and T_{h}(x^{k+1}), since they're not used for anything..
\begin{enumerate}[label=\alph*),ref=\alph*] 

\item \label{itm:criticality-round-app}  $T_g(\tilde x^{k+1}) \geq \tfrac{\rho (1-t_x)}{2} \| x^k - \tilde x^{k+1} \|^2 - \tfrac{\epsilon_x}{t_x}$, and $T_{h}(\tilde x^{k+1}) \leq -\tfrac{\rho}{2} \| x^k - \tilde x^{k+1} \|^2$. % and $T_{h}(x^{k+1}) \leq -\tfrac{\rho}{2} \| x^k - x^{k+1} \|^2$.

Moreover, for any $\epsilon \geq 0$, if $T_g(\tilde x^{k+1}) \leq \epsilon $, then $x^k$ is an $(\epsilon+ \epsilon_x, 0)$-critical point of $g - h$, with $y^k \in \partial_{\epsilon + \epsilon_x} g(x^k) \cap \p h(x^k)$, and
$\tfrac{\rho(1-t_x)}{2} \| x^k - \tilde x^{k+1} \|^2 \leq \tfrac{\epsilon_x}{t_x} + \epsilon $.
Conversely, if $x^k \in \partial_{\epsilon + \epsilon_x} g^*(y^k)$, then $T_g(\tilde x^{k+1}) \leq \epsilon_x + \epsilon$.

Similarly, if $T_{h}(\tilde x^{k+1}) \geq - \epsilon$, then $\tilde x^{k+1}$ is an $(\epsilon_x, \epsilon)$-critical point of $g - h$, with  $y^k \in \p_{\epsilon_x} g(\tilde x^{k+1}) \cap  \partial_\epsilon h(\tilde x^{k+1})$, and $\tfrac{\rho}{2} \| x^k - \tilde x^{k+1} \|^2 \leq \epsilon$. %, and $T_{h}( x^{k+1}) \geq - \epsilon \Rightarrow y^k \in \partial_\epsilon h(x^{k+1}), \tfrac{\rho}{2} \| x^k -  x^{k+1} \|^2 \leq \epsilon$. 
Conversely, if $y^k \in \partial_{\epsilon} h(\tilde x^{k+1})$, then $T_h(\tilde x^{k+1}) \geq  - \epsilon$. % and $y^k \in \partial_{\epsilon} h(x^{k+1}) \Rightarrow T_h(x^{k+1}) \geq - \epsilon$.

\item \label{itm:descent-round-app} $F(X^k) - F(X^{k+1})$ %= T_{g}(x^{k+1}) - T_{h}(x^{k+1}) 
$\geq f(x^k)-f(\tilde x^{k+1}) =  T_{g}(\tilde x^{k+1}) - T_{h}(\tilde x^{k+1}) \geq \tfrac{\rho (2-t_x)}{2} \|x^k-\tilde x^{k+1}\|^2 - \tfrac{\epsilon_x}{t_x}$.

\item \label{itm:convergence-round-app} For any $\epsilon \geq 0, F(X^k) - F(X^{k+1}) \leq \epsilon$ %\IFF T_{g}(x^{k+1}) - T_{h}(x^{k+1}) \leq \epsilon \Rightarrow  
then $f(x^k)-f(\tilde x^{k+1}) =  T_{g}(\tilde x^{k+1}) - T_{h}(\tilde x^{k+1}) \leq \epsilon$. In this case, $x^k$ is an $(\epsilon_x + \epsilon_1, 0)$-critical point of $g - h$, with 
$y^k\in\p_{\epsilon_x + \epsilon_1} g(x^k) \cap \p h(x^k), \tilde x^{k+1}$ is an $(\epsilon_x, \epsilon_2)$-critical point of $g - h$ with $y^k\in \p_{\epsilon_x} g(\tilde x^{k+1}) \cap \p_{\epsilon_2} h(\tilde x^{k+1})$ for some $\epsilon_1 + \epsilon_2 = \epsilon, \epsilon_1 \geq - \epsilon_x, \epsilon_2 \geq 0,$ and $\tfrac{\rho(2-t_x)}{2}\|x^k- \tilde x^{k+1}\|^2 \leq \epsilon + \tfrac{\epsilon_x}{t_x}$. Conversely, if $x^k \in\p_{\epsilon_x + \epsilon_1} g^*(y^k),$ and $y^k\in \p_{\epsilon_2} h(\tilde x^{k+1})$, then
$T_g(\tilde x^{k+1}) - T_h(\tilde x^{k+1}) \leq \epsilon_x  + \epsilon$ and  $f(x^k) - f(\tilde x^{k+1})  \leq \epsilon_x  + \epsilon$.

\item \label{itm:obj-rate-round-app} $\min_{k \in \{0,1,\dots,K-1\}} T_g(\tilde x^{k+1}) - T_h(\tilde x^{k+1}) =  \min_{k \in \{0,1,\dots,K-1\}} f(x^k)-f(\tilde x^{k+1}) \leq$ %\min_{k \in \{0,1,\dots,K-1\}} T_g(x^{k+1}) - T_h(x^{k+1}) =  
$\min_{k \in \{0,1,\dots,K-1\}} F(X^k)-F(X^{k+1}) \leq \frac{F(X^0) - F^\star}{K}$.

\item \label{itm:iterates-rate-round-app} If $\rho>0$, then
    \[
    \min_{k \in \{0,1,\dots,K-1\}}\|x^k-\tilde x^{k+1}\|\leq \sqrt{\tfrac{2}{\rho (2-t_x)} \big( \frac{F(X^0)-F^\star}{K } + \tfrac{\epsilon_x}{t_x} \big)}.
    \]

\end{enumerate}
\end{theorem}

\begin{proof}
\looseness=-1 Note that the iterates $\tilde{x}^{k+1}, y^k$ are generated by an approximate DCA step from $x^k$, so \cref{thm:convergence-app} apply to them.
\begin{enumerate}[label=\alph*)]
\item The claim follows from \cref{thm:convergence-app}-\ref{itm:criticality-app}. %The claim on $T_h(x^{k+1})$ follows in a similar way, since the proof only requires $y^k \in \p h(x^k)$.

\item By \cref{thm:convergence-app}-\ref{itm:descent-app}, we have
  \[f(x^k)-f(\tilde x^{k+1}) = T_g(\tilde x^{k+1}) - T_h(\tilde x^{k+1}) \geq \tfrac{\rho (2-t_x)}{2} \|x^k-\tilde x^{k+1}\|^2 - \tfrac{\epsilon_x}{t_x}.\]
  By \cref{prop:LEproperties}-\ref{itm:extension}, we also have $F(X^k) - F(X^{k+1}) = f(x^k)-f( x^{k+1})$.
The claim then follows since $f(x^{k+1}) \leq f(\tilde{x}^{k+1})$ by \cref{prop:LEproperties}-\ref{itm:round}.

\item This follows from \cref{itm:descent-round-app} and \cref{thm:convergence-app}-\ref{itm:convergence-app}. 

\item This follows from \cref{itm:descent-round-app} by telescoping sum.

\item This follows from Items \ref{itm:descent-round-app} and \ref{itm:obj-rate-round-app}.

\end{enumerate}
\end{proof}

\section{Proofs of \cref{sec:CDCA}} \label{sec:CDCA-proofs}

\subsection{Proof of \cref{corr:FWconvergence}} \label{sec:FWconv-proof}

\FWconvergence*
 \begin{proof}
We observe that approximate FW with $\gamma_t = 1$ is a special case of approximate DCA \eqref{alg:DCA}, with DC components $$g' = \delta_{\p h(x^k)} \text{ and } h' = - \phi_k,$$ 
and $\epsilon_x=0, \epsilon_y = \epsilon$. Indeed, we can write the approximate FW  iterates $w^0 \in \p h(x^k) = \dom \p g'$, $- s^t \in \p_\epsilon h'(w^t)$ and $w^{t+1} = v^t \in \argmin_{w} g'(w) - \ip{- s^t}{w} = \p (g')^*(-s^t)$, which are valid iterates of approximate DCA \eqref{alg:DCA}.

We show also that $g', h' \in \Gamma_0$: We can assume w.l.o.g that $\p h(x^k) \not = \emptyset$, otherwise the bound holds trivially. Hence, $g'$ is proper. And since $h \in \Gamma_0$, $\p h(x^k)$ is a closed and convex set, hence $g'$ is a closed and convex function. We also have that $h'$ is proper, since otherwise Problem \eqref{eq:CDCA-y} would not have a finite minimum, which also implies that the minimum of the DC dual \eqref{eq:DCdual} is not finite, contradicing our assumption that the minimum of the DC problem \eqref{eq:DC} is finite. Finally, since the fenchel conjugate $g^*$ is closed and convex, then $h'$ is also closed and convex.

We can thus apply \cref{thm:convergence-app}. We get
\begin{align*}
\min_{k \in \{0,1,\dots,K-1\}} T_{g'}(w^{k+1}) - T_{h'}(w^{k+1}) &\leq \frac{\phi(w^0) - \min_{ w \in \partial h(x^k)} \phi_k(w)}{K} & \text{(by \cref{thm:convergence-app}-\ref{itm:obj-rate-app})}\\
\Rightarrow \min_{k \in \{0,1,\dots,K-1\}} T_{g'}(w^{k+1})  &\leq \frac{\phi(w^0) - \min_{ w \in \partial h(x^k)} \phi_k(w)}{K} + \epsilon& \text{(by \cref{thm:convergence-app}-\ref{itm:criticality-app} with $t_y=1$)}
\end{align*} 
 The claim now follows by noting that $T_{g'}(w^{t+1}) =  \ip{ s^t}{w^t - w^{t+1}} = \mathrm{gap}(w^t)$.
 \end{proof}

\subsection{Proof of \cref{prop:FWLO} } \label{sec:FWLO-proof}

\FWLO*
 \begin{proof}
By \cref{prop:LEproperties}-\ref{itm:greedy}, $w \in \p h_L(x)$, so it is a feasible solution. Given any $w' \in \p h_L(x)$, $w'$ is a maximizer of $\max_{w \in  B(H)} \ip{w}{s}$, hence it must satisfy $w'(C_i) = H(C_i)$ for all $i \in \{1, \cdots, m\}$ \citep[Proposition 4.2]{Bach2013}. We have
\begin{align*}
    \ip{s}{w - w'} &= \sum_{i=1}^m  \sum_{k=1}^{|A_i|} s_{\sigma(|C_{i-1}|+k)} \left( H(\sigma(|C_{i-1}|+k) \mid S^\sigma_{|C_{i-1}|+k-1} ) - w'_{\sigma(|C_{i-1}|+k)} \right)\\
     &= \sum_{i=1}^m  \left( \begin{multlined} \sum_{k=1}^{|A_i|-1} (s_{\sigma(|C_{i-1}|+k)} - s_{\sigma(|C_{i-1}|+k+1)} ) ( H(S^\sigma_{|C_{i-1}|+k}) - w'(S^\sigma_{|C_{i-1}|+k}) ) \\
      - s_{\sigma(|C_{i-1}|+1)} (H(S^\sigma_{|C_{i-1}|}) - w'(S^\sigma_{|C_{i-1}|}) ) + s_{\sigma(|C_{i}|)} (H(S^\sigma_{|C_{i}|}) - w'(S^\sigma_{|C_{i}|}) ) \end{multlined} \right)\\
     &\geq 0.
\end{align*}
The last inequality holds since $w' \in  B(H)$ and $S^\sigma_{|C_{i}|} = C_i$ for all $i \in \{1, \cdots, m\}$.
 \end{proof}
 
 \subsection{Proof of \cref{thm:ConvCDCA}} \label{sec:ConvCDCA-proof}
To prove \cref{thm:ConvCDCA} we need the following lemma, which extends the result in \citep[Theorem 2.3]{tao1988duality}.

\begin{lemma}\label{lem:strongCriticalityCDCA}
For any $\epsilon \geq 0$, $\hat{x}$ is an $\epsilon$-strong critical point of $g - h$ %satisfies $\p h(\hat{x}) \subseteq \p_\epsilon g(\hat{x})$ 
if and only if there exists $\hat{y} \in \argmin \{ \ip{y} {\hat{x}} - g^*(y) : y \in \partial h(\hat{x})\}$ such that $\hat{x} \in \p_\epsilon g^*(\hat{y})$.
\end{lemma}
\begin{proof}
%We extend the proof from \citep[Theorem 2.3]{tao1988duality}.
If $\hat{x}$ is an $\epsilon$-strong critical point of $g - h$, i.e., $\p h(\hat{x}) \subseteq \p_\epsilon g(\hat{x})$, then for every $y \in \partial h(\hat{x})$, we have $y \in \p_\epsilon g(\hat{x})$. In particular, this holds for $\hat{y} \in \argmin \{ \ip{y} {\hat{x}} - g^*(y) : y \in \partial h(\hat{x})\}$, hence  $\hat{x} \in \p_\epsilon g^*(\hat{y})$ by \cref{prop:FenchelDualPairs}.
Conversely, given $\hat{y} \in \argmin \{\ip{y}{\hat{x}} - g^*(y) : y \in \partial h(\hat{x})\}$ such that $\hat{x} \in \p_\epsilon g^*(\hat{y})$, we have 
\begin{align}\label{eq:defdualsubgradient}
\ip{\hat{y}}{\hat{x}} - g^*(\hat{y}) \leq \ip{y}{\hat{x}} - g^*(y), \forall y \in \p h(\hat{x}).
\end{align}
Since $\hat{x} \in \p_\epsilon g^*(\hat{y})$, we have by \cref{prop:FenchelDualPairs}, $g^*(\hat{y}) + g(\hat{x}) - \ip{\hat{y}}{\hat{x}} \leq \epsilon$. Combining this with \eqref{eq:defdualsubgradient} yields
\[g(\hat{x}) - \epsilon \leq \ip{y}{\hat{x}} - g^*(y), \forall y \in \p h(\hat{x}).\]
By definition of $g^*$, we obtain
\[g(\hat{x}) - \epsilon \leq \ip{y}{\hat{x} - x} + g(x), \forall x \in \R^d, \forall y \in \p h(\hat{x}).\]
Hence $y \in \p_\epsilon g(\hat{x})$ for all $y \in \p h(\hat{x})$.
\end{proof}

 \ConvCDCA*
 \begin{proof}
 Since approximate CDCA is a special case of approximate DCA, with $\epsilon_y=0$,  \cref{thm:convergence} applies.
If $f(x^k) - f(x^{k+1}) \leq \epsilon$, we have by \cref{thm:convergence}-\ref{itm:convergence} that $x^k \in \p_{\epsilon_x + \epsilon} g^*(y^k)$. Hence, by \cref{lem:strongCriticalityCDCA}  $x^k$ is an $(\epsilon_x + \epsilon)$-strong critical point of $g - h$, i.e., $\p h(x^k) \subseteq \p_{\epsilon_x + \epsilon} g(x^k)$. 
\red{If $h$ is locally polyhedral, this implies that $x^k$ is an $(\epsilon + \epsilon_x)$-local minimum of $f$ by \cref{prop:LocalOptCondition}-\ref{itm:polyhedral}.}
\end{proof}
 
\subsection{Proof of \cref{corr:LocMinCDCA}} \label{sec:LocMinCDCA-proof}

\LocMinCDCA*
\begin{proof}
Assume that $\hat{x}$ is an $\epsilon$-strong critical point of $g - h$.
We first observe that any vector $x=\1_X$ corresponding to $X \subseteq \hat{X}$ or $X \supseteq \hat{X}$ has a common decreasing order with $\hat{x}$, hence choosing $\hat{y}$ as in \cref{prop:LEproperties}-\ref{itm:greedy} according to this common order yields $\hat{y} \in \p h_L(\hat{x}) \cap \p h_L(x)$, and $\hat{y} + \rho \hat{x} \in \p h(\hat{x}) \subseteq \p_\epsilon g(\hat{x})$. By \cref{lem:eps-subdiff-reg}, we thus have $\hat{y} \in \p_{\epsilon'} (g_L + \delta_{[0,1]^d})(\hat{x})$ and $\p_{\epsilon'} (g_L + \delta_{[0,1]^d})(\hat{x}) \cap   \p h_L(x) \not = \emptyset$. \cref{prop:LocalOptCondition}-\ref{itm:general} then implies that $f(\hat{x})\leq f(x) + \epsilon'$. Hence, $\hat{X}$ is an $\epsilon'$-strong local minimum of $F$ by \cref{prop:LEproperties}-\ref{itm:extension}.
%the Lovasz extension $f_L$ coincide with $F$ on the vertices of $[0,1]^d$.

\red{Conversely, assume $\hat{X}$ is an $\epsilon'$-strong local minimum of $F$. We argue that there exists a neighborhood $B_\delta(\hat{x}) = \{x: \| x - \hat{x} \|_\infty < \delta\}$ of $\hat{x}$, where $\delta = 1/4$, %any delta < 1/2 works 
such that any $X = \round(x)$ for $x \in B_\delta(\hat{x})$, satisfies $X \subseteq \hat{X}$ or $X \supseteq \hat{X}$. To see this note that for $x \in B_\delta(\hat{x})$, any $\sigma \in S_d$  such that  $x_{\sigma(1)}\geq \dots \geq x_{\sigma(d)}$ would have $S^\sigma_{|\hat X|} = \hat X$, since $\hat{x}_i - \delta > \hat{x}_j + \delta$ for all $i \in \hat{X}, j \not \in \hat{X}$. Since $X = S^\sigma_{\hat k}$ where $\hat k \in \argmin_{k=0,1,\dots,d}F(S^\sigma_k)$, it must satisfy $X \subseteq \hat{X}$ or $X \supseteq \hat{X}$.
As a result, we have by \cref{prop:LEproperties}-\ref{itm:round},\ref{itm:extension}, \[ f_L(\hat x) = F(\hat X) \leq F(X) + \epsilon' \leq f_L(x) + \epsilon'.\]
Hence, $\hat x$ is an $\epsilon$-local minimum of $f$, and thus also $\epsilon$-strong critical point of $g - h$ by \cref{prop:LocalOptCondition}-\ref{itm:general}.}
%Since $X = $ where $\hat k \in \argmin_{k=0,1,\dots,d}F(S^\sigma_k)$ 
%order any $i \in \hat{X}$ before any $j \not \in \hat{X}$, i.e., $\sigma$
\end{proof}
%\mtodo{Statement/proof in red is a vacuous result.}
\mtodo{We can still show that if $\hat{X}$ is an $\epsilon$-strong local minimum of $F$ then $\hat{x}$ is a $\epsilon / \tau$-strong critical point of $g - h$ for any $\tau < 1/2$. But for now I will remove all the converse direction unless we're using this somewhere.}

\subsection{Convergence properties of \CDCAR} \label{sec:CDCARconv-proof}

\begin{corollary}\label{corr:CDCARconv}
Let $\{x^k\}, \{\tilde x^{k+1}\}$ and $\{y^k\}$ be generated by an approximate version of CDCAR \eqref{eq:CDCAR} where $\tilde x^{k+1} \in \p_{\epsilon_x} g^*(y^k)$  and for some $\epsilon_x \geq 0$.
Then all of the properties in Theorem \ref{them:convergence-round-app} hold. 
%Moreover, the iterates $\{x^t\}$ converge to a strong local minimum of \eqref{eq:SetOpt} in finitely many iterations.
In addition, if $F(X^k) - F(X^{k+1}) \leq \epsilon$ for some $\epsilon \geq 0$ then $x^k$ is an $(\epsilon + \epsilon_x)$-strong critical point of $f$, with $\p h(x^k) \subseteq \p_{\epsilon_x + \epsilon} g(x^k)$, and $X^k$ is an $\epsilon'$-strong local minimum of $F$, where $\epsilon' = \sqrt{2 \rho d (\epsilon + \epsilon_x)}$ if $\epsilon + \epsilon_x \leq \tfrac{\rho d}{2}$, and $\tfrac{\rho d}{2} + \epsilon + \epsilon_x$ otherwise. \red{If $\rho=0$, $x^k$ is also an $\epsilon + \epsilon_x$-local minimum of $f$.}
\end{corollary}
\begin{proof}
Since \CDCAR is a special case of \DCAR, then all properties of the latter apply to the former.
In addition, if $F(X^k) - F(X^{k+1}) \leq \epsilon$, we have by \cref{them:convergence-round-app}-\ref{itm:convergence-round-app} that $x^k \in \p_{\epsilon_x + \epsilon} g^*(y^k)$. Hence, by \cref{lem:strongCriticalityCDCA}  $x^k$ satisfies $\p h(x^k) \subseteq \p_{\epsilon_x + \epsilon} g(x^k)$. Hence, $X^k$ is a $\epsilon'$-strong local minimum of $F$ by \cref{corr:LocMinCDCA}. \red{If $\rho=0$, $h$ is polyhedral, hence $x^k$ is an $\epsilon + \epsilon_x$-local minimum of $f$ by \cref{prop:LocalOptCondition}-\ref{itm:polyhedral}.}
\end{proof}
%\mtodo{Statement/proof in red is a vacuous result.}

\section{Remarks on Local Optimality Conditions}

The following example shows that rounding a fractional solution $x^K$ returned by DCA or CDCA will not necessarily yield an $\epsilon$-local minimum of $F$, for any $\epsilon \geq 0$, even if $x^K$ is a local minimum of $f_L$. It also shows that the objective achieved by a local minimum of $f_L$ can be arbitrarily worse than the minimum objective.
\begin{example}\label{ex:localMinRound}
For any $\epsilon \geq 0, \alpha > \epsilon$, 
let $V = \{1,2,3\}, G(X) = \alpha |X|$, and $H: 2^V \to \R$ be a set cover function defined as $H(X) = \alpha |\bigcup_{i \in X} U_i|$, where $U_1 = \{1\}, U_2= \{1,2\}, U_3 = \{1,2,3\}$. Then $G$ is modular, $H$ is submodular, % monotone submodular functions 
and their corresponding Lov\'asz extensions are $g_L(x) = \alpha (x_1 + x_2 + x_3)$ and $h_L(x)  = \alpha (\max\{x_1, x_2, x_3\} + \max\{x_2, x_3\} + x_3)$; see e.g., \citep[Section 6.3]{Bach2013}. The minimum value
$\min_{X \subseteq V} G(X) - H(X) = -2 \alpha$ is achieved at $X^* = \{3\}$. 
Consider a solution $\hat{x} = (1, 0.5, 0)$, $\hat{x}$ is a local minimum of $f_L$. %(and hence of $\tilde{g}_L - \tilde{h}_L$ too). 
To see this note that for any vector $x$ such that $x_1 > x_2 > x_3$ we have $h_L(x) = g_L(x)$, hence $f_L(x) = 0 = f_L(\hat{x})$. Accordingly, for any $x$ in the neighborhood $\{x : \|x -  \hat{x}\|_\infty < 0.25$ of $\hat{x}$, we have $f_L(x) = 0 = f_L(\hat{x})$, thus $\hat{x}$ is a local minimum of $f_L$.  
On the other hand none of the sets $\emptyset, \{1\}, \{1,2\}, \{1,2,3\}$ obtained by rounding $\hat{x}$ via \cref{prop:LEproperties}-\ref{itm:round} are $\epsilon$-local minima of $F$, since they all have objective value $F(\hat{X}) = 0$ and adding or removing a single element yields an objective $F(X) = -\alpha$ (we can choose $X$ to be $\{2\}, \{13\}, \{2\}, \{23\}$ respectively). \\
%(here $f_L(\hat{x})=0$ while $\min f_L(x) = -2$).
Note that if $x^k=\hat{x}$ at any iteration $k$ of DCA (e.g., if we initialize at $x^0 = \hat{x}$) and $\rho>0$, then $x^{k+1}=x^k$ and DCA will terminate. To see this note that $h$ has a unique subgradient at $x^k$ which is $y^k = \rho x^k + \1$, and $x^k =  \argmin_{x \in [0,1]^d} g_L(x) - \ip{x}{y^k} + \tfrac{\rho}{2} \| x \|^2$. This also applies to CDCA, since DCA and CDCA coincide in this case (since $\p h(x^k)$ has a unique element).  
Note also that the objective at this local minimum $f_L(\hat{x})=0$ is arbitrarily worse than the minimum objective $\min_{x \in [0,1]^3} f_L(x) = -2\alpha$.
\end{example}
Note that in the above example, the variant of DCA in \cref{alg:DCASetwithLocalMinCheck}
%checking local minimality and restarting from the best neighboring set if not satisfied, as explained in \cref{sec:reg}, 
would yield the optimal solution $X^*$ (e.g., if we pick $\emptyset$ as the rounded solution).

\mtodo{Can we find an example where rounding a local minimum of $f_L$ returned by CDCA will not necessarily yield an approximate strong local minimum of $F$ even if rounded solution is an approximate local minimum? This would better motivate having the rounded variants. We can't use any supermodular function since in that case local min and strong local min are equivalent. The below example also does not work since any local minimum of $f_L(x) = x_1 + \max \{x_2, x_3\} + \max \{x_4, \cdots, x_d\}  - \max \{x_1, x_4, \cdots, x_d\} - x_2 - x_3$ therein will have $x_2 = x_3 = 1$ and hence $\{23\}$ will the rounded solution, which is the global minimum.}
%The following example shows that the objective achieved by a discrete local minimum can be arbitrarily worse than any strong local minimum. This highlights the importance of the stronger guarantee of CDCA.
%\begin{example}\label{ex:strongLocalMin}
%Let $V = \{1,2,3,4\}$, $G(X) = \sqrt{|\bigcup_{i \in X} U_i|}$ where $U_1 = \{1\}, U_2 = \{2\}, U_3 = \{2,3\}, U_4 = \{1,3\}$, and $H(X) = \sqrt{|X|}$. Consider a solution $X = \{1,2\}$, $X$ is a local minimum since adding or removing a unique element results in the same objective $F(X) = 0$, while any strong local minimum will satisfy $F(X) \leq F(V) = \sqrt{3} - 2$, which is also the optimal solution.
%Note that this is a special case of the speech corpus selection problem we consider in the experiments. 
%\end{example} --> this example is not very relevant since for DCA we're guaranteed to have F(X^k) <= F(S_d^sigma) where S_d^sigma = V for any sigma. 

The following example shows that the objective achieved by a set satisfying the guarantees in \cref{corr:LocMinDCA} can be arbitrarily worse than any strong local minimum. This highlights the importance of the stronger guarantee of CDCA.
\begin{example}\label{ex:strongLocalMin}
Let $V = \{1,\cdots, d\}, \alpha > 0$, and $G, H: 2^V \to \R$ be set cover functions defined as $G(X) = \alpha |\bigcup_{i \in X} U^G_i|$, where $U^G_1 = \{1\}, U^G_2 = U^G_3= \{2\}, U^G_4 = \cdots =  U^G_d = \{3\}$ and $H(X) = \alpha  |\bigcup_{i \in X} U^H_i|$, where $U^H_1 = U^H_4  = \cdots =  U^H_d = \{1\}, U^H_2= \{2\}, U^H_3 = \{3\}$. Then $G$ and $H$ are submodular; see e.g., \citep[Section 6.3]{Bach2013}.
Consider $X = \{1\}$, $X$ is a local minimum since adding or removing any element results in the same objective $F(X) = 0$ or larger. We argue that $X$ also satisfies the rest of the guarantees in 
\Cref{corr:LocMinDCA}, i.e., $F(X) \leq F(S_\ell^{\sigma_i})$ for all $\ell \in V$, where $\sigma_2, \cdots, \sigma_d \in S_{d}$ correspond to decreasing orders of $\1_X$ with $\sigma_i(2) = i$. Each $\sigma_i$ admits $(d-2)!$ valid choices.
 %corresponding to decreasing orders of $x^k$ with different elements at $\sigma(2)$
Note that the only possible values of $F(S_\ell^{\sigma_i})$ are $0, \alpha$ and $-\alpha$, with $-\alpha$ achieved only at $S_3^{\sigma_2} = S_3^{\sigma_3} = \{1,2,3\}$ with the choices of $\sigma_2$ starting with $(1,2,3)$ and the choices of $\sigma_3$ starting with $(1, 3, 2)$. %, i.e., with $(d-3)!$ possible choices for each. 
So, for any other choices of $\sigma_2$ and $\sigma_3$, $X$ satisfies the guarantees in \cref{corr:LocMinDCA}. If $\sigma_i$'s are chosen uniformly at random, $X$ would satisfy the guarantees in \cref{corr:LocMinDCA} with probability $1 - \tfrac{2}{d-2}$. % 1- 2(d-3)!/(d-2)! 
On the other hand, any strong local minimum $\hat{X}$ must contain $\{2, 3\}$ since otherwise the set $X' = \hat{X} \cup (\{2, 3\} \setminus \hat{X}) \supset \hat{X}$ has a smaller objective $F(X') = F(\hat{X}) - \alpha$ leading to a contradiction. It follows then that any strong local minimum will satisfy $F(\hat{X}) \leq F(\{2, 3\}) = -\alpha$, which is also the optimal solution, and arbitratily better than the objective achieved by $X$.
\end{example}

\mtodo{In the above example, DCA with $\rho > 1$ would not converge at $x^k=(1,0 \cdots, 0)$ for any $\epsilon < \tfrac{1}{\rho}$, since any $(y^k - \rho x^k) \in \p h_L(x^k)$ will have $(y^k - \rho x^k)_1 = 0$, hence $x^{k+1} \in  \argmin_{x \in [0,1]^d} g_L(x) - \ip{x}{y^k} + \rho \|x \|^2 = x_1 + \max \{x_2, x_3\} + \max \{x_4, \cdots, x_d\}  - \ip{x}{y^k}+ \rho \|x \|^2$ will  have $x^{k+1}_1 = 1 - \tfrac{1}{\rho} \in [0,1)$. Thus $\| x^{k+1} - x^k\|_2 > \tfrac{1}{\rho} > \epsilon$ and $f(x^k) - f(x^{k+1}) > \epsilon$ for a proper choice of $\epsilon_x$.
Can we find an example where DCA with any $\rho$ (with $O(d)$ permutations, no heuristics) converges to a point which is much worse than what CDCA converges to?}

\section{Special Cases of DS Minimization}\label{sec:specialCases}

In this section, we discuss  some implications of our results to some special cases of the DS problem \eqref{eq:DS}. To that end, we define two types of approximate submodularity and supermodularity, and show how they are related.

First, we recall the notions of weak DR-submodularity/supermodularity, which were introduced in \cite{Lehmann2006} and \cite{Bian2017a}, respectively.

\begin{definition}\label{def:WDR}
A set function $F$ is \emph{$\alpha$-weakly DR-submodular}, with $\alpha > 0$, if 
\[F( i \mid A) \geq \alpha F( i \mid B), \text{ for all $ A \subseteq B, i \in V \setminus B$}. \]
Similarly, $F$ is $\beta$-weakly DR-supermodular, with $\beta > 0$, if 
\[ F( i \mid B)  \geq \beta F( i \mid A), \text{ for all $A \subseteq B, i \in V \setminus B$}. \]
We say that $F$ is \emph{$(\alpha,\beta)$-weakly DR-modular} if it satisfies both properties.
\end{definition}
In the above definition, if $F$ is non-decreasing, then $\alpha, \beta \in (0,1]$, if it is non-increasing, then $\alpha, \beta \geq 1$, and if it is neither (non-monotone) then $\alpha = \beta = 1$. 
 $F$ is submodular (supermodular) iff $\alpha = 1$ ($\beta = 1$) and modular iff both $\alpha = \beta = 1$.

Next, we recall the following characterizations of submodularity and supermodularity:  A set function $F$ is submodular if $F(A) + F(B) \geq F(A \cup B) + F(A \cap B)$ for all $A, B \subseteq V$, and supermodular  if $F(A) + F(B) \leq F(A \cup B) + F(A \cap B)$. We introduce other notions of approximate submodularity and supermodularity based on these properties.

\begin{definition} \label{def:WSub}
A set function $F$ is \emph{$\alpha$-submodular}, with $\alpha > 0$, if 
\[F(A) + F(B) \geq \alpha (F(A \cup B) + F(A \cap B)), \text{ for all $ A, B \subseteq V$}. \]
Similarly, $F$ is $\beta$-supermodular, with $\beta > 0$, if 
\[ \beta (F(A) + F(B)) \leq F(A \cup B) + F(A \cap B), \text{ for all $A, B \subseteq V$}. \]
We say that $F$ is \emph{$(\alpha,\beta)$-modular} if it satisfies both properties.
\end{definition}
In the above definition, if $F$ is non-negative, then $\alpha, \beta \in (0,1]$, if it is non-positive, then $\alpha, \beta \geq 1$, and if it is neither then $\alpha = \beta = 1$. 
$F$ is submodular (supermodular) iff $\alpha = 1$ ($\beta = 1$) and modular iff both $\alpha = \beta = 1$.

The two types of approximate submodularity and supermodularity are related as follows.
\begin{proposition}\label{prop:WDR-WSub}
$F$ is $\alpha$-weakly DR submodular iff
\begin{equation}\label{eq:WDRsub}
F(A) + \alpha F(B) \geq F(A \cap  B) + \alpha F(A \cup B), \forall A,B \subseteq V.
\end{equation}
If $F$ is also normalized, then $F$ is $\alpha$-submodular.
%$$  F(A) +  F(B) \geq \alpha (F(A \cap  B) + F(A \cup B)), \forall A,B \subseteq V.$$
Similarly, $F$ is $\beta$-weakly DR supermodular iff
\begin{equation}\label{eq:WDRsup}
F(A) + \frac{1}{\beta} F(B) \leq F(A \cap  B) + \frac{1}{\beta} F(A \cup B), \forall A,B \subseteq V.
\end{equation}
If $F$ is also normalized, then $F$ is $\beta$-supermodular.
%$$  \beta (F(A) +  F(B)) \leq  F(A \cap  B) +  F(A \cup B), \forall A,B \subseteq V.$$
\end{proposition}
\begin{proof}
Given an $\alpha$-weakly DR submodular function $F$, let $A \setminus B = \{i_1, i_2, \cdots, i_r\}$. Then
\begin{align*}
F(A \cap B \cup \{i_1, \cdots, i_k\}) - F(A \cap B \cup  \{i_1, \cdots, i_{k-1}\}) &\geq \alpha (F(B \cup\{i_1, \cdots, i_k\})  - F(B  \cup  \{i_1, \cdots, i_{k-1})), \forall k=1, \cdots, r \\  
\Rightarrow F(A ) - F(A \cap B) &\geq \alpha (F(A \cup B)  - F(B)) \tag{ by telescoping sum}  
\end{align*} 
Rearranging the terms yields \eqref{eq:WDRsub}.
If $F$ is also normalized, then $F$ is $\alpha$-submodular. To see this, note that  if $\alpha < 1$, $F$ is non-decreasing and hence $F(X) \geq F(\emptyset)= 0$, and if $\alpha > 1$, $F$ is non-increasing and hence $F(X) \leq F(\emptyset)= 0$. Thus for any $\alpha >0$, we have $F(X) \geq \alpha F(X)$ for any $X \subseteq V$. In particular, applying this to $X = B$ and $X = A \cap B$, we obtain 
\[F(A) + F(B) \geq F(A) + \alpha F(B) \geq F(A \cap  B) + \alpha F(A \cup B)  \geq \alpha ( F(A \cap  B) + F(A \cup B) ).\]

Conversely, if $F$ satisfies \eqref{eq:WDRsub}, then for all $A' \subseteq B' \subseteq V$, let $A = A' \cup \{i\}, B = B'$, then 
\begin{align*}
F(A) + \alpha F(B) &\geq F(A \cap  B) + \alpha F(A \cup B) \\
\Rightarrow F( A' \cup \{i\}) + \alpha F(B') &\geq F( A') + \alpha F( B' \cup \{i\})\\
\Rightarrow F( i \mid A')  &\geq  \alpha F( i \mid B').
\end{align*}
Hence $F$ is $\alpha$-weakly DR submodular.
The remaining claims follow similarly.
\end{proof}

%Applications where this holds are Structured sparse learning with submodular regularizers and batch Bayesian optimization with submodular cost.

\subsection{Approximately submodular functions} \label{sec:ApproxSub}

\looseness=-1 We consider special cases of the DS problem \eqref{eq:DS} where $F$ is approximately submodular. In \cref{sec:CDCA}, we showed that CDCA with integral iterates $x^k$ and CDCAR converge to an $\epsilon'$-strong local minimum of $F$ when $F(X^k) - F(X^{k+1}) \leq \epsilon$, where 
\begin{equation}\label{eq:epsilon'}
\epsilon' = \begin{cases} \sqrt{2 \rho d (\epsilon + \epsilon_x)} &\text{ if $\epsilon + \epsilon_x \leq \tfrac{\rho d}{2}$} \\ \tfrac{\rho d}{2} + \epsilon + \epsilon_x &\text{otherwise}.
\end{cases}
\end{equation}
%$\epsilon' = \sqrt{2 \rho d (\epsilon + \epsilon_x)}$ if $\epsilon + \epsilon_x \leq \tfrac{\rho d}{2}$, and $\tfrac{\rho d}{2} + \epsilon + \epsilon_x$ otherwise.
The following two propositions relate the approximate strong local minima of an approximately submodular function to its global minima, for two different notions of approximate submodularity. 
%Hence, the aforementionned CDCA variants converge to an approximate global minimum of $F$ in these cases.
 
\begin{proposition}\label{prop:localMinApproxSub}
If $F$ is an $\alpha$-submodular function, then for any $\varepsilon \geq 0$, any $\varepsilon$-strong local minimum $\hat{X}$ of $F$ satisfies $F(\hat{X}) \leq  \frac{1}{2 \alpha -1} ( \min_{X \subseteq V} F(X) + 2 \varepsilon \alpha)$.
\end{proposition}
\begin{proof}
Let $X^*$ be an optimal solution. Since $\hat{X}$ is an $\varepsilon$-strong local minimum of $F$, we have $F(\hat{X}) \leq F(\hat{X} \cup X^*) + \varepsilon$ and $F(\hat{X}) \leq F(\hat{X} \cap X^*) + \varepsilon$. Hence, 
\begin{align*}
2 F(\hat{X}) &\leq F(\hat{X} \cup X^*) + F(\hat{X} \cap X^*) + 2 \varepsilon\\
2 F(\hat{X}) &\leq \frac{1}{\alpha} ( F(\hat{X}) + F(X^*)) + 2 \varepsilon\\
F(\hat{X}) &\leq \frac{1}{2 \alpha -1} ( F(X^*) + 2 \varepsilon \alpha).
\end{align*}
\end{proof}
\cref{prop:localMinApproxSub} applies to the solutions returned by CDCA with integral iterates $x^k$ and CDCAR on the DS problem \eqref{eq:DS}, with $\varepsilon = \epsilon'$. Moreover,  when $F$ is submodular, we have $\alpha = 1$, then any $\varepsilon$-strong local minimum is a $2 \varepsilon$-global minimum of $F$ in this case. In particular, if $H$ is modular, DCA and CDCA with integral iterates $x^k$, DCAR, and CDCAR, all converge to a $2 \epsilon'$-global minimum of $F$. This holds for the DCA variants since by \cref{thm:convergence}-\ref{itm:convergence}, DCA converges to an $(\epsilon + \epsilon_x, 0)$-critical point of $g - h$, and when $H$ is modular, $h$ is differentiable, hence any $(\epsilon + \epsilon_x, 0)$-critical point of $g - h$ is also an $(\epsilon + \epsilon_x)$-strong critical point, and by \cref{corr:LocMinCDCA} it is also an $\epsilon'$-strong local minimum of $F$ if it is integral.

\begin{proposition}\label{prop:localMinWDRSub} Given $F = G - H$, if $G$ is submodular and $H$ is $\beta$-weakly DR-supermodular, then for any $\varepsilon \geq 0$, any $\varepsilon$-strong local minimum $\hat{X}$ of $F$ satisfies $F(\hat{X}) \leq  G(X^*)  - \beta H(X^*) + 2 \varepsilon$, where $X^*$ is a minimizer of $F$.
\end{proposition}
\begin{proof}
Since $\hat{X}$ is an $\varepsilon$-strong local minimum of $F$, we have $F(\hat{X}) \leq F(\hat{X} \cup X^*) + \varepsilon$ and $F(\hat{X}) \leq F(\hat{X} \cap X^*) + \varepsilon$.
By \cref{prop:WDR-WSub}, $H$ satisfies $ \frac{1}{\beta} H(\hat{X}) + H(X^*) \leq \frac{1}{\beta} H(\hat{X} \cup X^*) + H(\hat{X} \cap X^*) \leq \frac{1}{\beta} (H(\hat{X} \cup X^*) + H(\hat{X} \cap X^*))$. Hence, 
\begin{align*}
2 F(\hat{X}) &\leq F(\hat{X} \cup X^*) + F(\hat{X} \cap X^*) + 2 \varepsilon \\
2 F(\hat{X}) &\leq (G(\hat{X}) + G(X^*)) - ( H(\hat{X}) + \beta H(X^*)) + 2 \varepsilon \\
F(\hat{X}) &\leq  G(X^*) - \beta H(X^*) + 2 \varepsilon.
\end{align*}
\end{proof}
\looseness=-1 \Cref{prop:localMinWDRSub} again applies to the solutions returned by CDCA with integral iterates $x^k$ and CDCAR on the DS problem \eqref{eq:DS}, with $\varepsilon = \epsilon'$.
This guarantee matches the one provided in \citep[Corollary 1]{Halabi20} in this case (though the result therein does not require $H$ to be submodular), which is shown to be optimal \citep[Theorem 2]{Halabi20}. %taking \alpha=1, the hardness result is based on a function which is the difference of a submodular and a (1, beta)-modular function, so it applies to this case. 
%Applications where this holds are Structured sparse learning with submodular regularizers and batch Bayesian optimization with submodular cost.

The following proposition shows that a similar result to \cref{prop:localMinWDRSub} holds under a weaker assumption (recall from \cref{corr:LocMinCDCA} that if $\hat{X}$ is an $\varepsilon$-strong local minimum of $F$ then $\1_{\hat{X}}$ is an $(\varepsilon, 0)$-critical point of $g - h$). %; namely that $x^k$ is an $\epsilon$-critical point of $g - h$.

\begin{proposition}\label{prop:localMinWDRSub-weaker}
%Given $f= g -h$ as defined in \eqref{eq:DS-DCdecomposition}, let $\{x^k\}$ and $\{y^k\}$ be generated by approximate DCA \eqref{eq:DCASet}, and let $X^k = \round(x^k)$. If $f(x^k)$
Given $F = G - H$ where $G$ is submodular and $H$ is $\beta$-weakly DR-supermodular, $f= g -h$ as defined in \eqref{eq:DS-DCdecomposition}, $\varepsilon \geq 0$, let $\hat{x}$ be an $(\varepsilon, 0)$-critical point of $g - h$, with $\hat y \in \p_\varepsilon g(\hat{x}) \cap \p h(\hat{x})$, where $\hat y - \rho \hat x$ is computed as in \cref{prop:LEproperties}-\ref{itm:greedy}.
%If $H$ is $\beta$-weakly DR-supermodular, 
Then $\hat{X} = \round(\hat{x})$ satisfies $F(\hat{X}) \leq  G(X^*)  - \beta H(X^*) + \varepsilon'$, where $X^*$ is a minimizer of $F$, and $\varepsilon' = \sqrt{2 \rho d \varepsilon}$ if $\varepsilon \leq \tfrac{\rho d}{2}$ and $\tfrac{\rho d}{2} + \varepsilon $ otherwise.
\end{proposition}
\begin{proof}
Since $\hat y \in \p_\varepsilon g(\hat{x})$, we have by \cref{lem:eps-subdiff-reg} that $\hat y  - \rho \hat x \in \p_{\varepsilon'} (g_L + \delta_{[0,1]^d})(\hat x)$. Hence, for all $x \in [0,1]^d$
\begin{equation}\label{eq:gSubgrad}
g_L(x) \geq g_L(\hat x) + \ip{\hat y - \rho \hat x}{x - \hat x} - \varepsilon'.
\end{equation}
Since $H$ is $\beta$-weakly DR-supermodular and $\hat y - \rho \hat x$ is computed as in \cref{prop:LEproperties}-\ref{itm:greedy}, we have by \citep[Lemma 1]{Halabi20}, for all $x \in \R^d$, 
\begin{equation}\label{eq:hSupergrad}
- \beta h_L(x) \geq -h_L(\hat x) - \ip{\hat y - \rho \hat x}{x - \hat x}.
\end{equation}
Combining \eqref{eq:gSubgrad} and \eqref{eq:hSupergrad}, we obtain
\[g_L(x)- \beta h_L(x)  \geq g_L(\hat x)-h_L(\hat x) - \varepsilon'.\]
In particular, taking $x^* = \1_{X^*}$, we have by \cref{prop:LEproperties}-\ref{itm:extension},\ref{itm:round},
\[G(X^*) - \beta H(X^*) = g_L(x^*)- \beta h_L(x^*) \geq f_L(\hat x) - \varepsilon' \geq F(\hat X) - \varepsilon'.\]
\end{proof}
\Cref{prop:localMinWDRSub-weaker} applies to the solution returned by any variant of DCA and CDCA (including ones with non-integral iterates $x^k$) on the DS problem \eqref{eq:DS}, with $\varepsilon = \epsilon + \epsilon_x, \varepsilon' = \epsilon'$. In particular, if $H$ is modular ($\beta = 1$), they all obtain an $\epsilon'$-global minimum of $F$.
%converges to an approximate global minimum of $F$

\mtodo{Note that this result does not contradict \cref{ex:localMinRound}, since here we have the extra assumption that $H$ is $\beta$-weakly DR-supermodular. Note also that it is still possible for $\hat{X}$ satisfying the guarantee in \cref{prop:localMinWDRSub} to not be an $\epsilon'$-local minimum of $F$, if there exists $X'$ such that $|X' \Delta \hat X| = 1$ and $F(X') < F(X^*) + (1-  \beta) H(X^*)$.}

\mtodo{How are weak submodularity and weak supermodularity related to $\alpha$-submodularity and $\beta$-supermodularity. I tried to check if either conditions imply the other, but it does not seem to be the case. If not, can we provide counter-examples?}

\subsection{Approximately supermodular functions}

We consider special cases of the DS problem \eqref{eq:DS} where $F$ is approximately supermodular.
In \cref{sec:DCA}, we showed that DCA with integral iterates $x^k$ and DCAR converge to an $\epsilon'$-local minimum of $F$ when $F(X^k) - F(X^{k+1}) \leq \epsilon$, with $\epsilon'$ defined in \eqref{eq:epsilon'}.
The following proposition shows that approximate local minima of a supermodular function are also approximate strong local minima. 
%\citet{Feige2011} showed that approximate weak and strong local maxima of a submodular function $F$ are equivalent, and compared them to gobal maxima. The following two propositions restate these results for minimizing a supermodular function $F$, and generalizes.
 
\begin{proposition}[Lemma 3.3 in \cite{Feige2011}]\label{prop:localMinStrong}
If $F$ is a supermodular function, then for any $\varepsilon \geq 0$, any $\varepsilon$-local minimum of $F$ is also an $\varepsilon d$-strong local minimum of $F$.
\end{proposition}
\begin{proof}
The proof follows in a similar way to \citep[Lemma 3.3]{Feige2011}, we include it for completeness.
Given an $\varepsilon$-local minimum $X$ of $F$, for any $X' \subseteq X$, let $X \setminus X' = \{i_1, \cdots, i_k\}$, then
\begin{align*}
F(X) - F(X') &= \sum_{\ell=1}^{k} F(i_{\ell} \mid X' \cup \{i_1, \cdots, i_{\ell-1} \} ) \\ 
&\leq \sum_{\ell=1}^{k} F(i_{\ell} \mid X \setminus i_{\ell})\\
&\leq d \varepsilon
\end{align*}
We can show in a similar way that $F(X) \leq F(X') + d \varepsilon$ for any $X' \supseteq X$.
\end{proof}

The following proposition relates the approximate strong local minima of an approximately supermodular function to its global minima.

\begin{proposition}\label{prop:LocalMinWSup}
If $F$ is a non-positive $\beta$-supermodular function, then for any $\varepsilon \geq 0$, any $\varepsilon$-strong local minimum $\hat{X}$ of $F$ satisfies $\min\{F(\hat{X}), F(V \setminus \hat{X})\} \leq \tfrac{1}{3 \beta^2} \min_{X \subseteq V} F(X) + \tfrac{2}{3} \varepsilon$. In addition, if $F$ is also symmetric, then $\hat{X}$ satisfies $F(\hat{X}) \leq \tfrac{1}{2 \beta} \min_{X \subseteq V} F(X) + \varepsilon$.
\end{proposition}
\begin{proof}
This proposition generalizes \citep[Theorem 3.4]{Feige2011}. The proof follows in a similar way.
Let $X^*$ be an optimal solution. Since $\hat{X}$ is an $\varepsilon$-strong local minimum of $F$, we have $F(\hat{X}) \leq F(\hat{X} \cup X^*) + \varepsilon$ and $F(\hat{X}) \leq F(\hat{X} \cap X^*) + \varepsilon$. Hence, 
\begin{align*}
2 F(\hat{X}) + F(V \setminus \hat{X}) &\leq F(\hat{X} \cap X^*) +  F(\hat{X} \cup X^*) + F(V \setminus \hat{X}) + 2 \varepsilon \\
&\leq  \tfrac{1}{\beta} ( F(\hat{X} \cap X^*) +  F(X^* \setminus \hat{X}) + F(V)) + 2 \varepsilon \\
&\leq  \tfrac{1}{\beta^2} (F(X^*) + F(\emptyset) + F(V)) + 2 \varepsilon.
\end{align*}
If $F$ is also symmetric then 
\begin{align*}
2 F(\hat{X}) &\leq F(\hat{X} \cap X^*) +  F(\hat{X} \cup (V \setminus X^*))  + 2 \varepsilon \\
&= F(\hat{X} \cap X^*) +  F((V \setminus \hat{X}) \cap  X^*)  + 2 \varepsilon \\
&= \tfrac{1}{\beta} ( F(X^*) +  F(\emptyset))  + 2 \varepsilon.
\end{align*}
\end{proof}
\cref{prop:localMinApproxSub} applies to the solutions returned by CDCA with integral iterates $x^k$ and CDCAR on the DS problem \eqref{eq:DS}, with $\varepsilon = \epsilon'$.
Moreover, when $F$ is non-positive supermodular, we have $\beta=1$, then 
the solutions returned by CDCA with integral iterates $x^k$ and CDCAR satisfy $\min\{F(\hat{X}), F(V \setminus \hat{X})\} \leq \tfrac{1}{3} F^\star + \tfrac{2}{3} \epsilon'$ and $F(\hat{X}) \leq \tfrac{1}{2} F^\star + \epsilon'$ if $F$ is symmetric; and by \Cref{prop:localMinStrong} 
the solutions returned by DCA with integral iterates $x^k$ and DCAR satisfy $\min\{F(\hat{X}), F(V \setminus \hat{X})\} \leq \tfrac{1}{3} F^\star + \tfrac{2}{3} \epsilon' d$ and $F(\hat{X}) \leq \tfrac{1}{2} F^\star + \epsilon' d$ if $F$ is symmetric.
These guarantees match the ones for the deterministic local search provided in \citep[Theorem 3.4]{Feige2011}
%\footnote{The guarantees therein have $\varepsilon$ in the multiplicative factor; $(\tfrac{1}{3} - \epsilon/d) \min_{X \subseteq V} F(X)$ for general and $(\tfrac{1}{2} - \epsilon/d) \min_{X \subseteq V} F(X)$ for symmetric}
, which are optimal for symmetric functions \citep[Theorem 4.5]{Feige2011}, but not for general non-positive supermodular functions, where a $1/2$-approximation guarantee can be achieved \citep[Theorem 4.1]{Buchbinder2012}.

The non-positivity assumption in \cref{prop:LocalMinWSup} is necessary as demonstrated by the following example.

\begin{example}
Let $V = \{1,\cdots, 4\}, \alpha > 0, G(X) = 2 \alpha |X|$, and $H: 2^V \to \R$ be a set cover function defined as $H(X) = \alpha  |\bigcup_{i \in X} U_i|$, where $U_i = \{1, \cdots, i\}$. Then $G, H$ are submodular functions, and $F$ is supermodular but not non-positive, since $F(V) = 4 \alpha  > 0$. Consider a solution $\hat{X} = \{2\}$,$F(\hat{X}) = - \alpha (d - 4) = 0, F(V \setminus \hat{X}) = 2 \alpha$ and $\hat{X}$ is a strong local minimum of $F$ since adding or removing any number of elements yields the same objective or worse. On the other hand, the minimum is $\min_{X \subseteq V} F(X) = - 2 \alpha$, achieved at $X^* = \{4\}$, which is arbitrarily better than $\min\{F(\hat{X}), F(V \setminus \hat{X})\}$.
\end{example}

%add comment about case  where $G$ is weakly DR-supermodular --> we don't get something nice in this case because we end up with 1/\beta G(V) - H(V).

\end{document}